\documentclass[Afour,sageh,times]{sagej}

\usepackage{moreverb,url}
\usepackage{amsmath, amsthm}

\usepackage[colorlinks,bookmarksopen,bookmarksnumbered,citecolor=red,urlcolor=red]{hyperref}
\newtheorem{thm}{Theorem}
\newtheorem{proposition}{Proposition}

\newtheorem{define}{Definition}
\newtheorem{corollary}{Corollary}
\newtheorem{assumption}{Assumption}
\newtheorem{remark}{Remark}
\usepackage{subcaption}
\newcommand{\setdef}[2]{\left\{ \left. #1\right| #2 \right\} }
\newcommand{\setdefr}[2]{\left\{ #1 \left| #2 \right. \right\} }
\newcommand{\fundef}[2]{\mathcal{F}\left( #1, #2\right)}
\renewenvironment{proof}{{\bfseries Proof.}}{\qedsymbol \newline}
\renewcommand{\qedsymbol}{$\blacksquare$}
\usepackage{algorithm}
\usepackage{algpseudocode}

\newcommand{\citebr}[1]{[\cite{#1}]}

\newcommand\BibTeX{{\rmfamily B\kern-.05em \textsc{i\kern-.025em b}\kern-.08em
T\kern-.1667em\lower.7ex\hbox{E}\kern-.125emX}}

\def\volumeyear{2016}

\begin{document}

\runninghead{Operator STL}

\title{An Operator-based approach to STL}

\author{Panagiotis Rousseas\affilnum{1} and Dimos V. Dimarogonas\affilnum{1}}

\affiliation{\affilnum{1}Department of Decision and Control Systems, School of Electrical Engineering and Computer Science, Royal Institute of Technology KTH, Sweden.}

\corrauth{Panagiotis Rousseas, Department of Decision and Control Systems, School of Electrical Engineering and Computer Science, Royal Institute of Technology KTH, Sweden.}

\email{rousseas@kth.se}

\begin{abstract}
Signal Temporal Logic (STL), has recently seen extensive development, owing to its rich expressivenes for autonomous planning and control. Nevertheless, existing verification and control synthesis methods are limited {with respect to} the complexity and degree of nesting of the formulae. In this work, we propose a novel approach to STL based on an operator acting on reachability value functions. This constitutes a new theoretical framework for handling complex multi-nested formulae while at the same time providing tools for on-line control synthesis. In contrast to focusing on the design of STL-based reachability (or control barrier) functions, we develop operator-based nesting rules directly. Our method's expressiveness is demonstrated both theoretically, where necessary and sufficient conditions for STL formula satisfaction are extracted, as well as in simulations with complex fragments.  
\end{abstract}


\maketitle

\section{Introduction}
\label{sec:introduction}
Increasing the autonomy of robots operating in the real world is a fundamental goal and challenge across research and industry; while a plethora of methods for safe, robust and successful control of autonomous systems exist, \textit{decision-making}, in the sense of long-term planning and/or scheduling, is usually left to human operators. Formally addressing autonomous decision-making is a crucial and necessary step towards increasing the capabilities and autonomy of modern systems.  
\par 
Autonomous decision-making can be viewed as a task planning problem; given a collection of tasks for an autonomous system to complete, task planning rests on deciding when and which tasks to complete. Formal languages such as Linear Temporal Logic (LTL) \citebr{ltl}, Signal Temporal Logic (STL) \citebr{stl_paper} and robust STL semantics \citebr{robsem_1,robsem_2} have been widely employed to model tasks as spatio-temporal constraints, owing to their expressiveness and flexibility. These languages are also useful for control synthesis over discrete/continuous systems; such problems can be posed as Mixed-Integer Linear Programs (MILPs) for linear systems \citebr{9769752}, however MILPs are NP-hard and scale poorly w.r.t. the number of decision variables \citebr{MONTANARI202248}. 
% Concerning LTLs, temporal logic trees have also been employed \citebr{9562161}. 
This combinatorial aspect is not limited to MILP; task planning is {in principle} a combinatorial problem.
\par 
Beyond MILPs, a plethora of control synthesis methods have been proposed. In \citebr{8404080}, Control Barrier Functions (CBFs) were used to guarantee satisfaction of simple STL formulae. CBFs have recently peaked the community's interest for on-line constrained control design and work by ensuring forward invariance of time-invariant \citebr{7782377} and even time-varying sets \citebr{10886042}. Nevertheless, finding proper CBFs remains an open challenge with few constructive approaches. To this end, Marchesini et. al. \citebr{marchesini2025samplingbasedplanningstlspecifications} focus on linear systems and linear CBFs; their parametrization guarantees STL satisfaction for a subset of STL formulae by leveraging linear programming and sampling-based planning which significantly improves the complexity of their method. Model Predictive Control (MPC) has also been employed for STL \citebr{9655231}.
\par 
In control theoretic terms, reachability analysis, i.e., formally characterizing the ability of a system to reach parts of its state space under input constraints, external disturbances and controllability issues, is at the core of task planning. Computing reachable sets and/or controllers involves the solution of the nonlinear Hamilton-Jacobi Bellman (HJB) Partial Differential Equation and remains a fundamental open problem. Recent approaches include approximations for linear systems \citebr{LIU2023126}, as well as learning-based methods \citebr{9561949,4177130}.
\par 
The connection between LTL/STL and reachability analysis has been explored. In \citebr{10.1007/978-3-030-44051-0_34}, the authors connect a reachability operator with the satisfaction of simple (non-nested) STL fragments. They identify that logical operations between STL operators are equivalent to set operations and propose a control synthesis approach with the property of minimally violating a formula if it is deemed unsatisfiable. Similarly, \citebr{9867813} leverages reachability for control synthesis on linear systems, with the main caveat of considering very simple STL fragments. The authors in \citebr{10886238} employ reachability to synthesize disturbance-robust controllers for STL satisfaction in non-affine systems. Their focus is limited {to} simple operators (always, until, eventually) and logical operations thereof. LTLs have also been explored under HJB reachability. In \citebr{10886233} the authors employ a similar approach to our method, where reachability functions are combined with temporal logic trees and identify similar connections between reachability and temporal logics with \citebr{10.1007/978-3-030-44051-0_34}. However, they also treat non-nested formulae. 
\par 
An important step towards nested formulae was made in \citebr{10388467}, combining a tree representation with backwards-reachable sets, achieving higher level of nesting than previous efforts. More recently, a hierarchical approach including backwards-reachable sets, MILP-based global planning and low-level MPC was proposed \citebr{choi2026signaltemporallogicverification}. 
\par 
The main limitation of the above methods rests on the complexity of the formulae; most are limited either to non-nested formulae with logical operations, or to single nestings of operators. \citebr{10388467} is a notable exception, with more complex nested fragments. However, this work {may yield} conservative solutions and limits the complexity of the formula to a specific form. The aforementioned methods also fail to express repeated formulae; it was identified in \citebr{marchesini2025samplingbasedplanningstlspecifications} that specific nestings necessarily yield repeated satisfactions. The authors proposed a solution for a single nesting, yielding a single level of repetition. 
\par
Handling arbitrary nestings is a challenging problem. However, restricting the form of considered formulae is {an important} limitation, as the power of STL relies on its expressiveness; by limiting the complexity of formulae, this expressiveness is fundamentally impaired. 
In light of the above, this work proposes a novel framework for reachability-based STL. Similar to existing methods, we combine temporal and logic trees to represent STL formulae. Our main contribution rests on defining a novel operator, which we term ``CBF-STL Operator''. In contrast to the existing body of literature, our operator acts on pre-computed reachability value functions for each simple predicate. Then, in order to express STL formulae we propose a series of \textbf{composition rules} on the operator itself. Importantly, we prove that under the proposed nesting rules, the operator-enhanced reachability analysis yields necessary and sufficient conditions for formula satisfaction. 
\par 
In essence, the proposed framework is an extension of previous efforts \citebr{8404080,10.1007/978-3-030-44051-0_34,marchesini2025samplingbasedplanningstlspecifications,10388467} towards bridging the gap between planning and control, by expressing STL satisfaction in {\emph{purely control-theoretic terms (i.e., set invariance)}}.  
The ability of the proposed framework to tackle complex formulae is demonstrated both theoretically, as well as through numerical simulations. To the best of our knowledge, no existing state-feedback control synthesis method is able to handle as complex nestings as the ones presented here. Open-loop STL control synthesis can be achieved via MILP, but this solution is computationally expensive and practical only for shallow levels of nesting.

% MILP is able to compute input signals for dynamical systems; however this is in an open-loop fashion and comes at the cost of high computational complexity.      

\subsubsection{Notation}
Let $\mathbb{R}$ denote the set of real numbers, $\mathbb{R}^n, \mathbb{R}^n_{\geq 0}$ denote the sets of $n$-dimensional real and non-negative vectors respectively. Given two metric spaces $\mathcal{X}, \mathcal{Y}$, let $\fundef{\mathcal{X}}{\mathcal{Y}}$ denote the set of functions that map elements from $\mathcal{X}$ to $\mathcal{Y}$, i.e., $\fundef{\mathcal{X}}{\mathcal{Y}} \triangleq \setdef{F}{ F:\mathcal{X}\rightarrow\mathcal{Y}}$. For $x\in\mathcal{X}, y\in\mathcal{Y}$ let $(x,y) \triangleq \left[ x^\top, y^\top \right]^\top \in \mathcal{X}\times\mathcal{Y}$. Given a set $\mathcal{S}, |\mathcal{S}|$ denotes the cardinality and $\textrm{Pow}(\mathcal{S})$ denotes the power set of $\mathcal{S}$.

\section{Problem Formulation}
Consider the dynamics:
\begin{equation}\label{eq:dynamics}
    \dot{x} = f(x,u),\ x_0 = x(t_0) \in \mathcal{X}, 
\end{equation}
where $\mathcal{X}\subseteq \mathbb{R}^n$ denotes the system's state space, $x,  x_0\in\mathcal{X}$ denote the system's state and initial state at time $t_0 \in \mathbb{R}$, respectively, while $u\in \mathcal{U}$ denotes the system's input. The function $f\in\fundef{\mathcal{X} \times \mathcal{U}}{\mathcal{X}}$ denotes the dynamics of the system and is assumed to be locally Lipschitz. The set $\mathcal{U} \subset \mathbb{R}^m$ is termed the set of admissible inputs and is assumed to be convex. Further denote as $x\in\fundef{\mathbb{R}}{\mathcal{X}}$ the solution to \eqref{eq:dynamics} (trajectory) from the initial condition $x_0$ at time $t_0$ with an overloading of notation (the distinction between point/trajectory will be clear in-context). Consider now an STL \citebr{stl_paper} predicate $\mu$ and its associated predicate function $h\in\fundef{\mathcal{X}}{\mathbb{R}}$ as:
\begin{equation}\label{eq:pred_def}
    \mu :=
    \begin{cases}
        \mathrm{True}, &\ h(x)\geq 0\\
        \mathrm{False}, &\ h(x)< 0
    \end{cases},
\end{equation}
and its associated zero super-level set:
\(
    \mathcal{B}_h=\setdef{x\in\mathcal{X}}{h(x)\geq0},
\)
with boundary $\partial \mathcal{B}_h  =\setdef{x\in\mathcal{X}}{h(x)=0}$. The STL syntax defines an STL formula $\phi$ and is given by 
\begin{equation}\label{eq:stl:formula}
\phi ::= \text{True} \mid \mu \mid \neg\phi \mid T_1 \land T_2 \mid T_1 \, U_{[a,b]} \, T_2 ,
\end{equation}
where $T_1, T_2$ are STL formulae and $a, b \in \mathbb{R}_{\ge 0}$ with $a \le b$.  
The satisfaction relation $(x,t) \models T$ denotes if the signal $x(t)$, satisfies $\phi$ starting from the time $t\in\mathbb{R}_{\geq 0}$. Below we provide the definitions for STL semantics:
\begin{define}\label{def:stl}
For a signal $x \in\fundef{\mathbb{R}}{\mathcal{X}} $, the STL semantics \citebr{stl_paper} are recursively given by:
\begin{equation}\notag 
\begin{split}
(x,t) \models \mu 
&\Leftrightarrow h(x(t)) \ge 0, \\
(x,t) \models \neg\phi
&\Leftrightarrow \neg((x,t)\models\phi), \\
(x,t) \models \phi_1 \land \phi_2 
&\Leftrightarrow (x,t)\models\phi_1 \land (x,t)\models\phi_2, \\
(x,t) \models \phi_1 U_{[a,b]} \phi_2 
&\Leftrightarrow \exists t_1 \in [t+a, t+b] \; \\
\qquad \text{s.t.}\; (x,t_1)&\models\phi_2 \land \;\forall t_2 \in [t,t_1], (x,t_2)\models\phi_1, \\
(x,t) \models F_{[a,b]}\phi 
\Leftrightarrow& \exists t_1 \in [t+a, t+b] \;\text{s.t.}\; (x,t_1)\models\phi, \\
(x,t) \models G_{[a,b]}\phi
\Leftrightarrow& \forall t_1 \in [t+a, t+b], (x,t_1)\models\phi.
\end{split}
\end{equation}
\end{define}
\section{CBF-STL Operator Definition}\label{sec:operator:definition}
\subsection{Preliminaries}
Consider a predicate function $h\in\fundef{\mathcal{X}}{\mathbb{R}}$ \eqref{eq:pred_def}. Given system \eqref{eq:dynamics}, we make the following assumption:
\begin{assumption}\label{assum:hjb}
    Let $x\in\fundef{\mathbb{R}}{\mathcal{X}}$ denote the solution to \eqref{eq:dynamics} from an initial condition $x(t)$ for $t<0$. Assuming $\mathcal{U}$ is bounded, then the value function $V_h\in\fundef{\mathcal{X}\times\mathbb{R}_{\leq0}}{ \mathbb{R}}$ satisfying:
    \begin{equation}\label{eq:vfun:def}
        V_h(x,t) = \underset{u \in \mathcal{U}}{\sup}
        \left\{
            \underset{s \in [t,0]}{\max}\left\{ h(x(s))\right\}
        \right\},
    \end{equation}
    is well-defined and can be computed through the (viscosity) solution to the following HJB equation:
    \begin{equation}\label{eq:hjb}
        \begin{split}
            \frac{\partial V_h}{\partial t} + 
            &\underset{u\in\mathcal{U}}{\max}
            \left\{
                    \left( \nabla_x V_h (x,t) \right)^\top f(x,u)
            \right\}
            = 0 
            \\
            &\textrm{s.t.: }
            V_h(x,0) = h(x), \forall x \in \mathcal{X}.
        \end{split}
    \end{equation}
\end{assumption}
This value function will be employed in the sequel in order to ensure satisfaction of predicates via forward time-invariance of its zero super-level set.
Below we provide a useful result for the value function.
\begin{proposition}\label{prop:valu_fun:monotone}
    The value function $V_h\in\fundef{\mathcal{X}\times\mathbb{R}_{\geq0}}{ \mathbb{R}}$ of Assum. \ref{assum:hjb} is non-increasing w.r.t. its time argument, i.e.: $\forall t_1,t_2 \in \mathbb{R}_{\leq 0}:t_1 \leq t_2 \Rightarrow V(x,t_1) \geq V(x,t_2), \forall x \in \mathcal{X}/\mathcal{B}_h$.
\end{proposition}
\begin{proof}[Sketch]
    Taking any solution to \eqref{eq:dynamics} under an admissible input, and considering two time instances $t_1,t_2 \in \mathbb{R}_{\leq 0}:t_1 \leq t_2 $, we have $[t_1,0]\supset [t_2,0]$. Owing to the supremum in \eqref{eq:vfun:def}, taken over the same solution, this directly implies that $V(x,t_1) \geq V(x,t_2)$. 
    % Consider system \eqref{eq:dynamics} and let $x:\mathbb{R}\times \mathbb{R} \times \mathcal{F}\left( \mathbb{R},\mathbb{R}^m \right)\rightarrow\mathcal{X}$ (evaluated $x(s,t_0,u)\in\mathcal{X}$ - with an overloading of notation for the purposes of this proposition) denote its solution from some $\mathcal{X} \ni x_0 = x(t_0,t_0,u), t_0 < 0$, under an admissible input $u\in\fundef{\mathbb{R}}{\mathcal{U}}$. Let now $\bar{h}:\mathbb{R}\times \mathcal{F}\left( \mathbb{R},\mathcal{U}\right)\rightarrow\mathbb{R}$ denote:
    % $ \bar{h}(t_0,u) = \underset{s \in [t_0,0]}{\max}\left\{ h(x(s,t_0,u))\right\}$.
    % Then, for some $t_1,t_2 \in \mathbb{R}_{\leq 0}:t_1 \leq t_2 $, we have $[t_1,0]\supset [t_2,0]$, and therefore, over the trajectory $x(s,t_0,u)$:
    % $\underset{s \in [t_1,0]}{\max}\left\{ h(x(s,t_1,u))\right\} \geq \underset{s \in [t_2,0]}{\max}\left\{ h(x(s,t_2,u))\right\} \Leftrightarrow \bar{h}(t_1,u) \geq \bar{h}(t_2,u).$
    % Taking the supremum over the admissible controls preserves the inequality: $\underset{u \in \mathcal{U}}{\sup}\left\{ \bar{h}(t_1,u) \right\}\geq\underset{u \in \mathcal{U}}{\sup}\left\{ \bar{h}(t_2,u) \right\} \Leftrightarrow V(x,t_1) \geq V(x,t_2)$,
    % concluding the proof. 
\end{proof}

\subsection{CBF-STL Operator Definition}
\begin{figure}
    \centering
    \includegraphics[trim={0.5cm 1.0cm 0.75cm 0.5cm},clip,width=0.45\textwidth]{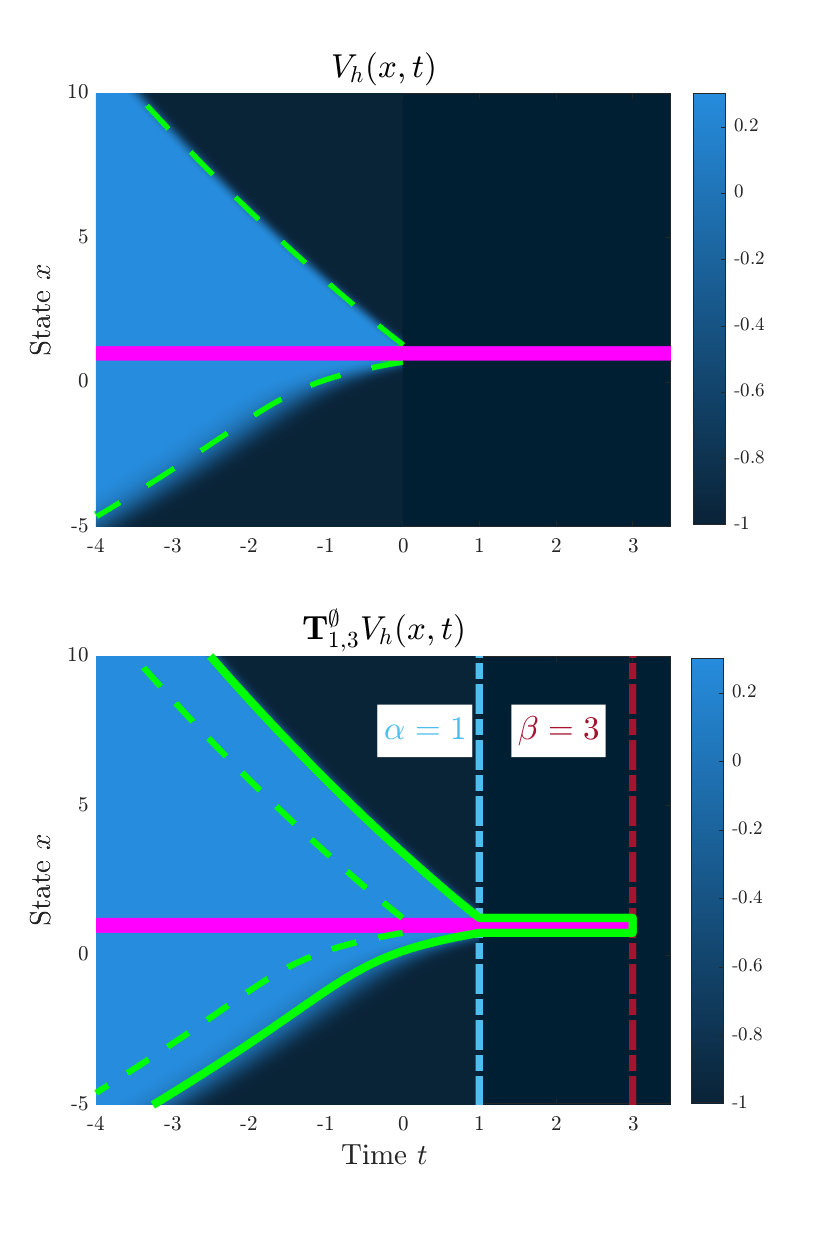}
    \caption{Scaled plot of a value function $V_h(x,t)$ (top) and $\mathbf{T}_{1,3}^{\emptyset}V_h(x,t)$ (bottom) for $t_0 = 0$. The zero super-level set boundary is depicted through green dashed lines for $V_h$ and continuous lines for $\mathbf{T}_{1,3}^{\emptyset}V_h$ . Any state within the light blue-shaded region can reach the zero super-level set of $h$ (magenta).}
    \label{fig:hjb_solution}
\end{figure}
Given a simple predicate $\mu$ and its predicate function $h$, positivity of the value function $V_h$ in \eqref{eq:vfun:def} implies positivity of the corresponding predicate function through the terminal condition in \eqref{eq:hjb} at time $t=0$. However, STL requires modeling the satisfaction time of STL fragments as \textbf{variable} time instances, e.g., for the fragment $F_{\left[\underline{t},\overline{t} \right]}\mu$, $h$ should become positive within the time window $\left[ t_0 + \underline{t},t_0 + \overline{t}\right]$. To model this aspect, we proceed by defining a set of linear functions:
\begin{define}
    Let $l\in\mathbb{N}$, then we define a set of linear functions as: $\mathcal{P}_{\Theta}^l \triangleq \setdef{\alpha_1^\top \tau + \alpha_0}{\alpha_0 \in \mathbb{R}_{\geq 0}, \alpha_1\in { \mathbb{R}^l_{\geq 0}}} \subset \fundef{\Theta}{\mathbb{R}_{\geq 0}}$, {and evaluated as $\alpha(\tau)$}, with $\tau \in \Theta \subset \mathbb{R}^l$ a free parameter.
\end{define}
\par 
Such functions will be designed in the sequel to ensure correspondence between nested operators and fragments. Briefly, $l$ denotes the number of free parameters of the functions {$\alpha, \beta  \in \mathcal{P}_{\Theta}^l$} and $\Theta$ denotes the domain of definition of the former. We can now define the CBF-STL operator through the set of linear functions $\mathcal{P}^l_\Theta$: 
\begin{define}[CBF-STL Operator]\label{def:operator}
    Given a predicate function $h\in\fundef{\mathcal{X}}{\mathbb{R}}$ and its corresponding value function $V_h \in \mathcal{F}\left(\mathcal{X}\times \mathbb{R}_{\leq 0},\mathbb{R}\right)$, the proposed CBF-STL operator\footnote{Operator in this context is interpreted as a mapping between functions.} $\mathbf{T}_{\alpha,\beta}^{\Theta} \in \mathcal{T}$ is defined as: 
    \begin{equation}\label{eq:operator:def}
        \mathbf{T}_{\alpha,\beta}^{\Theta} : 
        \mathcal{F}\left(\mathcal{X}\times \mathbb{R},\mathbb{R}\right) \rightarrow
        \mathcal{F}\left(\mathcal{X}\times \mathbb{R},\mathbb{R}\right),
    \end{equation}
    for $\alpha, \beta \in \mathcal{P}_{\Theta}^l,\ \Theta \subset \mathbb{R}^l\cup\emptyset $ and $\mathcal{T}$ denotes the set of such operators. The CBF-STL operator acts on value functions:
    \begin{equation}\label{eq:operator:action}
        \begin{gathered}
            \mathbf{T}_{\alpha,\beta}^{\Theta}V_h(x,t;t_0) \triangleq
            \\
            \begin{cases}
                V_h(x,t - \alpha(\tau)-t_0), & t  \in \left[t_0,\alpha(\tau)+t_0\right]
                \\
                h(x) , & t\in \left(\alpha(\tau)+t_0,\beta(\tau)+t_0\right]
            \end{cases},
        \end{gathered}
    \end{equation}
    for any $t_0 \in \mathbb{R},\tau \in \Theta$. Furthermore, we define the \textit{identity} operator $\mathcal{T} \ni \mathbf{T}_{0,0}^\emptyset : \mathbf{T}_{0,0}^\emptyset V_{h}(x,t;t_0) = V_h(x,t)$. 
\end{define}
\par
Intuitively, through its action \eqref{eq:operator:action}, the operator creates a ``window of satisfaction'' of the simple predicate $\mu$ through the value function $V_h$. 
More precisely, $\mathbf{T}_{\alpha,\beta}^{\Theta}V_h(x,t;t_0) \geq 0$ implies predicate satisfaction (through $h\geq 0$) during the time window $\left[ \alpha(\tau), \beta(\tau)\right]$ for any choice of $\tau\in\Theta$. Importantly, positivity of $\mathbf{T}_{\alpha,\beta}^{\Theta}V_h(x,t;t_0)$ also ensures that the zero super-level set of $h$ is reachable for all initial times. Hence, designing the functions $\alpha, \beta \in \mathcal{P}^l_\Theta$ can ensure correspondence between super-level set invariance and STL fragment satisfaction. 
\par 
The effect of an instance of the operator is visualized in Fig. \ref{fig:hjb_solution} for a one-dimensional system. Comparing the top and bottom subfigures, the operator $\mathbf{T}_{1,3}^{\emptyset}$ {(i.e., $\alpha = 1,\beta = 3,\Theta = \emptyset$)}, has two effects; 1) transposing $V_h$ in time such that the latter's boundary condition coincides with $t = 1$, and, 2) extending the function to the interval $[1,3]$. If the system stays within $\{V_h \geq 0\}$ then the predicate $\mu$ is always satisfied ($\{h\geq 0\}$) within $[1,3]$. We will demonstrate in the sequel that this corresponds exactly to $G_{\left[1,3\right]}\mu$. 
The fact that the time-window bounds $\alpha,\beta$ are parametrized through belonging to $\mathcal{P}^l_\Theta$ is crucial, since the time of satisfaction of fragments, e.g., $F_{\left[\underline{t},\overline{t} \right]}\mu$, may be a ``free'' (decision) variable. 
% Additionally, by virtue of the terminal condition $V_h(x,0) = h(x)$ in Assum. \ref{assum:hjb}, the function $\mathbf{T}_{\alpha,\beta}^{\Theta}V_h(x,t;t_0) $ inherits (non-)smoothness of $h$ for $t=\alpha(\tau)$.
\par 
Consider a formula $\psi$ consistent with Def. \ref{def:stl} as a result of nesting STL operators over the simple predicate $\mu$ which is associated with the predicate function $h\in\fundef{\mathcal{X}}{\mathbb{R}}$. Further consider the value function $V_h \in \fundef{\mathcal{X}\times\mathbb{R}}{\mathbb{R}}$ associated with $h$, i.e., $V_h$ is given by \eqref{eq:vfun:def}.
Throughout the sequel, for $\alpha,\beta \in \mathcal{P}_{\Theta}^l, \Theta \subset \mathbb{R}^l$, we use the terminology ``the operator $\mathbf{T}_{\alpha,\beta}^\Theta$ corresponds to the fragment $\psi$'', which we define as equivalent to $(\Leftrightarrow)$\footnote{In our definition of the equivalence, we have employed the time interval $\left[ t_0,t_0 +\beta(\tau)\right]$. The omission of $\alpha(\tau)$ is a crucial technical point, as the state should belong to the superlevel set of $\mathbf{T}_{\alpha(\tau),\beta(\tau)}^\Theta V_h\left(x\left(t\right),t;t_0\right)$ for any initial time $t_0$; otherwise the set $\{ h \geq 0\}$ is necessarily not reachable from the initial time $t_0$ at $t_0 + \alpha(\tau)$. {It will become clear in the sequel that, through their design, $\alpha,\beta \geq 0$.}}:
\begin{equation}
    \notag 
    \begin{gathered}
        \forall t_0 \in \mathbb{R}: \exists \tau(t_0) \in \Theta, \forall t \in \left[ t_0,t_0 +\beta(\tau(t_0))\right]:
        \\
        \mathbf{T}_{\alpha(\tau),\beta(\tau)}^\Theta V_h\left(x\left(t\right),t\right) \geq 0 \Leftrightarrow  (x,t_0)\models \psi,
    \end{gathered}
\end{equation}
% which is interpreted as follows:
% \begin{equation}
%     \begin{split}
%         &\mathrm{\textbf{the operator }} \mathbf{T}_{\alpha,\beta}^\Theta \mathrm{\textbf{ corresponds to the fragment }} \psi 
%         \\
%         &\qquad \qquad \qquad \qquad \qquad \qquad \Updownarrow
%         \\
%         &\forall \tau \in \Theta, \forall t \in \left[ t_0,t_0 +\beta(\tau)\right]:
%         \mathbf{T}_{\alpha(\tau),\beta(\tau)}^\Theta V_h\left(x\left(t\right),t\right) \geq 0
%         \\
%         & \qquad  \qquad \qquad \qquad \Leftrightarrow  (x,t_0)\models \psi,
%     \end{split}
%     \notag 
% \end{equation}
with $x\in\fundef{\mathbb{R}}{\mathcal{X}}$ denoting the trajectory of system \eqref{eq:dynamics} from $x(t_0) \in \mathcal{X}$.
This implies that:
\begin{equation}\label{eq:corr:toshow}
    \begin{gathered}
        \forall t_0 \in \mathbb{R}:\exists \tau(t_0) \in \Theta, \forall t \in \left[t_0  + \alpha(\tau(t_0)), t_0 + \beta(\tau(t_0))\right]:
        \\
        h(x(t))\geq 0  \Leftrightarrow (x,t_0) \models \psi.
    \end{gathered}
\end{equation}
According to Eq. \eqref{eq:corr:toshow}, if the operator $\mathbf{T}_{\alpha,\beta}^\Theta$ \textbf{corresponds} to the fragment $\psi$, this means that the simple predicate $\mu$ (the one that produces $\psi$ via nesting STL operators) is satisfied over the time window $[\alpha,\beta]$, since by definition $h(x(t))\geq 0 \Leftrightarrow (x,t) \models \mu$.
Concurrently, if $\psi$ comes from a single STL operator, i.e., $\psi = T_{\left[\underline{t},\overline{t}\right]}\mu,\ T \in \{G,F\}$, and $\mathbf{T}_{\alpha,\beta}^\Theta$ corresponds to $\psi$, then we \textbf{associate} the CBF-STL operator $\mathbf{T}_{\alpha,\beta}^\Theta$ with the STL operator $T_{\left[\underline{t},\overline{t}\right]}$ (denoted as $T_{\left[\underline{t},\overline{t}\right]} : \mathbf{T}_{\alpha,\beta}^\Theta$), where $\alpha,\beta \in \mathcal{P}_{\Theta}^l, \Theta \subset \mathbb{R}^l$ depend on $\underline{t},\overline{t}$ and will be chosen appropriately in the sequel to ensure consistency between STL operators and the proposed CBF-STL operator. We now prove the following useful properties:
\begin{proposition}\label{prop:operator:sequence}
    Given the CBF-STL operators \eqref{eq:operator:def} $\mathbf{T}_{\alpha_i ,\beta_i}^{\Theta_i}\in\mathcal{T},i\in\{1,2\}$, if for the time windows: $\alpha_1 (\tau_1)\leq \alpha_2(\tau_2),\ \forall \tau_1\in\Theta_1,\tau_2\in\Theta_2$, then the following holds:
    $\mathbf{T}_{\alpha_1 ,\beta_1}^{\Theta_1}V_h(x,t) \geq 0 \Rightarrow \mathbf{T}_{\alpha_{2} ,\beta_{2}}^{\Theta_{2}}V_h(x,t)\geq 0$,
    for any $x\in\mathcal{X}, t\in\left[ 0, \min \left\{ \alpha_2(\tau_2) , \beta_1(\tau_1) \right\}\right],\forall \tau_1 \in \Theta_1,\tau_2\in\Theta_2$.
\end{proposition}
\begin{proof}
    See Appendix.
\end{proof}
{
\begin{proposition}\label{prop:operator:reachable}
    Given the CBF-STL operators \eqref{eq:operator:def} $\mathbf{T}_{\alpha ,\beta}^{\Theta}\in\mathcal{T}$ and the function $h\in\fundef{\mathcal{X}}{\mathbb{R}}$ and $x\in\fundef{\mathbb{R}}{\mathcal{X}}$, then the following holds:
    \begin{equation}
        \begin{gathered}
            \forall t \in\left[ t_0 + \alpha,t_0 + \beta\right]:
           \mathbf{T}_{\alpha,\beta}^{\Theta}V_h(x,t;t_0)\geq 0
            \Leftrightarrow
            \\
            \forall t \in\left[ t_0,t_0 + \beta\right]:
            \mathbf{T}_{\alpha,\beta}^{\Theta}V_h(x,t;t_0) \geq 0.
        \end{gathered}
    \end{equation}
\end{proposition}
\begin{proof}
    We only need to consider the interval $t \in\left[ t_0,t_0 + \alpha\right]$.
    The equation $\forall t \in\left[ t_0 + \alpha,t_0 + \beta\right]:h(x(t))\geq 0$ implies that the set $\setdef{x\in\mathcal{X}}{h(x)\geq 0}$ is reachable from time $t_0$ at time $t_0 + \alpha$, i.e., $V(x(t),t - \alpha - t_0)\geq 0 \overset{\eqref{eq:operator:action}}{\Rightarrow} \forall t \in\left[ t_0,t_0 + \alpha\right]:
            \mathbf{T}_{\alpha,\beta}^{\Theta}V_h(x,t;t_0) \geq 0$, concluding the proof. 
\end{proof}
}

\section{CBF-STL Operator Composition}\label{sec:operator:application}
In this section we develop the theoretical connection between the CBF-STL operator and STL formulae satisfaction. Overall, this requires designing the functions $\alpha,\beta \in \mathcal{P}^l_{\Theta}$ and the corresponding domain of the parameters $\tau \in \Theta$.

\subsection{Simple STL Operators}\label{sec:operator:application:subsec:simple}
% We begin by demonstrating the CBF-STL operator on simple STL operators. 
Consider the simple predicate $\mu$ with its predicate function $h\in\fundef{\mathcal{X}}{\mathbb{R}}$ and the value function $V_h \in \fundef{\mathcal{X}\times\mathbb{R}}{\mathbb{R}}$. In this subsection the CBF-STL operators associated with the always and eventually STL operators are developed. {To distinguish between the two cases, we will employ the notation $\mathbf{\Gamma},\mathbf{\Phi}$ for the CBF-STL operators corresponding to always and eventually STL operators respectively.}
% \subsubsection{Always STL Operator}
% Consider the always operator $G_{\left[\underline{t},\overline{t}\right]}, \ 0 \leq \underline{t} \leq \overline{t}$ defined in Def. \ref{def:stl}. 
% Then, the CBF-STL operator $\mathbf{T}_{\underline{t} ,\overline{t}}^{\emptyset} \in \mathcal{T}$ corresponds to the fragment $G_{\left[\underline{t},\overline{t}\right]}\mu $:
\begin{proposition}\label{prop:always}
    Given the always operator $G_{\left[\underline{t},\overline{t}\right]}, \ 0 \leq \underline{t} \leq \overline{t}$ as per Def. \ref{def:stl} and the function $\mathbf{\Gamma}_{\underline{t} ,\overline{t}}^{\emptyset}V_h(x,t)$ given by \eqref{eq:operator:action}, then the operator {$\mathbf{\Gamma}_{\underline{t} ,\overline{t}}^{\emptyset} \in \mathcal{T}$} corresponds to the fragment $G_{\left[\underline{t},\overline{t}\right]}\mu $. 
\end{proposition}
\begin{proof}
    See Appendix.
\end{proof}
% \subsubsection{Eventually STL Operator}
% Consider the eventually operator $F_{\left[\underline{t},\overline{t}\right]}, \ 0 \leq \underline{t} \leq \overline{t}$ defined in Def. \ref{def:stl}, then the CBF-STL operator $\mathbf{T}_{\underline{t} + \tau ,\underline{t} + \tau}^{\left[ 0, \overline{t} - \underline{t}\right]} \in \mathcal{T}$ corresponds to the fragment $F_{\left[\underline{t},\overline{t}\right]}\mu$:
\begin{proposition}\label{prop:eventually}
    Given the eventually operator $F_{\left[\underline{t},\overline{t}\right]}, \ 0 \leq \underline{t} \leq \overline{t}$ as per Def. \ref{def:stl} and the function $ \mathbf{\Phi}_{\underline{t}+\tau ,\underline{t}+\tau }^{\left[ 0, \overline{t} - \underline{t}\right]}V_h(x,t)$ given by \eqref{eq:operator:action}, then the operator {$\mathbf{\Phi}_{\underline{t}+\tau ,\underline{t}+\tau }^{\left[ 0, \overline{t} - \underline{t}\right]}\in \mathcal{T}$} where {$\mathcal{P}_{\left[ 0, \overline{t} - \underline{t}\right]}^1\ni\alpha(\tau) = \beta(\tau) = \underline{t}+\tau $}, corresponds to the fragment $F_{\left[\underline{t},\overline{t}\right]}\mu $. 
\end{proposition}
\begin{proof}
    See Appendix.
\end{proof}
% \begin{remark}\label{rem:eventually:logic}
%     By our design of the functions $\alpha,\beta $ in Prop. \ref{prop:eventually}, the lower and upper bounds $\alpha(\tau), \beta(\tau)$ coincide for the eventually STL operator. Therefore, in the definition of eventually (see Def. \ref{def:stl}), the ``$\exists$'' logical condition can be replaced with ``$\forall$'', as the parametrization of the interval $\left[\underline{t},\overline{t}\right]$ through the variable $\tau \in \left[ 0, \overline{t} - \underline{t}\right]$ suffices to model the variable instance of satisfaction of the corresponding predicate (which is essentially the role of the ``$\exists$'' logical condition). This enables providing proofs in the sequel without distinguishing between nestings of eventually/always STL operators.     
% \end{remark}
% \subsubsection{Until STL Operator}\label{subsubsec:until:op}
Consider the simple predicates $\mu_1,\mu_2$  with the associated predicate functions $h_1,h_2 \in \mathcal{F}(\mathcal{X},\mathbb{R})$ \eqref{eq:pred_def} and the until operator $U_{\left[\underline{t},\overline{t}\right]}, \ 0 \leq \underline{t} \leq \overline{t}$ defined in Def. \ref{def:stl}. Then the STL fragment $\mu_1 U_{\left[\underline{t},\overline{t}\right]}\mu_2$ can be recast as:
\begin{equation}\label{eq:until:decomp}
    \mu_1 U_{\left[\underline{t},\overline{t}\right]}\mu_2 = 
    \left( 
        G_{\left[  0, \underline{t} + \tau\right]}\mu_1
    \right)
    \wedge
    \left( 
        F_{\left[ \underline{t} + \tau, \underline{t} + \tau \right]}\mu_2
    \right),
\end{equation}
for $\tau \in \left[ 0, \overline{t} - \underline{t}\right]$. Then, similar to Props. \ref{prop:always}, \ref{prop:eventually}, the following CBF-STL operators are constructed: 
\(
G_{\left[ 0, \underline{t} + \tau\right]}\mu_1: \mathbf{\Gamma}_{0, \underline{t} + \tau}^{\left[ 0, \overline{t} - \underline{t}\right]}V_{h_1}(x,t;t_0) \in \fundef{\mathcal{X}\times\mathbb{R}}{\mathbb{R}}
\)
and
\(
F_{\left[  \underline{t} + \tau, \underline{t} + \tau\right]}\mu_2: \mathbf{\Phi}_{\underline{t} + \tau,\underline{t} + \tau}^{\left[ 0, \overline{t} - \underline{t}\right]}V_{h_2}(x,t;t_0) \in \fundef{\mathcal{X}\times\mathbb{R}}{\mathbb{R}}.
\)

\begin{proposition}\label{prop:until}
        Given the simple predicates $\mu_2,\mu_2$ with the associated predicate functions $h_1,h_2 \in \mathcal{F}(\mathcal{X},\mathbb{R})$ \eqref{eq:pred_def} and the until operator $U_{\left[\underline{t},\overline{t}\right]}, \ 0 \leq \underline{t} \leq \overline{t}$ Def. \ref{def:stl}, the following holds: $\left( \mathbf{\Gamma}_{0, \underline{t} + \tau}^{\left[ 0, \overline{t} - \underline{t}\right]}V_{h_1}(x,t;t_0) \geq 0 \right) \wedge \left( \mathbf{\Phi}_{\underline{t} + \tau,\underline{t} + \tau}^{\left[ 0, \overline{t} - \underline{t}\right]}V_{h_2}(x,t;t_0) \geq 0 \right) \Rightarrow$ $(x,t) \models \mu_1 U_{\left[\underline{t},\overline{t}\right]}\mu_2$.
\end{proposition}
\begin{proof}
    The proof is identical to Props. \ref{prop:always}, \ref{prop:eventually}. 
\end{proof}
{
In Props \ref{prop:always}, \ref{prop:until}, there are two distinct operator cases for the \emph{always} operator, namely $G_{[\underline{t},\overline{t}]}, \underline{t},\overline{t}\in\mathbb{R}$ and $G_{[\alpha(\tau),\beta(\tau)]}, \alpha,\beta\in\mathcal{P}_l^\mathbb{R}$. The first case is a degenerate version of the second one, hence in the sequel we focus on the second case. 
}

\subsection{Nested STL Operators}
We now establish the CBF-STL operator composition rules for nested formulae. Throughout the following proofs, we begin with simpler nestings and leverage the association between the always/eventually/until operators and the CBF-STL operators of the previous subsection to inductively map nested STL formulae to reachability-based functions. 
{We begin by providing two existing nesting results \cite{marchesini2025samplingbasedplanningstlspecifications}) in a form consistent to our operator framework}. 
{
\begin{proposition}[(Always-Always)]\label{thm:nesting:always}
    Consider the fragment $\psi$ and, then, for the nested operator:
   \begin{equation}
      \begin{gathered}
        \forall \tau \in \fundef{\mathbb{R}}{\Theta}, \tau' \in \fundef{\mathbb{R}}{\Theta'}:
        \\
          (x,t_0) \models G_{[ \alpha'(\tau'),\beta'(\tau')]}G_{[ \alpha(\tau), \beta(\tau)]}\psi
          \Leftrightarrow
          \\
          (x,t_0) \models G_{[ \alpha'(\tau')+\alpha(\tau),\beta'(\tau')+ \beta(\tau)]}\psi,
      \end{gathered}
   \end{equation}
   with $\alpha,\beta \in \mathcal{P}_l^\Theta : \alpha(\tau)\leq\beta(\tau), \forall \tau \in \Theta,\alpha',\beta' \in \mathcal{P}_{l'}^{\Theta'}:\alpha'(\tau') \leq \beta'(\tau') , \forall \tau' \in \Theta'$.
\end{proposition}
\begin{proof}
   We have $\forall \tau \in \fundef{\mathbb{R}}{\Theta}, \tau' \in \fundef{\mathbb{R}}{\Theta'}$:
   \begin{equation}
       \begin{gathered}
            (x,t_0) \models G_{[ \alpha'(\tau'),\beta'(\tau')]}G_{[ \alpha(\tau), \beta(\tau)]}\psi \Leftrightarrow
            \\
            \forall t' \in
            \left[ 
                t_0 + \alpha'(\tau'), t_0 + \beta'(\tau')
            \right]
            :
            (x,t') \models G_{[ \alpha(\tau), \beta(\tau)]}\psi \Leftrightarrow
            \\
            \forall t' \in
            \left[ 
                t_0 + \alpha'(\tau'), t_0 + \beta'(\tau')
            \right]
            :
            \forall t \in
            \left[ 
                t' + \alpha(\tau), t' + \beta(\tau)
            \right]:
            \\
            (x,t) \models \psi \Leftrightarrow
            \\
             \forall t \in
            \left[ 
                 t_0 + \alpha'(\tau') + \alpha(\tau), 
                 t_0 + \beta'(\tau') + \beta(\tau)
            \right]:
            (x,t) \models \psi 
            \\
            \Leftrightarrow 
            (x,t_0) \models G_{[ \alpha'(\tau')+\alpha(\tau),\beta'(\tau')+ \beta(\tau)]}\psi,
            \notag 
       \end{gathered}
   \end{equation}
   concluding the proof.
\end{proof}
\begin{corollary}\label{cor:nesting:always:ops}
    The nesting rule for the corresponding operators of Prop. \ref{thm:nesting:always} is:
    \begin{equation}
        \mathbf{\Gamma}_{\alpha'(\tau),\beta'(\tau')}^{\Theta'}
        \mathbf{\Gamma}_{\alpha(\tau),\beta(\tau)}^{\Theta} = 
        \mathbf{\Gamma}_{\alpha'(\tau)+\alpha(\tau),\beta'(\tau')+\beta(\tau)}^{\Theta' \times \Theta}.
    \end{equation}
\end{corollary}
}
{
\begin{proposition}[(Eventually-Eventually)]\label{thm:nesting:eventually}
    Consider the fragment $\psi$ and, then, for the nested operator:
   \begin{equation}
      \begin{gathered}
       \forall \tau \in \fundef{\mathbb{R}}{\Theta}, \tau' \in \fundef{\mathbb{R}}{\Theta'}:
        \\
          (x,t_0) \models F_{[ \alpha'(\tau'),\beta'(\tau')]}F_{[ \alpha(\tau), \beta(\tau)]}\psi
          \Leftrightarrow
          \\
          (x,t_0) \models F_{[ \alpha'(\tau')+\alpha(\tau),\beta'(\tau')+ \beta(\tau)]}\psi,
      \end{gathered}
   \end{equation}
   with $\alpha,\beta \in \mathcal{P}_l^\Theta : \alpha(\tau)\leq\beta(\tau), \forall \tau \in \Theta,\alpha',\beta' \in \mathcal{P}_{l'}^{\Theta'}:\alpha'(\tau') \leq \beta'(\tau') , \forall \tau' \in \Theta'$.
\end{proposition}
\begin{proof}
   The proof is almost identical to Prop. \ref{thm:nesting:always}. 
\end{proof}
\begin{corollary}\label{cor:nesting:eventually:ops}
    The nesting rule for the corresponding operators of Prop. \ref{thm:nesting:eventually} is:
    \begin{equation}
        \mathbf{\Phi}_{\alpha'(\tau),\beta'(\tau')}^{\Theta'}
        \mathbf{\Phi}_{\alpha(\tau),\beta(\tau)}^{\Theta} = 
        \mathbf{\Phi}_{\alpha'(\tau)+\alpha(\tau),\beta'(\tau')+\beta(\tau)}^{\Theta' \times \Theta}.
    \end{equation}
\end{corollary}
By applying Props. \ref{thm:nesting:always}, \ref{thm:nesting:eventually}, any repeated instances of $FF,GG$ nestings within a formula can be dropped, leading to a formula that contains only nestings of $GF,FG$. We now move on to these cases.
}
{
\begin{proposition}[(Eventually Always)]\label{prop:nesting:event:alw}
    Consider the predicate $\mu$ and its corresponding predicate function $h\in\mathcal{F}(\mathcal{X},\mathbb{R})$ \eqref{eq:pred_def}. 
    Define for $F_{[ \underline{t}', \overline{t}']}G_{\left[\underline{t},\overline{t}\right]}\mu$ with $\underline{t}, \overline{t} \in \mathbb{R}:0\leq \underline{t}\leq \overline{t},\underline{t}', \overline{t}' \in \mathbb{R}:0\leq \underline{t}'\leq \overline{t}'$:
    \begin{subequations}\label{thm:nesting:event:equiv}
        \begin{equation}\label{thm:nesting:event:equiv:1}
        \begin{gathered}
            \exists \tau(t_0) \in \fundef{\mathbb{R}}{\Theta}:
            \\
            \forall t \in\left[ t_0,t_0 + \beta(\tau(t_0))\right]:
            \mathbf{T}_{\alpha,\beta}^{\Theta}V_h(x,t;t_0) \geq 0,
        \end{gathered}
        \end{equation}
        \begin{equation}\label{thm:nesting:event:equiv:2}
        \begin{gathered}
            (x,t_0) \models 
            F_{[ \underline{t}', \overline{t}']}G_{\left[\underline{t},\overline{t}\right]}\mu.
        \end{gathered}
        \end{equation}
    \end{subequations}
    with: 
    \begin{equation}\label{eq:thm:nesting:event:params}
       \begin{split}
            \alpha\left({\tau}\right) &= \underline{t}+\underline{t}'+\tau \in \mathcal{P}_{\Theta}^{1}, \\
            \beta\left({\tau}\right) &= \overline{t}+\underline{t}'+\tau \in \mathcal{P}_{\Theta}^{1},\\
            \Theta &= \left[ 0, \overline{t}' - \underline{t}'\right].
       \end{split}
    \end{equation}
    Then \eqref{thm:nesting:event:equiv:1} $\Leftrightarrow$\eqref{thm:nesting:event:equiv:2}.
\end{proposition}
\begin{proof}
    See Appendix.
\end{proof}
\begin{corollary}\label{cor:nesting:eventually:always:ops}
    The nesting rule for the corresponding operators of Prop. \ref{prop:nesting:event:alw} is:
    \begin{equation}
        \mathbf{\Phi}_{\underline{t}'+ \tau',\underline{t}'+ \tau'}^{[0,\overline{t}' -\underline{t}']}
        \mathbf{\Gamma}_{\underline{t},\overline{t}}^{\emptyset} = 
        \mathbf{\Gamma}_{\underline{t}'+ \tau'+\underline{t},\underline{t}'+ \tau'+\overline{t}}^{[0,\overline{t}' -\underline{t}']}.
    \end{equation}
\end{corollary}
The next result, nesting always with eventually, has also been considered in \cite{marchesini2025samplingbasedplanningstlspecifications}, where it is noted that it may result in the repeated satisfactions of a formula through recursion. The main difference with our result rests on the fact that \cite{marchesini2025samplingbasedplanningstlspecifications} pre-determined the number of repetitions of the formula. In contrast, in our case it is left as a ``free'' parameter, \textbf{enabling the equivalence between STL operators and CBF-STL ones}.    
\begin{thm}[(Always Eventually)]\label{thm:nesting:alw:event:alw}
    Consider the predicate $\mu$ and its corresponding predicate function $h\in\mathcal{F}(\mathcal{X},\mathbb{R})$ \eqref{eq:pred_def}. 
    Let the fragment $\psi$ correspond to the operator $\mathbf{T}_{\alpha,\beta}^{\Theta} \in\mathcal{T}, \alpha,\beta\in \mathcal{P}_l^\Theta$ according to Props. \ref{prop:always}, \ref{prop:eventually}, \ref{prop:until}, i.e.:
    \begin{equation}\label{thm:nesting:alw:event:alw:eq:psi}
        \begin{gathered}
            (x,t_0)\models \psi 
            \Leftrightarrow
            \\
            \forall 
            t\in
            \left[ 
                t_0 + \alpha(\tau), t_0 + \beta(\tau)
            \right]:
            h(t_0)\geq 0 \overset{\mathrm{Prop. }\ref{prop:operator:reachable}}{\Leftrightarrow}
            \\
            \forall t\in
              \left[ 
                t_0, t_0 + \beta(\tau)
            \right]:
            \mathbf{T}_{\alpha,\beta}^{\Theta}V_h(x,t;t_0)\geq 0.
        \end{gathered}
    \end{equation}
    Then for the nested fragment $\phi = G_{[\alpha''(\tau''),\beta''(\tau'')]} F_{[\underline{t}',\overline{t}']}\psi$ with $\alpha'', \beta''\in\fundef{\mathbb{R}}{\mathbb{R}}$, $0 \leq \underline{t}'\leq \overline{t}'$, define:
    \begin{subequations}\label{eq:thm:nesting:alw:event:alw}
        \begin{equation}\label{eq:thm:nesting:alw:event:alw:1}
        \begin{gathered}
            \forall J\in\{1,\cdots,\bar{J}\}:
            \\
            \exists \tau_J' \in \fundef{\mathbb{R}}{\Theta}:
            \forall t \in\left[ t_0,t_0 + \beta_{J}\right]:
            \\
            \underset{j \in \{1,\cdots,{J}\}}{\min} 
            \left\{
                \mathbf{T}_{\alpha_j,\beta_j}^{\Theta}V_h(x,t;t_0) 
            \right\}    
                \geq 0,
        \end{gathered}
        \end{equation}
        \begin{equation}\label{eq:thm:nesting:alw:event:alw:2}
        \begin{gathered}
            (x,t_0) \models \phi 
        \end{gathered}
        \end{equation}
    \end{subequations}
    with: 
    \begin{equation}\label{eq:thm:nesting:alw:event:alw:params}
       \begin{gathered}
            \mathcal{P}_{\Theta}^{J}\ni\alpha_J\left({\tau}_J\right) = 
            \alpha''(\tau'') + \sum_{j=1}^{J}\tau_j' + \underline{t}' + \alpha(\tau), \\
            \mathcal{P}_{\Theta}^{J}\ni \beta_J\left({\tau}_J\right) = 
            \\
            \min
            \left\{
                \beta''(\tau'')+ \underline{t}' + \beta(\tau),
                \alpha''(\tau'') + \sum_{j=1}^{J}\tau_j' + \underline{t}' + \beta(\tau)
            \right\},
            \\
            \Theta_J= \left[ 0, \overline{t}' - \underline{t}'\right]^J,
       \end{gathered}
    \end{equation}
    and $\bar{J}\in\mathbb{N}$ is the final time of repetition of the inner formula:
    \begin{equation}\label{eq:thm:nesting:alw:event:alw:terminal}
        \begin{gathered}
            \bar{J} = \underset{J \in \mathbb{N}}{\max}\{J\}
            \\
            \mathrm{s.t.: }
            \exists \tau_{J-1}'\in\Theta:\sum_{j=1}^{J-1}\tau_j' \leq \beta''(\tau'') - \alpha''(\tau''),
        \end{gathered}
    \end{equation}
    Then \eqref{eq:thm:nesting:alw:event:alw:1} $\Leftrightarrow$\eqref{eq:thm:nesting:alw:event:alw:2}.
\end{thm}
\begin{proof}
    See Appendix. 
\end{proof}
The next theorem goes beyond existing nesting results, by further nesting $GF$ with another instance of $GF$.

\begin{thm}[(Nested Always Eventually)]\label{thm:nested:nested}
    Consider a nested fragment $\psi$ corresponding to the recursive operator of Thm. \ref{thm:nesting:alw:event:alw}, i.e.:
        \begin{equation}\label{eq:thm:nested:nested:psi0}
            \begin{gathered}
                (x,t_0) \models \psi_0 \Leftrightarrow
                \forall J_0\in\{1,\cdots,\bar{J}_0\}:
                \\
                \exists \tau_{J_0}' \in \fundef{\mathbb{R}}{\Theta_0}:
                \forall t \in\left[ t_0,t_0 + \beta_{J_0}\right]:
                \\
                \underset{j \in \{1,\cdots,{J}_0\}}{\min} 
                \left\{
                    \mathbf{T}_{\alpha_j,\beta_j}^{\Theta}V_h(x,t;t_0) 
                \right\}    
                    \geq 0.
            \end{gathered}
        \end{equation}
    Then for the nested fragment $\psi_1 = G_{[\alpha_1''(\tau_1''), \beta_1''(\tau_1'')]}F_{[\underline{t}_1',\overline{t}_1']}\psi_0$, $\forall J_0 \in \{1,\cdots,\bar{J}_0\}, \forall J_1 \in \{1,\cdots,\bar{J}_1\}$ for $\tau^1_{j_1} \in\left[ 0,\overline{t}'-\underline{t}'\right],\forall j_1 \in \{1,\cdots,J_1\}$ let:
    \begin{equation}\label{eq:thm:nesting:nesting:params}
        \begin{gathered}
            \alpha_{J_0,J_1} = 
            \alpha_1''(\tau_1'') + \underline{t}_1' + \alpha_{J_0} + \sum_{j_1 = 1}^{{J}_1}\tau^1_{j_1} 
            \\
            \beta_{J_0,J_1} = 
            \\
            \min\left\{
                 \beta_1''(\tau_1'') + \underline{t}_1' + \beta_{J_0},
                \alpha_1''(\tau_1'') + \underline{t}_1' + \beta_{J_0} + \sum_{j_1 = 1}^{{J}_1}\tau^1_{j_1}
            \right\},
        \end{gathered}
    \end{equation}
    where $\bar{J}_1\in\mathbb{N}$ denotes the final time of repetition of the inner formula:
    \begin{equation}\label{eq:thm:nesting:nested:terminal}
        \begin{gathered}
            \bar{J}_1 = \underset{J_1 \in \mathbb{N}}{\max}\{J_1\}
            \\
            \mathrm{s.t.: }
            \exists \tau^1_{J_1-1}\in\Theta:\sum_{j_1=1}^{J_1-1}\tau^1_{j_1} \leq \beta_1''(\tau_1'') - \alpha_1''(\tau_1'').
        \end{gathered}
    \end{equation}
    Then for:
    \begin{subequations}
        \begin{equation}\label{eq:thm:nesting:nested:1}
        \begin{gathered}
            \forall J_1 \in \{1,\cdots,\bar{J}_1\},
            \exists \tau_{J_1}^1 \in\fundef{\mathbb{R}}{\Theta_1}:
            \\
            \forall J_0 \in \{1,\cdots,\bar{J}_0\},
            \exists \tau_{J_0}^0\in\fundef{\mathbb{R}}{\Theta_0}
            :
            \\
            \underset{j_1\in\{1,\cdots \bar{J}_1\}}{\min}
            \left\{
                \underset{j_0\in\{1,\cdots \bar{J}_0\}}{\min}
                \left\{
                    \mathbf{T}_{\alpha_{J_0,J_1},\beta_{J_0,J_1}}^{\Theta_{0,1}}V_h(x,t;t_0)
                \right\}
            \right\}
            \geq 0
        \end{gathered}
    \end{equation}
        \begin{equation}\label{eq:thm:nesting:nested:2}
        \begin{gathered}
            (x,t_0) \models \phi,
        \end{gathered}
    \end{equation}
    \end{subequations}
    it holds that \eqref{eq:thm:nesting:nested:1} $\Leftrightarrow$\eqref{eq:thm:nesting:nested:2}.
\end{thm}
\begin{proof}
    See Appendix.
\end{proof}
By repeatedly applying Thm. \ref{thm:nested:nested}, further nestings can be acquired for fragments of the form $G_{[\underline{t}',\overline{t}']}F_{[\underline{t},\overline{t}]}\psi$. 
\begin{corollary}\label{cor:nested:nested:ops}
    For the always-eventually operator pair of Thm. \ref{thm:nesting:alw:event:alw}, we can define a pair of operators:
    \begin{equation}
        G_{\left[\alpha'',\beta''\right]}F_{\left[\underline{t}',\overline{t}'\right]}:
        \mathbf{\Gamma}_{\alpha'',\beta''}^{\Theta''}\mathbf{\Phi}_{\underline{t}'+\tau',\underline{t}'+\tau'}^{\left[ 0,\overline{t}' -\underline{t}'\right]},
    \end{equation}
    where the nesting rule produces a \emph{collection} of operators:
    \begin{equation}
        \begin{gathered}
             \mathbf{\Gamma}_{\alpha'',\beta''}^{\Theta''}\mathbf{\Phi}_{\underline{t}'+\tau',\underline{t}'+\tau'}^{\left[ 0,\overline{t}' -\underline{t}'\right]}
             =
             \left\{
                \mathbf{\Gamma}_{\alpha_J,\beta_J}^{\Theta_J}, J \in \{1,\cdots, \bar{J}\}
             \right\},
        \end{gathered}
    \end{equation}
    where $\alpha_J,\beta_J, \Theta_J$ are given by \eqref{eq:thm:nesting:alw:event:alw:params}, by setting $\alpha = \beta = 0,\Theta = \emptyset$.
\end{corollary}
\begin{remark}
    Note that Cor. \ref{cor:nested:nested:ops} essentially covers both Thms. \ref{thm:nesting:alw:event:alw}, \ref{thm:nested:nested} by virtue of Eqs. \eqref{eq:thm:nesting:alw:event:alw:params}, \eqref{eq:thm:nesting:nesting:params} being identical, by noticing that $\alpha,\beta$ in \eqref{eq:thm:nesting:alw:event:alw:params} appear exactly like $\alpha_{J_0}, \beta_{J_0}$  in \eqref{eq:thm:nesting:nesting:params}. Further nestings of $GF$ maintain this equivalence through $\alpha_{J_0,J_1,J_2\cdots}, \beta_{J_0,J_1,J_2\cdots}$, where each additional nesting with $GF$ adds another index, another free variable, and another possible repetition.  
\end{remark}
Consider now any fragment $\phi$ containing the STL operators $G,F$. Note that by virtue of parameterizing the always operator through $G_{\left[ \alpha,\beta\right]}, \alpha,\beta\in\mathcal{P}_1^\Theta$, we capture both the ``original'' always operator as well as its modified version in the decomposition of the until operator -- see Prop. \ref{prop:until}--.
By applying Props. \ref{thm:nesting:always}, \ref{thm:nesting:eventually} $\phi$ can be turned into a form with no repeated nestings of $F,G$ respectively, denoted by $\phi'$. Owing to Props \ref{thm:nesting:always}, \ref{thm:nesting:eventually}, $(x,t_0) \models \phi\Leftrightarrow(x,t_0)\models \phi'$, and henceforth $\phi'$ will be referred to as a simplified fragment. Then, there are two cases for $\phi'$: 1) If the number of STL operators is even, then the fragment takes the form $GF(GF(\cdots GF(GF\mu)))$ and applying Thm. \ref{thm:nesting:alw:event:alw} for $\mathbf{T}_{\alpha,\beta}^{\Theta} = \mathbf{T}_{0,0}^\emptyset$ (that is, the identity operator in Def. \ref{eq:operator:def}), followed by repeated application of Thm. \ref{thm:nested:nested} yields the CBF-STL operator corresponding to $\phi'$. 2) If the number of STL operators is odd, then there are two more cases: 2a) $\phi' = GF(GF(\cdots GF(GF(G\mu))))$, which is exactly covered by the repeated application of Thms. \ref{thm:nesting:alw:event:alw}, \ref{thm:nested:nested}, and 2b)  $\phi' = F(GF(GF(\cdots GF(GF\mu))))$. This last case is treated in the following theorem:
}
{
\begin{thm}\label{thm:nested:nested:eventually}
    Consider a fragment $\psi$ that corresponds to $N\in\mathbb{N}$ nestings of $GF$ operators obtained by the repeated application of Thm. \ref{thm:nested:nested}, i.e.:
        \begin{equation}\label{eq:thm:event:nesting:nested:1}
        \begin{gathered}
            \forall J_N \in \{1,\cdots,\bar{J}_N\},
            \exists \tau_{J_N}^N \in\fundef{\mathbb{R}}{\Theta_N}:
            \\
            \forall J_{N-1} \in \{1,\cdots,\bar{J}_{N-1}\},
            \exists \tau_{J_{N-1}}^{N-1} \in\fundef{\mathbb{R}}{\Theta_{N-1}}:
             \\
            \vdots 
            \\
            \forall J_0 \in \{1,\cdots,\bar{J}_0\},
            \exists \tau_{J_0}^0\in\fundef{\mathbb{R}}{\Theta_0}
            :
            \\
            \underset{j_{N}\in\{1,\cdots \bar{J}_{N}\}}{\min}
            \left\{
                \cdots
                \underset{j_{0}\in\{1,\cdots \bar{J}_{0}\}}{\min}
                \left\{
                    \mathbf{T}_{\alpha,\beta}^{\Theta}V_h(x,t;t_0)
                \right\}
            \right\}
            \geq 0
            \\
            \Leftrightarrow
            (x,t_0) \models \psi,
        \end{gathered}
    \end{equation}
    with $\alpha = \alpha_{J_0,\cdots,J_{N}},\beta = \beta_{J_0,\cdots,J_N},\Theta = \Theta_{0,\cdots,N}$. Then for the fragment $\phi = F_{[\underline{t},\overline{t}]}\psi$:
    \begin{equation}\label{eq:thm:event:nesting:nested:2}
        \begin{gathered}
            \exists \tau \in \left[ 0, \overline{t} - \underline{t}\right]:
            \\
            \forall J_N \in \{1,\cdots,\bar{J}_N\},
            \exists \tau_{J_N}^N \in\fundef{\mathbb{R}}{\Theta_N}:
            \\
            \forall J_{N-1} \in \{1,\cdots,\bar{J}_{N-1}\},
            \exists \tau_{J_{N-1}}^{N-1} \in\fundef{\mathbb{R}}{\Theta_{N-1}}:
             \\
            \vdots 
            \\
            \forall J_0 \in \{1,\cdots,\bar{J}_0\},
            \exists \tau_{J_0}^0\in\fundef{\mathbb{R}}{\Theta_0}
            :
            \\
            \underset{j_{N}\in\{1,\cdots \bar{J}_{N}\}}{\min}
            \left\{
                \cdots
                \underset{j_{0}\in\{1,\cdots \bar{J}_{0}\}}{\min}
                \left\{
                    \mathbf{T}_{\alpha',\beta'}^{\Theta'}V_h(x,t;t_0)
                \right\}
            \right\}
            \geq 0
            \\
            \Leftrightarrow
            (x,t_0) \models \phi,
        \end{gathered}
    \end{equation}
    with $\tau \in \left[ 0, \overline{t} - \underline{t} \right]$:
    \begin{equation}\label{eq:thm:event:nesting:nested:params}
        \begin{gathered}
            \alpha' = \alpha + \underline{t} + \tau,\\
            \beta' = \beta + \underline{t} + \tau,\\
            \Theta' = \Theta\times\left[ 0, \overline{t} - \underline{t} \right].
        \end{gathered}
    \end{equation}
\end{thm}
\begin{proof}
    See Appendix. 
\end{proof}
}
{
This completes the operator composition rules for all possible nesting cases. Nevertheless, to treat the complete STL syntax, we need to consider logical operations, which may appear on their own or as part of the decomposition of the until operator --see Prop. \ref{prop:until}--. This will be the focus of the next subsection.
}
{
\subsection{Logical Compositions}
In this section we demonstrate how logical operations are handled within the proposed framework. We consider conjunctions/disjunctions nested with eventually, always and always eventually. This covers all cases, by inductive application of the resulting rules (note that no assumption is made on the fragments in the theorems that follow, i.e., they may themselves contain arbitrary nestings). 
\begin{thm}\label{thm:conjunction:event}
    Consider $I\in\mathbb{N}$ STL fragments $\psi_i, i\in\{1,\cdots,I\}$. Then:
    \begin{equation}\notag
        \begin{gathered}
            (x,t_0)\models F_{\left[\underline{t},\overline{t}\right]}
            \left( 
                \bigwedge_{i=1}^I \psi_i
            \right)
            \Leftrightarrow
            \\
            \forall \tau \in \left[ 0,\overline{t} - \underline{t}\right]:
            \bigwedge_{i=1}^I
            \left( 
                (x,t_0) \models 
                F_{\left[ \underline{t} + \tau , \underline{t} + \tau \right]}\psi_i
            \right).
        \end{gathered}
    \end{equation}
\end{thm}
\begin{proof}
    By Def. \ref{def:operator} we have:
\begin{equation}\notag 
    \begin{gathered}
        (x,t_0)\models F_{\left[\underline{t},\overline{t}\right]}
         \left( 
            \bigwedge_{i=1}^I \psi_i
        \right)
        \Leftrightarrow
        \\
        \exists t \in \left[t_0 + \underline{t},t_0 + \overline{t}\right]:
        (x,t) \models 
        \bigwedge_{i=1}^I \psi_i
        \Leftrightarrow
        \\
         \exists t \in \left[t_0 + \underline{t},t_0 + \overline{t}\right]:
        \bigwedge_{i=1}^I 
        \left( 
            (x,t) \models \psi_i
        \right)
        \Leftrightarrow
        \\
        \exists \tau \in \left[ 0, \overline{t}-\underline{t}\right]:
        \bigwedge_{i=1}^I 
        \left( 
            (x,t_0 + \underline{t} + \tau) \models \psi_i
        \right)
        \Leftrightarrow\\
        \forall \tau \in \left[ 0, \overline{t}-\underline{t}\right]:
        \\
        \bigwedge_{i=1}^I 
        \left( 
            \exists t' = t_0 + \underline{t} + \tau \in
            \left[ 
                t_0 +\underline{t}, t_0 +\overline{t}
            \right]:
            (x,t') \models \psi_i
        \right)
        \Leftrightarrow\\
         \forall \tau \in \left[ 0, \overline{t}-\underline{t} \right]:
         \bigwedge_{i=1}^I 
        \left( 
            (x,t_0) \models  F_{\left[ \underline{t} + \tau, \underline{t} + \tau\right]}\psi_i
        \right),
    \end{gathered}
\end{equation}
concluding the proof. 
\end{proof}
\begin{thm}\label{thm:conjunction:alw}
    Consider $I\in\mathbb{N}$ STL fragments $\psi_i, i\in\{1,\cdots,I\}$. Then:
    \begin{equation}\notag
        \begin{gathered}
            \forall \tau \in \Theta:
            (x,t_0)\models G_{\left[\alpha(\tau),\beta(\tau)\right]}
            \left( 
                \bigwedge_{i=1}^I \psi_i
            \right)
            \Leftrightarrow
            \\
            \forall \tau \in\Theta:
            \bigwedge_{i=1}^I
            \left( 
                (x,t_0) \models 
                G_{\left[ \alpha(\tau),\beta(\tau) \right]}\psi_i
            \right).
        \end{gathered}
    \end{equation}
\end{thm}
\begin{proof}
    By Def. \ref{def:operator} we have:
\begin{equation}\notag
    \begin{gathered}
        \forall \tau \in \Theta: (x,t_0)\models G_{\left[ \alpha(\tau),\beta(\tau) \right]}
         \left( 
            \bigwedge_{i=1}^I \psi_i
        \right)
        \Leftrightarrow
        \\
        \forall \tau \in \Theta: 
        \forall t \in \left[t_0 + \alpha(\tau),t_0 + \beta(\tau)\right]:
        (x,t) \models 
        \bigwedge_{i=1}^I \psi_i
        \Leftrightarrow
        \\
        \forall \tau \in \Theta: 
         \forall t \in \left[t_0 + \alpha(\tau),t_0 + \beta(\tau)\right]:
        \bigwedge_{i=1}^I 
        \left( 
            (x,t) \models \psi_i
        \right)
        \Leftrightarrow
        \\
        \forall \tau \in \Theta:
        \bigwedge_{i=1}^I 
        \left( 
             \forall t \in \left[t_0 + \alpha(\tau),t_0 + \beta(\tau)\right]:
            (x,t) \models \psi_i
        \right)
        \Leftrightarrow
        \\
        \forall \tau \in \Theta:
        \bigwedge_{i=1}^I 
        \left( 
            (x,t_0) \models G_{\left[ \alpha(\tau),\beta(\tau) \right]}\psi_i
        \right),
    \end{gathered}
\end{equation}
concluding the proof. 
\end{proof}
\begin{remark}
    The proofs of Thms. \ref{thm:conjunction:event}, \ref{thm:conjunction:alw}, hold similarly for disjunctions.
\end{remark}
\begin{thm}\label{thm:alw:event:conjunction:indep}
    Consider a collection of simplified fragments $\psi_i',i\in\{1,\cdots,I\},I\in\mathbb{N}$. Then, for the nested fragment:
    $\phi = G_{\left[\alpha'(\tau'),\beta'(\tau')\right]}F_{\left[\underline{t},\overline{t}\right]}\left( \bigwedge_{i=1}^{I}\psi_i' \right)$: 
    $(x,t_0)\models \phi 
    \Leftrightarrow
    (x,t_0)\models
    \bigwedge_{i=1}^{I}
     G_{\left[\alpha'(\tau'),\beta'(\tau')\right]}F_{\left[\underline{t},\overline{t}\right]}\psi_i'$.
\end{thm}
\begin{proof}
    We have :
    \begin{equation}\notag 
        \begin{gathered}
            \forall \tau' \in \Theta':
            (x,t_0)\models \phi = 
            G_{\left[\alpha'(\tau'),\beta'(\tau')\right]}F_{\left[\underline{t},\overline{t}\right]}
            \left( \bigwedge_{i=1}^{I}\psi_i' \right)
            \Leftrightarrow
            \\
            \forall t\in
            \left[
                t_0 + \alpha'(\tau'),t_0 + \beta'(\tau')
            \right]:
            \exists t' \in
            \left[
                t + \underline{t},t + \overline{t}
            \right]:
            \\
            (x,t') \models  \bigwedge_{i=1}^{I}\psi_i'
            \overset{\mathrm{Def.} \ref{def:stl}}{=}
             \bigwedge_{i=1}^{I}
             \left( 
                (x,t') \models \psi_i'
             \right)
             \Leftrightarrow
             \\
             \forall i \in \{1,\cdots,I\}:
             \forall \tau_i' \in \Theta':
             \forall t\in
            \left[
                t_0 + \alpha'(\tau'),t_0 + \beta'(\tau')
            \right]:
            \\
            \exists t' \in
            \left[
                t + \underline{t},t + \overline{t}
            \right]:
            (x,t') \models \psi_i'
            \overset{\mathrm{Def.} \ref{def:stl}}{\Longleftrightarrow}
            \\
             \forall i \in \{1,\cdots,I\}:
             \forall \tau_i' \in \Theta':
             (x,t_0) \models 
             G_{\left[\alpha'(\tau_i'),\beta'(\tau_i')\right]}F_{\left[\underline{t},\overline{t}\right]}\psi_i'
             \\ 
             \overset{\mathrm{Def.} \ref{def:stl}}{\Longleftrightarrow}
                (x,t_0) \models 
            \bigwedge_{i=1}^{I}
             G_{\left[\alpha'(\tau_i'),\beta'(\tau_i')\right]}F_{\left[\underline{t},\overline{t}\right]}\psi_i',
        \end{gathered}
    \end{equation}
    where the equivalence holds iff $\tau_1' = \tau_2' = \cdots = \tau_N'$, concluding the proof. 
\end{proof}
\begin{thm}\label{thm:alw:event:disjunction:indep}
    Consider a collection of simplified fragments $\psi_i',i\in\{1,\cdots,I\},I\in\mathbb{N}$. Then, for the nested fragment:
    $\phi = G_{\left[\alpha'(\tau'),\beta'(\tau')\right]}F_{\left[\underline{t},\overline{t}\right]}\left( \bigvee_{i=1}^{I}\psi_i' \right)$: 
    $(x,t_0)\models \phi 
    \Leftrightarrow
    (x,t_0)\models
    \bigvee_{i=1}^{I}
     G_{\left[\alpha'(\tau'),\beta'(\tau')\right]}F_{\left[\underline{t},\overline{t}\right]}\psi_i'$.
\end{thm}
\begin{proof}
    The proof is identical to Thm. \ref{thm:alw:event:conjunction:indep}.
\end{proof}
\begin{remark}\label{rem:param:consistency}
    The condition $\tau_1' = \tau_2' = \cdots = \tau_N'$ in Thms. \ref{thm:alw:event:conjunction:indep}, \ref{thm:alw:event:disjunction:indep}, and similarly within the proof of Thms. \ref{thm:conjunction:event}, \ref{thm:conjunction:alw}, is necessary to maintain \emph{consistency}; each CBF-STL operator $\mathbf{T}_{\alpha_i,\beta_i}^{\Theta_i},i\in\{1,\cdots,I\}$ corresponds to $\psi_i',i\in\{1,\cdots,I\}$ resulting from applying the rules of the previous section. By applying the rules of Thms. \ref{thm:alw:event:conjunction:indep}, \ref{thm:alw:event:disjunction:indep}, every CBF-STL operator is \emph{independently} modified. For the latter to correspond to the logic of the nested fragment, this independence needs to be amended, which is exactly the role of the condition $\tau_1' = \tau_2' = \cdots = \tau_N'$.
\end{remark}
By virtue of the simplified form $\psi'$ (which we remind the reader can always be acquired through Props. \ref{prop:always}, \ref{prop:eventually}) Thms. \ref{thm:alw:event:conjunction:indep}, \ref{thm:alw:event:disjunction:indep} cover all possible cases of nesting logical operators. 
}
\section{Scheduling, Task Planning \& Control}
{ In this section we employ a similar approach to \citebr{10388467} for modeling STL formulae using graphs. We include here a discussion on the notation for clarity, and the reader is directed there for more details. The authors in the former propose signal Temporal Logic Trees (sTLTs), and an equivalence relationship between complete paths of an sTLT and satisfaction of the associated STL formula is demonstrated. 
Consider the example $\phi = F_{\left[0,15\right]}\left( G_{\left[2,10\right]}\mu_1 \vee \mu_2 U_{\left[5,10\right]}\mu_3\right)$, with its associated sTLT depicted in Fig. \ref{fig:stlt}. Its nodes fall into three categories, 1) set nodes, denoted by $\mathbb{X}_i,i\in\{1,\cdots,8\}$, 2) logic nodes denoted by $\{\wedge,\vee\}$, and 3) temporal nodes denoted similarly to STL operators. The set nodes correspond to \emph{satisfying sets}, i.e., subsets of the state space for which the nodes' sub-trees are reachable under the input constraints. The authors in \citebr{10388467} denote by $\mathbf{p}$ complete paths of the sTLT:
% \begin{equation}\notag 
   \( \mathbf{p} = \mathbb{X}_0 \widetilde{\Theta}_1 \mathbb{X}_1 \widetilde{\Theta}_2 \cdots\),
% \end{equation}
where $\widetilde{\Theta}_i,i\in\mathbb{N}$ denotes\footnote{we use $\widetilde{\Theta}$ instead of the symbol $\Theta$ in \citebr{10388467} which we have already reserved here in the operator definition} logic/temporal nodes. The discrepancy between $\phi = F_{\left[0,15\right]}\left( G_{\left[2,10\right]}\mu_1 \vee \mu_2 U_{\left[5,10\right]}\mu_3\right)$ and Fig. \ref{fig:stlt} stems from the use of a ``desired form'' $\hat{\phi} = ( F_{\left[0,15\right]}G_{\left[2,10\right]}\mu_1)\vee F_{\left[0,15\right]}\left(  G_{[0,10]}\mu_2 \wedge F_{\left[5,10\right]}\mu_3\right)$ in \citebr{10388467}, which we will not employ in this work. 
}
\begin{figure}
    \centering
    \includegraphics[width= 0.5\linewidth]{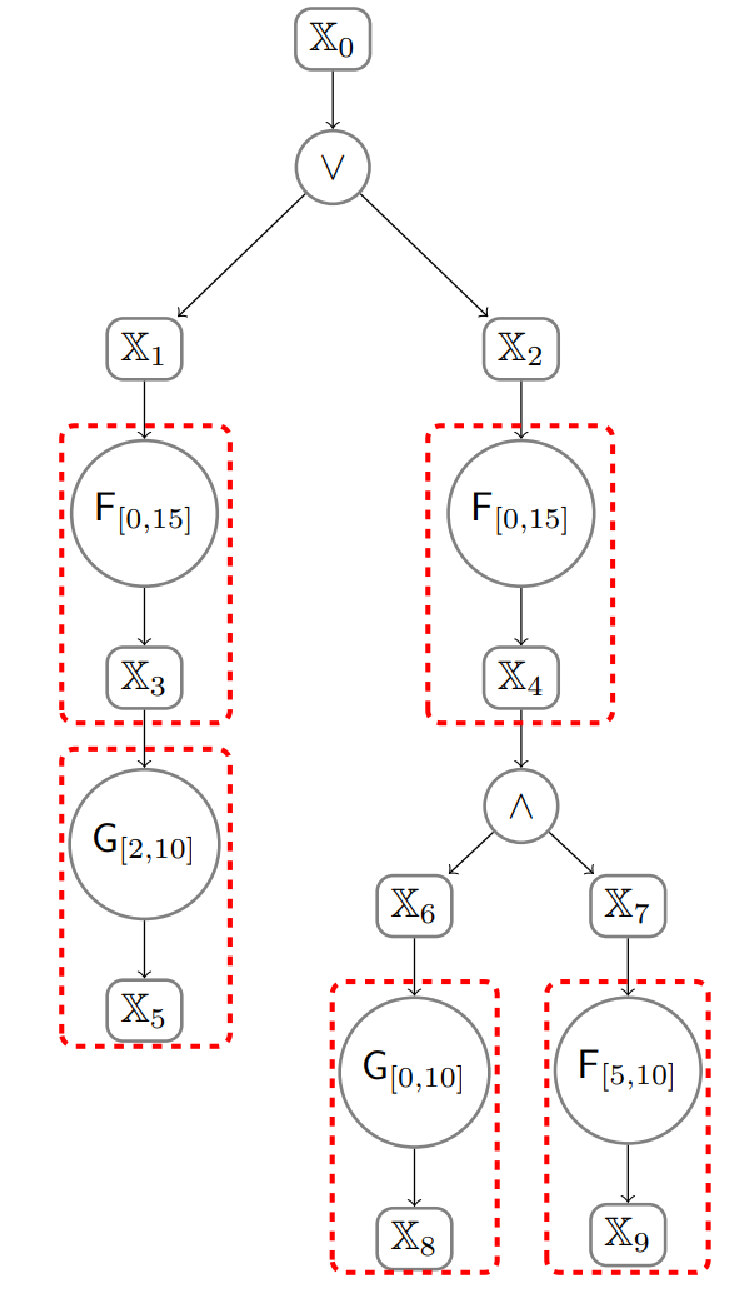}
    \caption{Fig. 2 in \citebr{10388467}, sTLT of $\phi = F_{\left[0,15\right]}\left( G_{\left[2,10\right]}\mu_1 \vee \mu_2 U_{\left[5,10\right]}\mu_3\right)$ for the desired form $\hat{\phi} = ( F_{\left[0,15\right]}G_{\left[2,10\right]}\mu_1)\vee F_{\left[0,15\right]}\left(  G_{[0,10]}\mu_2 \wedge F_{\left[5,10\right]}\mu_3\right)$.}
    \label{fig:stlt}
\end{figure}
\begin{figure*}
    \centering
    \includegraphics[trim={0cm 15.5cm 11.5cm 0cm},clip,width= 0.9\linewidth]{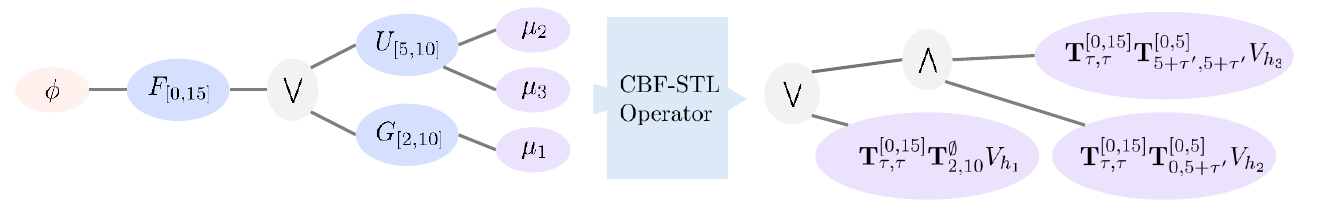}
    \caption{Examples of STL operator graph (left) and logic graph (right) for the STL formula $\phi= F_{\left[0,15\right]}\left( G_{\left[2,10\right]}\mu_1 \vee \mu_2 U_{\left[5,10\right]}\mu_3\right)$.}
    \label{fig:task:logic:graphs}
\end{figure*}
\par
Consider now a general STL formula \eqref{eq:stl:formula}, which can be modeled using a rooted tree structure similar to sTLTs. Let $\mathcal{G}_{\textrm{STL}} = \left\{ \mathcal{V}_{\textrm{STL}},  \mathcal{E}_{\textrm{STL}}\right\}$ denote a tree whose vertices $v \in \mathcal{V}_{\textrm{STL}}$ consist of simple predicates, STL and logical operators. The edge set $ \mathcal{E}_{\textrm{STL}}$ consists of pairs of nodes with a parent-child relationship, where each edge corresponds to a nesting operation. For instance, for $F_{[\underline{t},\overline{t}]}\mu$, where $\mu$ is a simple predicate, the vertex corresponding to $\mu$ is a child of the node corresponding to $F_{[\underline{t},\overline{t}]}$. If $\mu$ is an STL \textbf{fragment} then its vertex forms a sub-branch with root $F_{[\underline{t},\overline{t}]}$. In such trees, leaf nodes always correspond to simple predicates, while the root node is considered to model the entire STL formula. An example $\phi= F_{\left[0,15\right]}\left( G_{\left[2,10\right]}\mu_1 \vee \mu_2 U_{\left[5,10\right]}\mu_3\right)$ \citebr{10388467} is depicted in Fig. \ref{fig:task:logic:graphs} (left).
\par 
A main difference between $\mathcal{G}_{\textrm{STL}}$ and sTLTs lies in the omission from $\mathcal{G}_{\textrm{STL}}$ of the corresponding reachable sets, since reachability is baked into the proposed operator's action on reachability value functions. 
Hence, interpreting the vertices of $\mathcal{G}_{\textrm{STL}}$ as containing the corresponding sets of sTLTs turns $\mathcal{G}_{\textrm{STL}}$ into an sTLT. {A notable difference between this work and \citebr{10388467} rests on computational complexity; first, each one of the set nodes in \citebr{10388467} requires the solution of a reachability problem whereas in our work the number of reachability functions is exactly equal to the number of simple predicates. Second, we bypass computing CBFs for the system, since we employ the reachability value functions directly.}
\par 
We further construct a \textit{logic tree}, denoted as $\mathcal{G}_{\textrm{LOG}}= \left\{ \mathcal{V}_{\textrm{LOG}},  \mathcal{E}_{\textrm{LOG}}\right\}$, by first replacing the leaf vertices (corresponding to simple predicates) with the associated operators acting on the predicates' value functions, and second, removing the STL operator vertices. Note that by construction $\mathcal{G}_{\textrm{LOG}}$ is a proper tree (i.e., any inner node has at least two children; every inner node is a logical operator and its children are the operator's arguments). {We denote the set of such logic graphs, i.e., proper trees with inner nodes corresponding to logical operators and leaf nodes corresponding to nested CBF-STL operators \eqref{eq:operator:def} acting on value functions \eqref{eq:vfun:def}, as $\mathfrak{G}_{\textrm{log}} \ni \mathcal{G}_{\textrm{log}}$.} An example is illustrated in Fig. \ref{fig:task:logic:graphs} (right), where the until operator is replaced by its decomposition.

\subsection{Scheduling \& Task Planning}
\subsubsection{{Preliminaries}}
Given an STL formula $\phi$ \eqref{eq:stl:formula}, \textit{scheduling} refers to finding a sequence of tasks that can be completed given system \eqref{eq:dynamics} under admissible inputs defined by $\mathcal{U}$. Concurrently, we distinguish task \textit{planning} as the exact time for which each task is satisfied. In our approach, both aspects are encoded via the parameters $\tau\in\Theta$ in \eqref{eq:operator:def}. 
\par 
Given a collection of $K\in\mathbb{N}$ simple predicates $\{\mu_1,\cdots,\mu_K\}$ with their associated predicate functions $\{h_1,\cdots,h_K\}, h_k \in \fundef{\mathcal{X}}{\mathbb{R}}$ and value functions $V_{h_k}\in\fundef{\mathcal{X}\times\mathbb{R}}{\mathbb{R}},\forall k\in\{1,\cdots,K\}$ \eqref{eq:hjb}, following the operator composition rules yields the operators $\mathbf{T}_{\alpha_k,\beta_k}^{\Theta_k}\in\mathcal{T},\alpha_k,\beta_k\in\mathcal{P}_{\Theta_k}^{l_k},\Theta_k \subset \mathbb{R}^{l_k}\cup\emptyset,\forall k \in \{1,\cdots,K\}$. 
\par 
For each operator, let $\boldsymbol{\tau}^k \in \Theta_k,\forall k \in \{1,\cdots,K\}$ denote their parameters. Note that several of these parameters may coincide (i.e., should take identical values at all times), which can be modeled via a (symmetric) adjacency matrix $\mathcal{A} \in \{0,1\}^{\left( \sum_{k=1}^{K}l_k\right)\times \left( \sum_{k=1}^{K}l_k\right)}$, corresponding to the stacked parameter vector $\boldsymbol{\tau} \triangleq \left( \boldsymbol{\tau}^1,\cdots,\boldsymbol{\tau}^K\right) \in  \Theta_1 \times\cdots\times\Theta_K \subset \mathbb{R}^{\sum_{k=1}^{K}l_k}\cup \emptyset$. Then, $\mathcal{A}_{i,j}=1$ if the $i$-th and $j$-th elements of $\boldsymbol{\tau}$ coincide, otherwise $\mathcal{A}_{i,j} = 0$. 
Given the above matrix's rank $\mathbb{N}\ni\hat{L} = \textrm{rank}(\mathcal{A})$, let $\mathcal{L} = \left\{\hat{l}_1,\cdots,\hat{l}_{\hat{L}}\right\} \subset \left\{1,\cdots,\sum_{k=1}^{K}l_k\right\}$ denote the set containing the indices of the independent columns of $\mathcal{A}$. Then, the independent parameters $\hat{\boldsymbol{\tau}} \in \mathbb{R}^{\hat{L}}$ of $\phi$ satisfy ${\boldsymbol{\tau}} = \hat{\mathcal{A}}^{\top}\hat{\boldsymbol{\tau}}$, where $\hat{\mathcal{A}}\in\{0,1\}^{\hat{L}\times \left( \sum_{k=1}^{K}l_k\right)}$ is formed by the independent columns of $\mathcal{A}$ and $\hat{\boldsymbol{\tau}} \in \Theta_{\mathcal{L}} \triangleq \Theta_{\hat{l}_1} \times \cdots \times \Theta_{\hat{L}}$. 
\par 
Given the formula $\phi= F_{\left[0,15\right]}\left( G_{\left[2,10\right]}\mu_1 \vee \mu_2 U_{\left[5,10\right]}\mu_3\right)$ (example \citebr{10388467}, Fig. \ref{fig:task:logic:graphs}), the adjacency matrix for $\boldsymbol{\tau} = \left( \boldsymbol{\tau}^1, \boldsymbol{\tau} ^2, \boldsymbol{\tau} ^3\right) = \left( \tau,\left( \tau,\tau'\right), \left( \tau,\tau'\right)\right) \in \mathbb{R}^5$ is: 
\begin{equation}\notag 
    \mathcal{A} = 
    \begin{bmatrix}
        1 & 1 & 0 & 1 & 0\\
        1 & 1 & 0 & 1 & 0 \\
        0 & 0 & 1 & 0 & 1 \\
        1 & 1 & 0 & 1 & 0 \\
         0 & 0 & 1 & 0 & 1
    \end{bmatrix},
\end{equation}
and the independent parameters with $\mathcal{L} = \{1,3\}$ satisfy:
\[
    {\boldsymbol{\tau}}
    =
    \begin{bmatrix}
        1 & 1 & 0 & 1 & 0 \\
         0 & 0 & 1 & 0 & 1
    \end{bmatrix}^\top 
    \hat{\boldsymbol{\tau}},
\]
where $\hat{\boldsymbol{\tau}} = \left( \tau,\tau'\right) \in \left[0,15\right]\times\left[0,5\right]$. In case of repeated fragments, the dimensions of the matrices $\mathcal{A}, \hat{\mathcal{A}}$ and the parameter vectors $\boldsymbol{\tau}, \hat{\boldsymbol{\tau}}$ should appropriately and consistently increase as task repetitions elapse.  
\par 
Finally, given a rooted tree $\mathcal{G} = \{\mathcal{V},\mathcal{E}\}$ with root $r$ and leaves $\mathcal{L} \subseteq \mathcal{V}$. We denote as $d(v)$ the depth (distance to the root) of vertex $v$, $p(v)$ the parent of $v$, $\deg(v)$ the number of outgoing edges of $v$ and the group $L(v) = \{ \ell \in \mathcal{L} \mid v \text{ is an ancestor of } \ell \}$. We denote as $\textrm{LCA}(l_1,l_2),l_1,l_2 \in \mathcal{L}$ the lowest common ancestor of two leaf nodes, i.e.:
\begin{equation}
    \textrm{LCA}(l_1,l_2) = \underset{v \in \mathcal{V}/\mathcal{L}}{\arg\max\{d(v)\}}, \textrm{such that: } l_1, l_2 \in L(v).
\end{equation}
\subsubsection{{sTLT Semantics}}
{
Throughout the sequel we implicitly employ the nesting rules of the previous section. An important step rests on leveraging Cors. \ref{cor:nesting:always:ops}, \ref{cor:nesting:eventually:ops}, \ref{cor:nested:nested:ops}, for replacing each STL operator of a formula with their corresponding CBF-STL operators for simplified fragments (we remind the reader that such fragments do not contain nestings of the form $GG, FF$). While in the aforementioned corollaries we have distinguished between always/eventually through the notation $\mathbf{\Gamma},\mathbf{\Phi}$, in the sequel we employ the symbol $\mathbf{T}$ to capture both cases. 
}
\begin{proposition}\label{prop:complete_paths}
    Given an STL formula $\phi$ \eqref{eq:stl:formula}, which stems from a collection of simple predicates $\{\mu_k,\ k=1,\cdots,K\}$, consider the corresponding STL operator tree $\mathcal{G}_{\textrm{STL}}$ its logic tree $\mathcal{G}_{\textrm{LOG}}$ as well its sTLT.
    Denote as $\mathbf{p}_k$ complete path of the sTLT ending on the leaf vertex corresponding to $\mu_k,\forall k\in\{1,\cdots,K\}$ (see \citebr{10388467} for details):
     \begin{equation}\label{eq:complete:path}
        \mathbf{p}_k = \mathbb{X}_0 \widetilde{\Theta}_1 \mathbb{X}_1 \widetilde{\Theta}_2 \cdots \widetilde{\Theta}_{N_f^k} \mathbb{X}_{N_f^k}.
    \end{equation}
    Let $(x,t) \cong \mathbf{p}_k$ denote that a trajectory $x\in\fundef{\mathbb{R}}{\mathcal{X}}$ of \eqref{eq:dynamics} \textit{satisfies} $\mathbf{p}$ (see. Definition 13 in \citebr{10388467} for details). Then:
    \begin{equation}
        (x,0) \cong \mathbf{p}_k \Leftrightarrow \mathbf{T}_{\alpha_k(\boldsymbol{\tau}^k(t)),\beta_k(\boldsymbol{\tau}^k(t))}^{\Theta_k}V_{h_k}(x(t),t) \geq 0, 
    \end{equation}
    $\forall t \geq 0$ and for some $\boldsymbol{\tau}^k \in\fundef{\mathbb{R}}{\mathbb{R}^{l_k}}$, where $\mathbf{T}_{\alpha_k,\beta_k}^{\Theta_k} \in \mathcal{T}$ is the operator corresponding to $\mu_k$, i.e.,:
    \begin{equation}\label{eq:neted:operator:ex} 
        \mathbf{T}_{\alpha_k,\beta_k}^{\Theta_k}V_{h_k}(x,t) = \mathbf{T}_{\alpha_k^1,\beta_k^1}^{\Theta_k^1}\mathbf{T}_{\alpha_k^2,\beta_k^2}^{\Theta_k^2}\cdots \mathbf{T}_{\alpha_k^{N_f^k},\beta_k^{N_f^k}}^{\Theta_k^{N_f^k}}V_{h_k}(x,t),
    \end{equation}
    where for notational consistency between the indices the operators corresponding to $\widetilde{\Theta}_i\in\{\wedge,\vee\}$ can be replaced with the identity operator $\mathbf{T}_{0,0}^{\emptyset}$, i.e., $\alpha_k^i=\beta_k^i=0$ for any $\Theta_k^i=\emptyset, \forall i\in\{0,\cdots,N_f^k\}$ such that $\widetilde{\Theta}_i\in\{\wedge,\vee\}$. The function $V_{h_k}\in\fundef{\mathcal{X}\times\mathbb{R}}{\mathbb{R}}$ is the value function corresponding to $\mu_k$ and $\alpha_k,\beta_k\in\mathcal{P}_{\Theta_k}^{l_k},\Theta_k \subset \mathbb{R}^{l_k}\cup\emptyset, \alpha_k^i,\beta_k^i\in\mathcal{P}^{l_k^i},\Theta_k^i \subset \mathbb{R}^{l_k^i}\cup\emptyset,\forall i\in\{1,\cdots,N_f^k\}, \forall k \in \{1,\cdots,K\}$. 
\end{proposition}
\begin{proof}
    Consider some $k \in \{1,\cdots,K\}$. We will prove the claim for each item i) through iv) of Definition 13 in \citebr{10388467}. For the complete path \eqref{eq:complete:path}, $\mathbb{X}_i$ correspond to sets and $\widetilde{\Theta}_i,i\in\{1,\cdots,{N_f^k}\}$ correspond to STL operators (see \citebr{10388467} for details).  
    First note that \textit{time interval codings} Definition 12 in \citebr{10388467} correspond to the CBF-STL operator intervals $[\alpha_k^i,\beta_k^i]$. 
    To see this, the existence of a time interval coding is defined as ``\textit{assigning a time interval to each reachable set $\mathbb{X}_i$}'', which corresponds exactly to finding some $\boldsymbol{\tau}^k(t)$ (possibly discontinuous) which defines the starting and ending times of each STL operator.  
    \par 
    Furthermore, consider the constituents of the path's $\mathbf{p}_k$ corresponding operator in Eq. \eqref{eq:neted:operator:ex}.
    Each $\widetilde{\Theta}_i, \mathbb{X}_i$ in \citebr{10388467} is defined recursively from $i={N_f^k}$ to $i=1$, with $\mathbb{X}_{N_f^k}$ corresponding to $\mathcal{B}_{h_k}$ in our notation (i.e., $\setdef{x\in\mathcal{X}}{{h_k}(x) = V_{h_k}(x,0)\geq 0}$), and $\mathbb{X}_{i-1}$ is the reachable set for the STL operator $\widetilde{\Theta}_{i}$ from the set $\mathbb{X}_{i}$ given its time interval coding, which is exactly:
    \begin{equation}\notag 
        \mathbb{X}_{i-1}= 
        \setdefr
        {x\in\mathcal{X}}
        { \mathbf{T}_{\alpha_k^i,\beta_k^i}^{\Theta_k^i}\cdots \mathbf{T}_{\alpha_k^{N_f^k},\beta_k^{N_f^k}}^{\Theta_k^{N_f^k}}V_{h_k}(x,t) \geq 0}.
    \end{equation}  
    {
    Finally, item i) in \citebr{10388467} corresponds to Thms. \ref{thm:alw:event:conjunction:indep}, \ref{thm:alw:event:disjunction:indep} by replacing the outer operator with the identity one (see Def. \ref{def:operator}), items ii) and iii) are covered by Props. \ref{thm:nesting:always}, \ref{thm:nesting:eventually} and Thms. \ref{thm:nesting:alw:event:alw}, \ref{thm:nested:nested}, while item iv) corresponds exactly to set invariance by the action of the operator on the value function. Sufficiency and necessity stem directly from the aforementioned proofs, where STL formula satisfaction is equivalent to the state belonging to the corresponding reachable sets  (i.e., super-level sets of the value function acted on by the operator). 
    }
\end{proof}

\subsubsection{{Formula Satisfaction}}
In order to satisfy the entire formula $\phi$, the logical relationships between complete paths of the sTLT must be taken into account. To this end, we will design a function $\sigma\in\fundef{\mathcal{X}\times\mathbb{R}\times\Theta}{\mathbb{R}}$ whose positivity is necessary and sufficient for formula satisfaction, via Alg. \ref{alg:groupings}.
\begin{figure*}[ht!]
    \centering
    \includegraphics[trim={0.0cm 12.75cm 5.5cm 0cm},clip,width=0.9\linewidth]{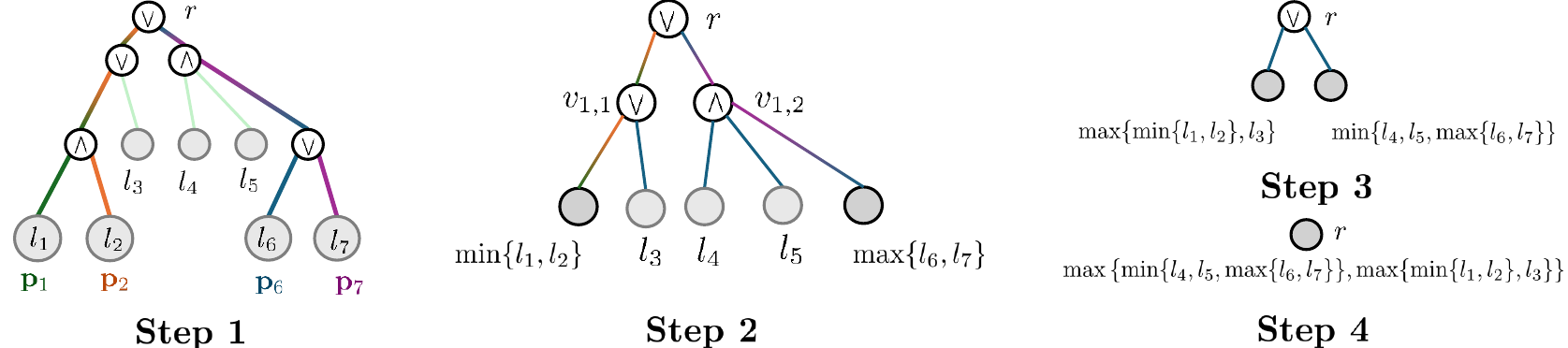}
    \caption{Example/visual aid for Alg. \ref{alg:groupings} and the proof of Thm. \ref{thm:formula:satisfaction}. At step 1, two pairs of leaves with the maximum LCA are identified $\left(\{l_1,l_2\},\{l_6,l_7\}\right)$.}
    \label{fig:graph:ex:thm}
\end{figure*}
\begin{algorithm}[ht!]
\caption{Bottom-Up Logic Tree Groupings}\label{alg:groupings}
\begin{algorithmic}[1]
\State Consider the proper logic tree $\mathcal{G}_{\textrm{LOG}} = (\mathcal{V}_{\textrm{LOG}}, \mathcal{E}_{\textrm{LOG}})$ of depth $D\in\mathbb{N}$,
with its root removed, let its new root denoted by $r$. Since the tree is proper, for internal nodes $\deg(v) \geq 2$, and for leaves $\deg(\ell) = 0$. For each internal node $v \in \{\vee,\wedge\}$.
\State $\bar{d}\gets D-1$.
\While{$\bar{d} \geq  0$}
    \State Find the leaf subset(s) with LCA of maximum depth:
    \[
        \{l_1^\star, l_2^\star,\cdots\} = \underset{\boldsymbol{l} \in \textrm{Pow}(\mathcal{L})}{\arg \max} \left\{ d\left( \textrm{LCA}(\boldsymbol{l}\right) \right\}.
    \]
    \State Remove leaves $l_1^\star, l_2^\star,\cdots$ from tree, relabel $v \triangleq p(l_1^\star) = p(l_2^\star) = \cdots$ as $v'$, with:
    \[
        v' = 
        \begin{cases}
            \min\{ l_1^\star, l_2^\star,\cdots \} , &v \in \{\wedge\}
            \\
            \max\{ l_1^\star, l_2^\star,\cdots \} , &v \in \{\vee\}
        \end{cases}.
    \]
    % Since $\mathcal{G}_{\textrm{LOG}}$ is a proper tree, the LCA of any leaf pair is their parent (no single degree chains on tree), and $v'$ becomes a new leaf.
    \State Update $\bar{d}\gets \underset{\boldsymbol{l} \in \textrm{Pow}(\mathcal{L})}{\max} \left\{ d\left( \textrm{LCA}(\boldsymbol{l})\right) \right\}$
\EndWhile
\State Return root label $r$.
\end{algorithmic}
\end{algorithm}
\\
An example of the steps of Alg. \ref{alg:groupings} is depicted in Fig. \ref{fig:graph:ex:thm}.
The function $\sigma$ then results from replacing the leaf node labels with their corresponding operator-based value functions (see Fig. \ref{fig:task:logic:graphs}) and is formally defined via the mapping:
\begin{equation}\label{eq:sigma_fun_def}
    \mathfrak{G}_{\textrm{log}} \overset{\textrm{Alg. }\ref{alg:groupings}}{\longrightarrow} \fundef{\mathcal{X}\times\mathbb{R}\times\Theta}{\mathbb{R}} \ni \sigma .
\end{equation}
We overload the notation that we have used thus far for the operators by also including the dependence on the corresponding parameters:
\begin{equation}\notag 
    \fundef{\mathcal{X}\times\mathbb{R}\times\Theta}{\mathbb{R}}\ni\mathbf{V}_h(x,t,\boldsymbol{\tau})
    = \mathbf{T}_{\alpha(\boldsymbol{\tau}), \beta(\boldsymbol{\tau})}^{\Theta}V_h(x,t;t_0).
\end{equation}
For the case of $\phi= F_{\left[0,15\right]}\left( G_{\left[2,10\right]}\mu_1 \vee \mu_2 U_{\left[5,10\right]}\mu_3\right)$ (example \citebr{10388467}, Fig. \ref{fig:task:logic:graphs}) this yields:
\begin{equation}\label{eq:sigma_fun:example}
    \begin{split}
         &\sigma(x,t,\boldsymbol{\tau})
         =
         \sigma\left(x,t,\left(\boldsymbol{\tau}^1,\boldsymbol{\tau}^2,\boldsymbol{\tau}^3\right)\right)=
         \\
         &{\max}
         \left\{ 
            \mathbf{V}_1(x,t;\boldsymbol{\tau}^1), 
            \min
            \left\{
                \mathbf{V}_2(x,t;\boldsymbol{\tau}^2),\mathbf{V}_3(x,t;\boldsymbol{\tau}^3)
            \right\}
         \right\},
    \end{split}
    \raisetag{1.0cm}
\end{equation}
where $\mathbf{V}_i(x,t;\boldsymbol{\tau}^i) = \mathbf{T}_{\alpha_i(\boldsymbol{\tau}^i),\beta_i(\boldsymbol{\tau}^i)}^{\Theta_i}V_{h_i}(x,t;t_0),i\in\{1,2,3\}$.

\begin{thm}\label{thm:formula:satisfaction}
    Given an STL formula $\phi$ \eqref{eq:stl:formula}, which contains a collection of simple predicates $\{\mu_k,\ k=1,\cdots,K\}$, consider the corresponding STL operator tree $\mathcal{G}_{\textrm{STL}}$ its logic tree $\mathcal{G}_{\textrm{LOG}}$ as well its sTLT, denoted by $\mathcal{T}^T$. Construct the corresponding function $\sigma \in \fundef{\mathcal{X}\times\mathbb{R}\times \Theta}{\mathbb{R}}$, evaluated as $\sigma(x,t,\boldsymbol{\tau})$. Now let $\mathcal{A}$ denote the adjacency matrix for the parameters $\boldsymbol{\tau} = \left(\boldsymbol{\tau}^1, \cdots, \boldsymbol{\tau}^K \right)$, with $\hat{\mathcal{A}}$ a matrix consisting of its independent columns and denote the free parameters as $\hat{\boldsymbol{\tau}}$, satisfying $\boldsymbol{\tau} = \hat{\mathcal{A}}^\top\hat{\boldsymbol{\tau}}$. Further consider the satisfaction relationship $(x,t) \cong \mathcal{T}^T$ Definition 14 in \citebr{10388467} for a trajectory $x\in\fundef{\mathbb{R}}{\mathcal{X}}$ of \eqref{eq:dynamics} from $x(t_0),t_0 =0$ w.l.o.g.. Then
    \( (x,0)\cong \mathcal{T}^T \Leftrightarrow(x,0) \models \phi\Leftrightarrow \exists \hat{\boldsymbol{\tau}} \in \fundef{\mathbb{R}}{\Theta_{\mathcal{L}}} : \sigma\left(x(t),t,\hat{\mathcal{A}}^\top \hat{\boldsymbol{\tau}}(t)\right)\geq 0, \forall t \in \mathbb{R}_{\geq 0}\).    
\end{thm}

Prior to proving Thm. \ref{thm:formula:satisfaction}, we prove the following:
\begin{proposition}\label{prop:prelim:thm:formula:satisfaction}
    Consider the setup of Thm. \ref{thm:formula:satisfaction}, two complete paths of $\mathcal{T}^T$, denoted as $\mathbf{p}_{k_1}, \mathbf{p}_{k_2},k_1,k_2\in\{1,\cdots,K\}$ and given by \eqref{eq:complete:path} such that they share a common segment that differs after some depth $d>0$, i.e., $v = \textrm{LCA}(l_{k_1},l_{k_2})\Rightarrow d(v)>0$. Further let the (nested) CBF-STL operators of the corresponding leaves $l_{k_1},l_{k_2}\in\mathcal{V}_{\textrm{LOG}}$ be $\mathbf{T}_{\alpha_{k_1},\beta_{k_1}}^{\Theta_{k_1}}V_{h_{k_1}}(x,t),\mathbf{T}_{\alpha_{k_2},\beta_{k_2}}^{\Theta_{k_2}}V_{h_{k_2}}(x,t)$ Then:
    \begin{equation}
        \begin{gathered}
            \left( 
                (x,0) \cong \mathbf{p}_{k_1}
            \right) 
            \bigwedge
            \left( 
                (x,0) \cong \mathbf{p}_{k_2}
            \right)
            \Leftrightarrow
            \\
            \min
            \left\{ 
                \mathbf{T}_{\alpha_{k_1},\beta_{k_1}}^{\Theta_{k_1}}V_{h_{k_1}}(x,t),
                \mathbf{T}_{\alpha_{k_2},\beta_{k_2}}^{\Theta_{k_2}}V_{h_{k_2}}(x,t)
            \right\}
            \geq 0,
        \end{gathered}
    \end{equation}
    {where the same holds for $\bigvee,\max$}.
\end{proposition}
\begin{proof}
    See Appendix.
\end{proof}
We are now ready to prove Thm. \ref{thm:formula:satisfaction}.
\newline
\begin{proof}[Proof of Thm. \ref{thm:formula:satisfaction}]
    Consider every complete path $\mathbf{p}_k,k\in\{1,\cdots,K\}$. {The proof follows the same process of Alg. \ref{alg:groupings}}: consider the subset of leaves of $\mathcal{G}_{\textrm{LOG}}$ with the LCA of maximum degree:
    \begin{equation}\notag 
        \{l_{i_1}, l_{i_2},\cdots l_{i_N} \} = \underset{\boldsymbol{l} \in \textrm{Pow}(\mathcal{L})}{\arg \max} \left\{ d\left( \textrm{LCA}(\boldsymbol{l}\right) \right\},
    \end{equation}
    with their corresponding complete paths denoted as $\{\mathbf{p}_{i_1}, \cdots \mathbf{p}_{i_N}\},\ i_n \in \{1,\cdots,K\}, \forall n \in \{1,\cdots,N\}:\ N\in\{ 2,\cdots,K\}$. Let $v\in\{\wedge,\vee\}$ denote the label of the above LCA (which is unique and well defined since the logic tree is proper). Per Prop. \ref{prop:prelim:thm:formula:satisfaction} (we can apply this proposition to more than two paths with exactly the same arguments):
    \begin{equation}\label{thm:formula:satisfaction:paths_eq}
        \begin{gathered}
            \bigwedge_{n=1}^{N}
            \left( 
                (x,0) \cong \mathbf{p}_{i_n}
            \right)
            \Leftrightarrow
            \\
            \min
            \left\{
                \mathbf{V}_{h_{i_n}}(x,t,\boldsymbol{\tau}^{i_n}),\ \forall n \in \{1,\cdots,N\}
            \right\}\geq 0,
        \end{gathered}
    \end{equation}
    (where $\bigwedge, \min$ should be replaced with $\bigvee,\max$ based on the LCA $v \in \{ \wedge,\vee\}$). The leaves $\{l_{i_1}, l_{i_2},\cdots l_{i_N} \}$ are then removed from the graph, and the LCA is replaced with \eqref{thm:formula:satisfaction:paths_eq}.
    This process can be repeated inductively similarly to Alg. \ref{alg:groupings} (see Fig. \ref{fig:graph:ex:thm} for an example). Note that in Thms. 1,2 in \citebr{10388467} this inductive process is exactly how the authors show the equivalence:
    \begin{equation}\label{eq:thm:final:1}
        (x,0) \cong \mathcal{T}^T
        \Leftrightarrow
        (x,0) \models \phi.
    \end{equation}
    {
    Note that, starting from \eqref{thm:formula:satisfaction:paths_eq} and following Alg. \ref{alg:groupings} while replacing the leaf nodes with the corresponding functions $\mathbf{V}_{h_{i_n}}(x,t,\boldsymbol{\tau}^{i_n}),\ \forall n \in \{1,\cdots,N\}$ is exactly how the function $\sigma \in\fundef{\mathcal{X}\times\mathbb{R}\times \Theta}{\mathbb{R}}$ is obtained --see Eq. \eqref{eq:sigma_fun_def}--. Therefore, inductively applying Prop. \ref{prop:prelim:thm:formula:satisfaction} yields:
    }
    \begin{equation}\label{eq:thm:pf:sigma_fun}
        \begin{gathered}
            \sigma
            \left( 
                x(t),t,\left( \boldsymbol{\tau}^1, \cdots, \boldsymbol{\tau}^K \right)
            \right)
            \geq 
            0.
        \end{gathered}
    \end{equation}
    Nevertheless, in order to establish the equivalence between \eqref{eq:thm:final:1} and \eqref{eq:thm:pf:sigma_fun}, the parameters in $\left( \boldsymbol{\tau}^1, \cdots, \boldsymbol{\tau}^K \right)$ should also be consistent {(see Rem. \ref{rem:param:consistency}, Thms. \ref{thm:alw:event:conjunction:indep}, \ref{thm:alw:event:disjunction:indep})}. Per Prop. \ref{prop:complete_paths}, for any complete path $\mathbf{p}_k$, the time interval codings correspond to a choice of parameters $\boldsymbol{\tau}^k,k\in\{1,\cdots,K\}$. However this does not necessarily imply consistence of the parameters between operators/value functions. To achieve this, substitute in \eqref{eq:thm:pf:sigma_fun}:
    \(
        \left( \boldsymbol{\tau}^1, \cdots, \boldsymbol{\tau}^K \right) = \hat{\mathcal{A}}^\top \hat{\boldsymbol{\tau}}.
    \)
    Therefore, by necessity and sufficiency of Props. \ref{prop:complete_paths}, \ref{prop:prelim:thm:formula:satisfaction} existence of some $\hat{\boldsymbol{\tau}} \in \fundef{\mathbb{R}}{\Theta_\mathcal{L}}$ such that:
    \begin{equation}\label{eq:thm:pf:sigma_fun:2}
        \sigma
            \left( 
                x(t),t,\hat{\mathcal{A}}^\top \hat{\boldsymbol{\tau}}
            \right)
            \geq 
            0,
    \end{equation}
    directly implies that $(x,0) \cong \mathcal{T}^T$, which by \eqref{eq:thm:final:1} is a necessary and succificent condition for formula satisfaction. To finish the proof, we need to show necessity ($\Leftarrow$): If $(x,0) \cong \mathcal{T}^T$, by the inductive arguments of Thms. 1,2 in \citebr{10388467} and the necessary and sufficient conditions of the nesting rules, this directly implies that: 1) a time interval coding exists $\Rightarrow$ a consistent set of parameters $\boldsymbol{\tau} = \hat{\mathcal{A}}^\top \hat{\boldsymbol{\tau}}$ also exists for \eqref{eq:thm:pf:sigma_fun}, and 2) the resulting function in \eqref{eq:thm:pf:sigma_fun} should be positive. 
\end{proof}
\begin{remark}
    In Thm. 1 in \citebr{10388467}, the authors {employ a similar inductive process} for a desired form $\hat{\phi}$ of the formula $\phi$, which yields a conservative approach to control for STL. As noted in Thm. 2, Rem. 4 \citebr{10388467}, the lack of equivalence between these two forms stems from: 1) the ``fixing'' of the time of satisfaction for the decomposition of the until operator, 2) when multiple always operators are combined through a disjunction, the aforementioned method enforces two or more \textbf{separate} disjunctions, 3) when encountering disjunctions, the method assumes that the disjunction only appears in the root node of the tree and hence is forced to choose a specific path on the STLT (i.e., only one out of the disjuncted formulae is enforced).
    The above limit the equivalence relationship only on formulae with the desired form in \citebr{10388467}. Otherwise, Thm. 1 in \citebr{10388467} can be used for general formulae (not necessarily in desired form). 
    \par 
    In contrast, our approach does not rely on desired forms and hence is not an under-approximation.
    More specifically, for item 1) when decomposing the until operator, its satisfaction time instance is left ``free'', owing to our parametrization of the CBF-STL operator,
    while for items 2), 3) the our technical proofs do not assume the form of the formula for disjunctions; the corresponding $\max$ operation on invariant sets enables not choosing a specific path of the STLT.
\end{remark}

\subsection{Online Control for STL Satisfaction}
Eq. \eqref{eq:thm:pf:sigma_fun:2} yields a necessary and sufficient condition through the existence of the parameter vector $\hat{\boldsymbol{\tau}}\in\fundef{\mathbb{R}}{\Theta_{\mathcal{L}}}$. Task scheduling and planning boils down to finding such a (possibly discontinuous) function. Therefore, to ensure STL formula satisfaction, two aspects should be tackled: 1) maintaining positivity of $\sigma$ \eqref{eq:thm:pf:sigma_fun:2}, and 2) designing the evolution of $\hat{\boldsymbol{\tau}}$. {This approach is in line with, and generalizes, an existing line of work on STL satisfaction through forward set invariance \citebr{10.7551/mitpress/15136.001.0001}.}
% \subsubsection{Enforcing Conjunctions/Disjunctions}
\par 
To synthesize control inputs, we need to enforce positivity of the non-smooth function $\sigma$,
which can be achieved through the combinatorial CBF framework proposed in Sec. IV in \citebr{ong2025combinatorialcontrolbarrierfunctions}. This yields the conditions (arguments are suppressed for brevity):
\(
        \dot{\mathbf{V}}_k
        >
        -\kappa
        \left(
            {\mathbf{V}}_k
            + 
            \left|
                {\mathbf{V}}_k - \sigma
            \right|
        \right), \forall k \in\{1,\cdots,K\}
\)
where $\kappa \in \fundef{\mathbb{R}}{\mathbb{R}}$ is a class $\mathcal{K}$ function. Then, per Thm. 1 in \citebr{ong2025combinatorialcontrolbarrierfunctions} trajectories that satisfy this inequality are forward time-invariant w.r.t. the zero super-level set of $\sigma$, which is a necessary and sufficient condition for STL formula satisfaction per Thm. \ref{thm:formula:satisfaction}. The main advantage of using combinatorial CBFs as opposed to treating $\sigma$ as a single, non-smooth CBF rests on provable continuity of resulting CBFQP controllers (see \citebr{ong2025combinatorialcontrolbarrierfunctions}). 

\subsubsection{Adaptive Parameter Control}
Thm. \ref{thm:formula:satisfaction}'s statement involves the  existence of time-varying parameters $\boldsymbol\tau \in \fundef{\mathbb{R}}{\Theta_1 \times \cdots \Theta_K}$. To design such a function, we employ an adaptive-control approach. 
We remind the reader that we only need to design an adaptive law for the independent parameters $\hat{\boldsymbol{\tau}} \in \Theta_{\mathcal{L}}$, which satisfy $\boldsymbol{\tau} = \hat{\mathcal{A}}^{\top}\hat{\boldsymbol{\tau}}$.
Consider the virtual model:
\begin{equation}\label{eq:dynamics:virtual:tau}
    \dot{\hat{\boldsymbol{\tau}}}= \omega, \ \hat{\boldsymbol{\tau}}(0)  \in  \Theta_{\mathcal{L}},
\end{equation}
where $\omega \in \fundef{\mathbb{R}}{\mathbb{R}^{\sum_{l \in \mathcal{L}}l}}$ denotes a virtual input. In order to design the latter, we enforce two conditions: 1) restricting $\hat{\boldsymbol{\tau}}$ within its domain $\Theta_{\mathcal{L}}$, and 2) forward invariance of the zero super-level set of $\sigma$. First, note that $\Theta_{\mathcal{L}}$ is in the form $\hat{\boldsymbol{\tau}}^{\textrm{lb}} \leq \hat{\boldsymbol{\tau}}\leq \hat{\boldsymbol{\tau}}^{\textrm{ub}}$, where $\hat{\boldsymbol{\tau}}^{\textrm{lb}}, \hat{\boldsymbol{\tau}}^{\textrm{ub}} \in \Theta_{\mathcal{L}}$ and the inequalities are interpreted element-wise. Then, $\hat{\boldsymbol{\tau}} \in \Theta_{\mathcal{L}}$ can be enforced through the {following} CBFs $\hat{\sigma}_i\in\fundef{\Theta_{\hat{l}_i}}{\mathbb{R}},i \in \{1,\cdots,\hat{L} \}$, for the virtual system \eqref{eq:dynamics:virtual:tau}:
\[
    \hat{\sigma}_{i}(\hat{\boldsymbol{\tau}}_i) =
    -\left( 
        \hat{\boldsymbol{\tau}}_i - \hat{\boldsymbol{\tau}}^{\textrm{lb}}_i
    \right)
    \left( 
         \hat{\boldsymbol{\tau}}_i - \hat{\boldsymbol{\tau}}^{\textrm{ub}}_i
    \right),\ i \in \{1,\cdots,\hat{L} \},
\]
where $\hat{\boldsymbol{\tau}}_i,\hat{\boldsymbol{\tau}}^{\textrm{lb}}_i,\hat{\boldsymbol{\tau}}^{\textrm{ub}}_i \in \mathbb{R}, i \in \{1,\cdots,\hat{L} \}$ denote the elements of the corresponding vectors. Positivity of $\hat{\sigma}_i, \forall  i \in \{1,\cdots,\hat{L} \}$ ensures that the parameters remain within $\Theta_{\mathcal{L}}$.

\subsubsection{Controller Design}
Putting the above together, consider again $K$ simple predicates and the corresponding functions $\sigma, \hat{\sigma}_i, i \in \{1,\cdots,\hat{L}\}$, which result from the nesting rules and which yield the functions $\mathbf{V}_k,k\in\{1,\cdots,K\}$. We get (the dependence on the arguments of $\mathbf{V}_k$ is suppressed for readability):
\(
        \frac{\mathrm{d}{\mathbf{V}}_k}{\mathrm{d}t}
        =
        \frac{\partial{\mathbf{V}}_k}{\partial t} +
        \frac{\partial{\mathbf{V}}_k^\top }{\partial x}f(x,u) + 
        \frac{\partial{\mathbf{V}}_k^\top }{\partial \boldsymbol{\tau}}\omega,
        \forall k \in \{1,\cdots,K\}.
\)
Hence, the conditions for ensuring forward invariance of $\sigma, \sigma_i, i \in \{1,\cdots,\hat{L}\}$ are:
\begin{subequations}\label{eq:constraints:final}
    \begin{equation}
        \begin{split}
            \begin{bmatrix}
                \frac{\mathrm{d}{\mathbf{V}}_1}{\mathrm{d}t}
                \\
                \vdots
                \\
                \frac{\mathrm{d}{\mathbf{V}}_K}{\mathrm{d}t}
            \end{bmatrix}
            >
            \begin{bmatrix}
                -\kappa
                \left(
                    {\mathbf{V}}_1
                    + 
                    \left|
                        {\mathbf{V}}_1 - \sigma(x,t,\hat{\mathcal{A}}^{\top}\hat{\boldsymbol{\tau}})
                    \right|
                \right)
                \\
                \vdots
                \\
                -\kappa
                \left(
                    {\mathbf{V}}_K
                    + 
                    \left|
                        {\mathbf{V}}_K - \sigma(x,t,\hat{\mathcal{A}}^{\top}\hat{\boldsymbol{\tau}})
                    \right|
                \right)
            \end{bmatrix},
        \end{split}
    \end{equation}
    \begin{equation}
        -\begin{bmatrix} 
            \left( 
                2\hat{\boldsymbol{\tau}}_1 - \hat{\boldsymbol{\tau}}^{\textrm{lb}}_1
                 - \hat{\boldsymbol{\tau}}^{\textrm{ub}}_1
            \right)^\top
            \\
            \vdots
            \\
            \left( 
                2\hat{\boldsymbol{\tau}}_{\hat{L}} - \hat{\boldsymbol{\tau}}^{\textrm{lb}}_{\hat{L}}
                 - \hat{\boldsymbol{\tau}}^{\textrm{ub}}_{\hat{L}}
            \right)^\top
        \end{bmatrix}
        \omega 
        \geq 
        \begin{bmatrix}
            -\hat{\kappa}_1
            \left( 
                \hat{\sigma}_{1}(\hat{\boldsymbol{\tau}}_1)
            \right)
            \\
            \vdots
            \\
            -\hat{\kappa}_{\hat{L}}
            \left( 
                \hat{\sigma}_{\hat{L}}(\hat{\boldsymbol{\tau}}_{\hat{L}})
            \right)
        \end{bmatrix},
    \end{equation}
\end{subequations}
where $\hat{\kappa}_i,\forall i\in\{1,\cdots,\hat{L}\}$ are class $\mathcal{K}$ functions\footnote{Note that the value functions can be shown to be mutually compatible CBFs (as discussed in\citebr{10388467}) if the STL formula is logically consistent and feasible under the input constraints, therefore the choice of $\kappa \in \mathcal{K}$ in that case is trivial. Furthermore, each variable of the virtual system \eqref{eq:dynamics:virtual:tau} forms an independent, fully actuated dynamical system with a \textbf{single} associated CBF constraint. Therefore, the choice of class $\mathcal{K}$ functions $\kappa_i,\forall i\in\{1,\cdots,\hat{L}\}$ is also trivial.}. Combining Eqs. \eqref{eq:dynamics}, \eqref{eq:dynamics:virtual:tau} yields the following enhanced state-space model:
\begin{equation}\label{eq:dynamics:state_enhanced}
    \begin{split}
       \mathcal{X} \times\Theta_{\mathcal{L}}
       \ni 
       \dot{z} = 
        \frac{\textrm{d}}{\textrm{d}t}
        \begin{bmatrix}
            x \\ \hat{\boldsymbol{\tau}}
        \end{bmatrix}
        =
        \begin{bmatrix}
            f(x,u)
            \\
            \omega 
        \end{bmatrix},
    \end{split}
\end{equation}
with $ z(0) = \left( x(0), \hat{\boldsymbol{\tau}}(0)\right) \in \mathcal{X} \times\Theta_{\mathcal{L}}$.
The constraints in Eqs. \eqref{eq:constraints:final} can be enforced via an optimization-based controller:
\begin{equation}\label{eq:cbf:controller}
    \begin{gathered}
        \left\{ u^\star(x,t,\hat{\boldsymbol{\tau}}), \omega^\star(x,t,\hat{\boldsymbol{\tau}}) \right\}  =
        \\
        \underset{u\in\mathcal{U}, \omega \in \mathbb{R}^{\sum_{l \in \mathcal{L}}l}}{\arg\min}
        \left\{ 
            \delta \|u - u_{\textrm{ref}}\|^2 + 
            (1 - \delta) \|\omega - \omega_{\textrm{ref}}\|^2
        \right\},
        \\ 
        \textrm{s.t.: } C(t,x,\hat{\boldsymbol{\tau}},u,\omega) \leq 0,
    \end{gathered}
\end{equation}
where $C(t,x,\hat{\boldsymbol{\tau}},u,\omega) \leq 0$ encodes Eq. \eqref{eq:constraints:final}\footnote{While the first inequality in \eqref{eq:constraints:final} is strict, it can be turned into a non-strict inequality by introducing a slack variable to the optimization problem \eqref{eq:cbf:controller}.} and $\delta \in (0,1)$ is a tuning parameter that weighs the two optimization objectives. Furthermore, the reference inputs $u_{\textrm{ref}} \in \mathbb{R}^m, \omega_{\textrm{ref}}\in \mathbb{R}^{\sum_{l \in \mathcal{L}}l}$ may encode some performance objectives. Note that in case the dynamics \eqref{eq:dynamics} are affine w.r.t. the input, then \eqref{eq:cbf:controller} becomes a QP. 
\par 
A crucial aspect of the proposed controller \eqref{eq:cbf:controller} rests on (in)feasibility of the optimization problem, i.e., (non-)emptiness of the set $\setdef{u\in\mathcal{U}, \omega \in \mathbb{R}^{\sum_{l \in \mathcal{L}}l}}{ C(t,x,\hat{\boldsymbol{\tau}},u,\omega) \leq 0}$. While \eqref{eq:cbf:controller} ensures forward invariance of $\{\sigma \geq 0\}$, this assumes that the set is non-empty, which depends on several aspects, such as logical consistency of the formula and the evolution of the parameter values determined by $\omega_{\textrm{ref}}$. We will investigate the feasibility aspect in future works.   

\begin{remark}
    Through the enhanced state in \eqref{eq:cbf:controller}, the timings of the STL formula satisfaction can be designed through control-theoretic principles. For instance, an optimal control problem can be formulated over the enhanced state and solved to achieve a balance between control performance and STL formula satisfaction. At the same time, owing to necessity in Thm. \ref{thm:formula:satisfaction}, (non)emptiness of the constraint set $\{C \leq 0\}$ can be used to determine logical (in)consistency of formulae and which satisfaction instances are (in)feasible under the dynamics and control input constraints. These aspects are linked to feasibility of the optimization problem \eqref{eq:cbf:controller} and will be examined in future works.   
\end{remark}

{
\section{Methodological Example}
Prior to presenting simulations, we provide an example of working through our method for the nested fragment
\[
    \phi = G_{[0,15]} \left(  F_{[0,15]} \left( \psi_1 U_{[15,20]}\psi_2  \right)\right),
\]
with 
\(
\psi_1 = F_{[0,8]} \left( \mu_1 U_{[0,2]} \left( F_{[1,2]} \mu_2 \right)\right)
\),
\(
\psi_2 = F_{[5,20]}  G_{[2,3]} \mu_3. 
\)
We have:
\begin{equation}\notag 
    \begin{split}
        \psi_1 
        &=  F_{[0,8]} \left( \mu_1 U_{[0,2]} \left( F_{[1,2]} \mu_2 \right)\right)
        \\
         &= 
         F_{[0,8]} 
         \left[ 
             \left( 
                G_{[0,\tau_2]}\mu_1 
            \right)
            \wedge
            \left(
                F_{[\tau_2,\tau_2]}
                \left( 
                    F_{[1+\tau_1,1+\tau_1]} \mu_2 
                \right)
            \right)
         \right],
         \\
         &\qquad \forall \tau_1 \in \left[0,1\right], \tau_2 \in\left[0,2\right]
         \overset{\textrm{Prop. } \ref{thm:nesting:eventually}}{\Longleftrightarrow}
         \\
         \psi_1 &=
         F_{[0,8]} 
         \left[ 
             \left( 
                G_{[0,\tau_2]}\mu_1 
            \right)
            \wedge
            \left(
                F_{[1+\tau_1+\tau_2,1+\tau_1+\tau_2]} \mu_2 
            \right)
         \right]\\
         &\qquad \forall \tau_1 \in \left[0,1\right], \tau_2 \in\left[0,2\right]
         \overset{\textrm{Thm. }\ref{thm:conjunction:event}}{\Longleftrightarrow}\\
         \psi_1 &= 
         \left( 
                F_{\left[ \tau_3, \tau_3\right]} G_{[0,\tau_2]}\mu_1 
            \right)
            \wedge
            \left(
                F_{\left[ \tau_3, \tau_3\right]}
                F_{[1+\tau_1+\tau_2,1+\tau_1+\tau_2]} \mu_2 
            \right),\\
            &\qquad \forall \tau_1 \in \left[0,1\right], \tau_2 \in\left[0,2\right],  \tau_3 \in\left[0,8\right]
             \overset{\textrm{Prop. } \ref{thm:nesting:eventually}}{\Longleftrightarrow}
             \\
            \psi_1 &= 
            \left( 
                F_{\left[ \tau_3, \tau_3\right]} G_{[0,\tau_2]}\mu_1 
            \right)
            \wedge
            \left(
                F_{[1+\tau_1+\tau_2+\tau_3,1+\tau_1+\tau_2+\tau_3]} \mu_2 
            \right)\\
            &\triangleq
            \psi_{1,1} \wedge \psi_{1,2},
             \forall \tau_1 \in \left[0,1\right], \tau_2 \in\left[0,2\right],  \tau_3 \in\left[0,8\right].
    \end{split}
\end{equation}
Similarly:
\(
    \psi_2 = 
    F_{\left[ 5 + \tau_4, 5+\tau_4\right]}G_{\left[ 2,3\right]}\mu_3,  \forall \tau_4 \in \left[ 0,15 \right],
\)
and
\begin{equation}\notag 
    \begin{gathered}
        \psi_1 U_{[15,20]}\psi_2 = 
        \\
        \left( 
            G_{\left[ 0, 15 + \tau_5\right]}\psi_1
        \right)
        \wedge
        \left( 
            F_{\left[ 15 + \tau_5, 15 + \tau_5\right]}\psi_2
        \right),
        \forall \tau_5 \in \left[ 0,5\right].
    \end{gathered}
\end{equation}
Furthermore, 
\begin{equation}\notag 
    \begin{gathered}
        F_{\left[ 15 + \tau_5, 15 + \tau_5\right]}\psi_2 =
        \\
        F_{\left[ 15 + \tau_5, 15 + \tau_5\right]}
        F_{\left[ 5 + \tau_4, 5+\tau_4\right]}G_{\left[ 2,3\right]}\mu_3,  
        \\ \forall \tau_4 \in \left[ 0,15 \right], \tau_5 \in \left[ 0,5 \right]
        \overset{\textrm{Prop. \ref{thm:nesting:eventually}}}{=}
        \\
        F_{\left[ 20 + \tau_4+\tau_5, 20 + \tau_4+\tau_5\right]}G_{\left[ 2,3\right]}\mu_3,  
         \forall \tau_4 \in \left[ 0,15 \right], \tau_5 \in \left[ 0,5 \right],
    \end{gathered}
\end{equation}
and
\begin{equation}\notag 
    \begin{gathered}
        G_{\left[ 0, 15 + \tau_5\right]}\psi_1 = 
        G_{\left[ 0, 15 + \tau_5\right]}
        \left( 
            \psi_{1,1} \wedge \psi_{1,2}
        \right),\\
         \forall \tau_1 \in \left[0,1\right], \tau_2 \in\left[0,2\right],  \tau_3 \in\left[0,8\right],
         \tau_5 \in\left[0,5\right]
         \overset{\textrm{Thm. }\ref{thm:conjunction:alw}}{=}
         \\
        \left( 
            G_{\left[ 0, 15 + \tau_5\right]}\psi_{1,1} 
        \right)
        \wedge
        \left( 
            G_{\left[ 0, 15 + \tau_5\right]}\psi_{1,2}
        \right)
        ,\\
         \forall \tau_1 \in \left[0,1\right], \tau_2 \in\left[0,2\right],  \tau_3 \in\left[0,8\right],\tau_5 \in\left[0,5\right].
    \end{gathered}
\end{equation}
Therefore:
\begin{equation}\notag 
    \begin{gathered}
        \psi_1 U_{[15,20]}\psi_2 = 
        \\
        \left( 
            \left( 
                G_{\left[ 0, 15 + \tau_5\right]}\psi_{1,1} 
            \right)
            \wedge
            \left( 
                G_{\left[ 0, 15 + \tau_5\right]}\psi_{1,2}
            \right)
        \right)
        \wedge
        \\
        \left( 
            F_{\left[ 20 + \tau_4+\tau_5, 20 + \tau_4+\tau_5\right]}G_{\left[ 2,3\right]}\mu_3
        \right),  
        \\
         \forall \tau_4 \in \left[ 0,15 \right], \tau_5 \in \left[ 0,5 \right]
         =
         \\
        \left( 
            G_{\left[ 0, 15 + \tau_5\right]}F_{\left[ \tau_3, \tau_3\right]} G_{[0,\tau_2]}\mu_1
        \right)
        \wedge
        \\
        \left( 
            G_{\left[ 0, 15 + \tau_5\right]}F_{[1+\tau_1+\tau_2+\tau_3,1+\tau_1+\tau_2+\tau_3]} \mu_2 
        \right)
        \wedge
        \\
        \left( 
            F_{\left[ 20 + \tau_4+\tau_5, 20 + \tau_4+\tau_5\right]}G_{\left[ 2,3\right]}\mu_3
        \right),  
        \\
          \forall \tau_1 \in \left[0,1\right], \tau_2 \in\left[0,2\right],  \tau_3 \in\left[0,8\right], \tau_4 \in \left[ 0,15 \right], \tau_5 \in \left[ 0,5 \right].
    \end{gathered}
\end{equation}
Finally, 
\begin{equation}\notag 
    \begin{gathered}
        \phi = G_{[0,15]} \left(  F_{[0,15]} \left( \psi_1 U_{[15,20]}\psi_2  \right)\right)
        \\
        =
        G_{[0,15]} 
        \left[ F_{[0,15]} 
            \left[
                \left(
                    G_{\left[ 0, 15 + \tau_5\right]}F_{\left[ \tau_3, \tau_3\right]} G_{[0,\tau_2]}\mu_1
                \right)
                \wedge
                \right.\right.
                \\
                \left.\left.
                \left( 
                    G_{\left[ 0, 15 + \tau_5\right]}F_{[1+\tau_1+\tau_2+\tau_3,1+\tau_1+\tau_2+\tau_3]} \mu_2 
                \right)
                \wedge
                \right.\right.
                \\
                \left.\left.
                \left( 
                    F_{\left[ 20 + \tau_4+\tau_5, 20 + \tau_4+\tau_5\right]}G_{\left[ 2,3\right]}\mu_3
                \right)
            \right]
        \right],
        \\
      \forall \tau_1 \in \left[0,1\right], \tau_2 \in\left[0,2\right],  \tau_3 \in\left[0,8\right], \tau_4 \in \left[ 0,15 \right], \tau_5 \in \left[ 0,5 \right]
      \\
      \overset{\textrm{Thms. } \ref{thm:conjunction:event}, \ref{thm:conjunction:alw}}{=}
        \left(
           G_{[0,15]} F_{[\tau_6,\tau_6]}  G_{\left[ 0, 15 + \tau_5\right]}F_{\left[ \tau_3, \tau_3\right]} G_{[0,\tau_2]}\mu_1
        \right)
        \wedge
        \\
        \left( 
             G_{[0,15]} F_{[\tau_6,\tau_6]}G_{\left[ 0, 15 + \tau_5\right]}F_{[1+\tau_1+\tau_2+\tau_3,1+\tau_1+\tau_2+\tau_3]} \mu_2 
        \right)
        \wedge
        \\
        \left( 
             G_{[0,15]} F_{[\tau_6,\tau_6]}F_{\left[ 20 + \tau_4+\tau_5, 20 + \tau_4+\tau_5\right]}G_{\left[ 2,3\right]}\mu_3
        \right),
        \\
      \forall \tau_1 \in \left[0,1\right], \tau_2 \in\left[0,2\right],  \tau_3 \in\left[0,8\right], 
      \\
      \tau_4 \in \left[ 0,15 \right], \tau_5 \in \left[ 0,5 \right], \tau_6 \in \left[0,15 \right]
      \\
      \overset{\textrm{Prop. \ref{thm:nesting:eventually}}}{=}
    \left(
           G_{[0,15]} F_{[\tau_6,\tau_6]}  G_{\left[ 0, 15 + \tau_5\right]}F_{\left[ \tau_3, \tau_3\right]} G_{[0,\tau_2]}\mu_1
        \right)
        \wedge
        \\
        \left( 
             G_{[0,15]} F_{[\tau_6,\tau_6]}G_{\left[ 0, 15 + \tau_5\right]}F_{[1+\tau_1+\tau_2+\tau_3,1+\tau_1+\tau_2+\tau_3]} \mu_2 
        \right)
        \wedge
        \\
        \left( 
             G_{[0,15]} F_{\left[ 20 + \tau_4+\tau_5 + \tau_6, 20 + \tau_4+\tau_5+ \tau_6\right]}G_{\left[ 2,3\right]}\mu_3
        \right),
        \\
      \forall \tau_1 \in \left[0,1\right], \tau_2 \in\left[0,2\right],  \tau_3 \in\left[0,8\right], 
      \\
      \tau_4 \in \left[ 0,15 \right], \tau_5 \in \left[ 0,5 \right], \tau_6 \in \left[0,15 \right],
    \end{gathered}
\end{equation}
where each of the above fragments in conjunction are of the form $GF(\cdots GF(GF\mu))$ and can be handled through Thm. \ref{thm:nested:nested}. To demonstrate how the repetitive instances are realized, consider the fragment $G_{[0,15]} F_{[\tau_6,\tau_6]}  G_{\left[ 0, 15 + \tau_5\right]}F_{\left[ \tau_3, \tau_3\right]} G_{[0,\tau_2]}\mu_1$. Let $i_3\in\{1,\cdots,I_3\},I_3 \in \mathbb{N}$ and $i_6 \in \{1,\cdots,I_6\},I_6 \in \mathbb{N}$ denote counters of the repetitions of $F_{\left[ \tau_3, \tau_3\right]} G_{[0,\tau_2]}\mu_1$ (inner) and $ F_{[\tau_6,\tau_6]}G_{\left[ 0, 15 + \tau_5\right]}F_{\left[ \tau_3, \tau_3\right]} G_{[0,\tau_2]}\mu_1$ (outer) respectively. 
Further let
$\tau_3^{i_3}\in\left[ 0,8\right]$ denote the value of $\tau_3(t)$ at the instances where the function $h_1$ corresponding to $\mu_1$ becomes equal to zero at time $t_{i_3} \geq t_0$. As long as $\exists \tau_3^{i_3}: \sum_{i_3 = 1}^{I_3}\tau_3^{i_3}\leq 15 + \tau_5$ (per Thm. \ref{thm:nesting:alw:event:alw}, Eq. \eqref{eq:thm:nesting:alw:event:alw:terminal}) the fragment $F_{\left[ \tau_3, \tau_3\right]} G_{[0,\tau_2]}\mu_1$ is repeated. Let $\tau_6^{i_6}$ denote the value of $\tau_6^{i_6}(t)$ at time $t_1$ (for $i_3 = 1$), and assume that the inner task has elapsed, i.e., $\nexists \tau_3^{i_3}: \sum_{i_3 = 1}^{I_3}\tau_3^{i_3}\leq 15 + \tau_5$, if $\exists \tau_6^{i_6}: \sum_{i_6 = 1}^{I_6}\tau_6^{i_6}\leq 15$, the counter $i_3$ is reset to $1$ and the whole inner fragment is repeated.   
Any additional nestings beyond Thm. \ref{thm:nested:nested} are handled similarly by induction. 
}

\section{Simulations}
In this section we demonstrate the applicability of our method for modeling complex STL formulae and computing control laws for satisfying them. The examples were coded in MATLAB 2024b, on a laptop running the Windows 10 operating system and equipped with an Intel Core Ultra 7 165U CPU. To simulate the dynamics \eqref{eq:dynamics:state_enhanced} under the controller \eqref{eq:cbf:controller}, a first-order Euler integration scheme with norm control is employed. 
% To solve \eqref{eq:cbf:controller}, the constrained non-linear optimization function {\fontfamily{qcr}\selectfont fmincon} is employed for non-affine dynamics, while {\fontfamily{qcr}\selectfont quadprog} is employed for affine dynamics. 
Unless explicitly specified, $u_{\textrm{ref}}\in\mathbb{R}^m$ is set equal to zero, while $\omega_{\textrm{ref}} = -k_\omega \hat{\boldsymbol{\tau}}$, for $k_\omega >0$. This choice of stable dynamics for the parameters $\hat{\boldsymbol{\tau}}$ forces them (and thus the satisfaction instances) to decrease over time, if the constraints encoded in \eqref{eq:cbf:controller} allow it, promoting earlier predicate satisfaction when/if possible. Animations depicting the evolution of the trajectories of each system over time, are available in the GitHub repository: \url{https://github.com/KTH-DHSG/Operator-Based-STL}.
% The code that reproduces the simulations, as well as animations depicting the evolution of the trajectories of each system over time, are available in the following GitHub repository: \url{https://github.com/KTH-DHSG/Operator-Based-STL}.
\subsection{Non-Affine Dynamics}\label{sec:sims:subsec:nonaff}
Consider the system:
\begin{equation}\label{eq:dynamics:non-affine}
    \dot{x} = f(x,u) = -\tanh{(x)} + axu^3 + bu,
\end{equation}
with $a,b\in\mathbb{R}$ system parameters, state $x\in\mathbb{R},\ u\in \mathcal{U} = \left[ -0.5,0.5\right]\subset\mathbb{R}$. Then, Eq. \eqref{eq:hjb} yields:
\begin{equation}\notag 
    \begin{split}
        \frac{\partial V_h}{\partial t} +  \underset{u\in\mathcal{U}}{\max} \left\{\frac{\partial V_h}{\partial x}\left( -\tanh{(x)} + axu^3 + bu \right) \right\} &= 0
        \Leftrightarrow
        \\
        \frac{\partial V_h}{\partial t} - \frac{\partial V_h}{\partial x}\tanh{(x)} + \underset{u\in\mathcal{U}}{\max} \left\{\frac{\partial V_h}{\partial x}\left( axu^3 + bu \right) \right\} &= 0.
    \end{split}
\end{equation}
The maximizer of $\frac{\partial V_h}{\partial x}\left( axu^3 + bu \right)$ only depends on the sign of $\frac{\partial V_h}{\partial x}$. Since $axu^3 + bu$ is continuous and defined over a continuous compact interval (w.r.t. $u\in\mathbb{R}$), the maximizer, denoted by $u^\star \in \fundef{\mathbb{R}}{\mathcal{U}}$ attains values on the interval's boundary or at any interior critical points:
\(
    u^\star (x) = 
    \begin{cases}
        \pm 0.5 \\
        \pm u_{\textrm{crit}}\triangleq \pm \sqrt{\frac{-b}{3ax}}
    \end{cases}
\),
where for the second branch to be well-defined, $\frac{-b}{3ax} \geq 0$ and the corresponding maximizer(s) should be admissible: $ \pm \sqrt{\frac{-b}{3ax}} \in \mathcal{U}$ (otherwise only the boundary values are valid possible maximizers). Owing to convexity of $\setdef{x\in\mathbb{R}}{h(x)\geq 0}$ (equivalently concavity of $h$ w.r.t. $x$) and the fact that \eqref{eq:dynamics:non-affine} is one-dimensional, the sign of $\frac{\partial V_h}{\partial x}$ is the same as the sign of $\frac{\partial h}{\partial x}$. Therefore,  
\(
       u^\star(x) 
        = \underset{u^i \in \mathcal{U}^\star}{\arg\max}
            \left\{
                \frac{\partial h}{\partial x}\left( ax(u^i)^3 + bu^i \right)
            \right\},
\)
with 
\begin{equation}
    \mathcal{U}^\star
    =
    \begin{cases}
        \{\pm 0.5\}, & \frac{-b}{3ax}  < 0 ,
        \\
        \{\pm 0.5\}, & \frac{-b}{3ax} \geq0, \pm u_{\textrm{crit}}\notin \mathcal{U}
        \\
        \left\{\pm 0.5, u_{\textrm{crit}} \right\} , & \frac{-b}{3ax} \geq0,   u_{\textrm{crit}}\in\mathcal{U}, -u_{\textrm{crit}}\notin\mathcal{U}
        \vspace{0.1cm}
        \\
        \left\{\pm 0.5, \pm u_{\textrm{crit}} \right\} , & \frac{-b}{3ax} \geq0,   \pm u_{\textrm{crit}}\in\mathcal{U},
    \end{cases}.
    \notag 
\end{equation}
In order to compute the value function, note that according to \eqref{eq:hjb} the total derivative along trajectories $x^\star\in\fundef{\mathbb{R}}{\mathbb{R}}$ under the optimal input $u^\star$, is:
\begin{equation}\notag 
    \begin{split}
        &\frac{\textrm{d}V_h}{\textrm{d}t}
        =
        \frac{\partial V_h}{\partial t} +\frac{\partial V_h}{\partial x}\ f\left(x^\star(t),u^\star(x^\star(t))\right)
        =0
        \Leftrightarrow
        \\
        \int_{-t}^{0} &\frac{\textrm{d}V_h}{\textrm{d}t} \textrm{d}\tau 
        =
        V_h(x^\star(-t),-t) - V_h(x^\star(0),0) = 0
        \overset{\eqref{eq:hjb}}{\Longleftrightarrow} \\
        &V_h(x^\star(-t),-t) 
        = h(x^\star(0),0).
    \end{split}
\end{equation}
Therefore, for any $x\in\mathbb{R},t\in\mathbb{R}_{< 0}$ system \eqref{eq:dynamics:non-affine} is integrated from some $t<0$ under the optimal input until $t=0$ and the terminal value $h(x^\star(0),0)$ corresponds to the value function. A scaled version (for illustration purposes across a large range of values) of the value function for \eqref{eq:dynamics:non-affine} is depicted in Fig. \ref{fig:hjb_solution}.

\subsubsection{Case 1 - Simply Repeated Formula}\label{sec:sims:subsec:nonaff:case1}
Let $\mu_1,\mu_2$ denote simple predicates with the predicate functions $h_1,h_2:\fundef{\mathcal{X}}{\mathbb{R}}$ given by $h_1(x) = 10\left( 0.25^2 - (x - 1.0)^2 \right)$ and $h_2(x) = 10\left( 0.25^2 - (x - 1.75)^2 \right)$. Consider the nested formula
\(
    \phi = G_{[10]} F_{[0,4]} \left( \mu_1 U_{[1,2]} \left( F_{[1,2]} \mu_2  \right) \right),
\)
which, according to our theoretical results, boils down to a repeated execution of $F_{[0,4]} \left( \mu_1 U_{[1,2]} \left( F_{[1,2]} \mu_2  \right)\right)$ due to the outer nesting of $G_{[0,25]}$. The trajectory of the state and the parameters are depicted in Fig. \ref{fig:nonaffine:res1}.

\subsubsection{Case 2 - Simply Repeated Formula with Disjunction}
Consider the simple predicates $\mu_1,\mu_2,\mu_3$ with the predicate functions $h_1,h_2,h_3:\fundef{\mathcal{X}}{\mathbb{R}}$, where $h_1,h_2$ are the same as previously and $h_3 = 10\left( 0.2^2 - (x - 1.5)^2 \right)$. We consider the formula 
\(
   \phi =  G_{[0,10]} F_{[0,4]} \left( \mu_1 U_{[1,2]} \left( F_{[1,2]} \left( \mu_2 \bigvee G_{[0,1]}\mu_3\right)  \right) \right),
\)
In order to demonstrate the ability of the proposed method to satisfy disjunctions, we present three choices for $u_{\textrm{ref}}$ in Fig. \ref{fig:nonaffine:res2}: a) for $u_{\textrm{ref}} = +1$ in Subfig. \ref{fig:nonaffine:res2:a}, b) $u_{\textrm{ref}} = -1$ in Subfig. \ref{fig:nonaffine:res2:b} and c) $u_{\textrm{ref}}(t) = \sin(0.5t)$ in Subfig. \ref{fig:nonaffine:res2:c}. As a result of the design of $u_{\textrm{ref}}$, choice a) results in satisfaction of the predicate $\mu_2$ for all repetitions, choice b) results in satisfaction of the fragment $G_{[0,1]}\mu_3$ for all repetitions and choice c) results in $\mu_2$ being satisfied for the first repetition while $G_{[0,1]}\mu_3$ is satisfied for the rest of them. This is better illustrated in the animations provided in the GitHub repository: \url{https://github.com/KTH-DHSG/Operator-Based-STL}.
% \url{https://github.com/KTH-DHSG/Operator-Based-STL}. 

\begin{figure}[ht!]
    \centering
    \includegraphics[trim={0cm 0.0cm 0cm 0.0cm},clip,width= 1.0\linewidth]{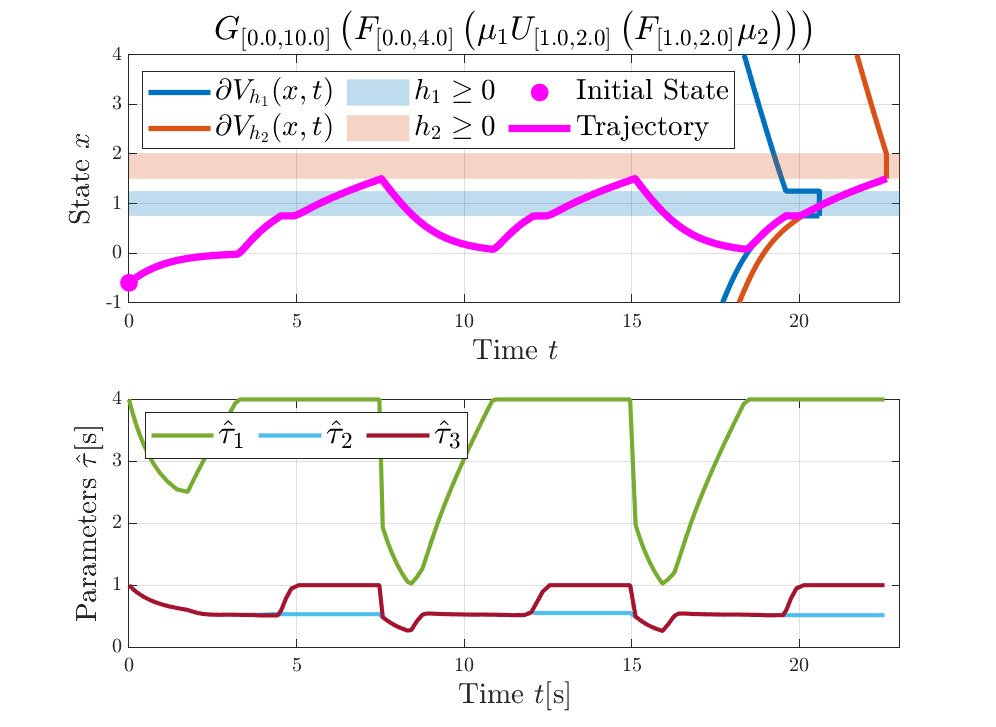}
    \caption{Evolution of the state (top figure) and the parameter values (bottom figure) for the formula $G_{[0,25]}\left( F_{[0,4]} \left( \mu_1 U_{[1,2]} \left( F_{[1,2]} \mu_2  \right) \right) \right)$. The zero super-level sets of the predicate functions are depicted through the blue and red-shaded regions. The parameters $\hat{\tau}_1, \hat{\tau}_3$ correspond to the satisfaction instances of the eventually operators respectively, while $\hat{\tau}_2$ corresponds to the free parameter of the until operator.}
    \label{fig:nonaffine:res1}
\end{figure}
\begin{figure}[ht!]
\centering
    \begin{subfigure}{1\linewidth}
    \centering
    \includegraphics[trim={1.cm 0.0cm 2cm 1.1cm},clip,width=0.95\linewidth]{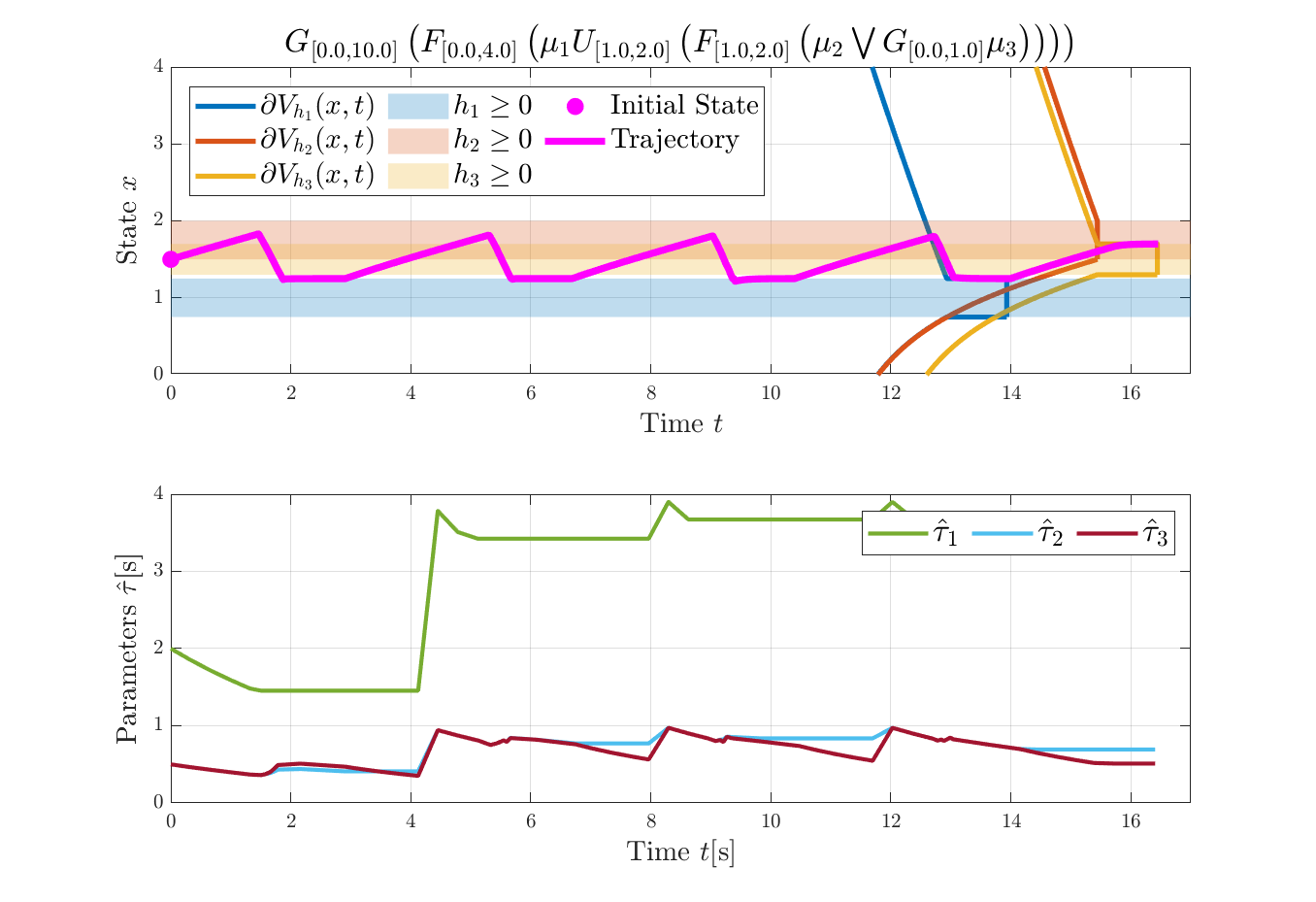}
    \caption{$u_{\textrm{ref}} = +1$.}
    \label{fig:nonaffine:res2:a}
    \end{subfigure}
     \\
    \begin{subfigure}{1\linewidth}
    \centering
    \includegraphics[trim={1.cm 0.0cm 1.5cm 1.1cm},clip,width=0.95\linewidth]{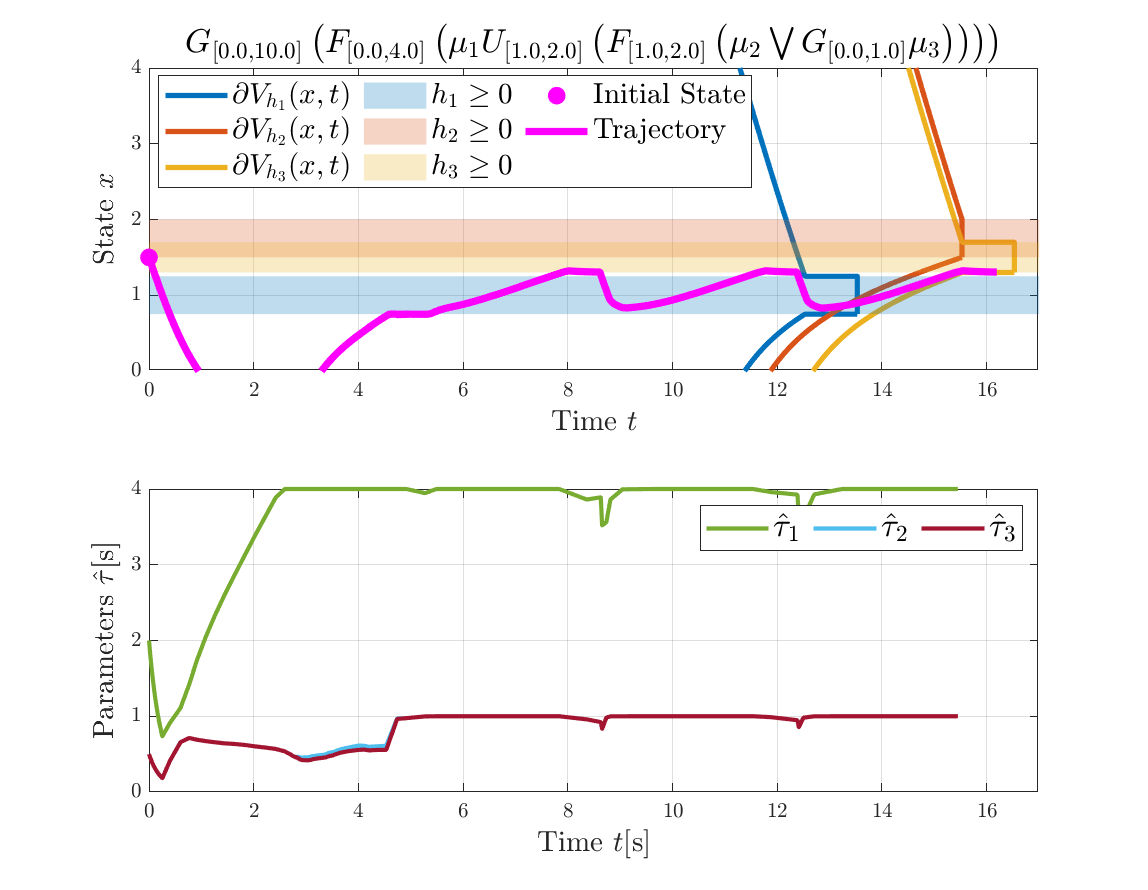}
    \caption{$u_{\textrm{ref}} = -1$.}
    \label{fig:nonaffine:res2:b}
    \end{subfigure}
    \\
    \begin{subfigure}{1\linewidth}
    \centering
    \includegraphics[trim={1.25cm 0.0cm 1.65cm 1.1cm},clip,width=0.95\linewidth]{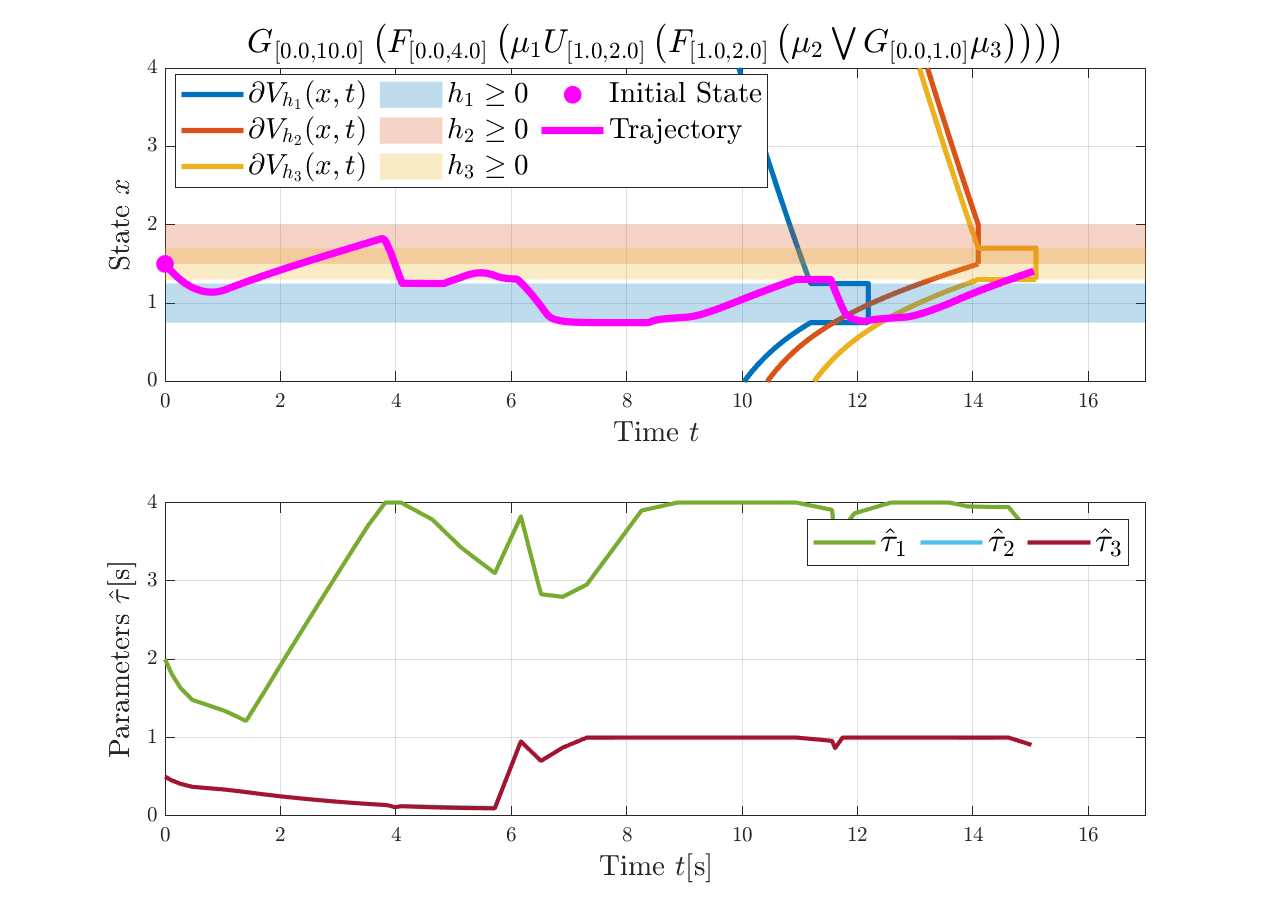}
    \caption{$u_{\textrm{ref}}(t) = \sin(0.5t)$.}
    \label{fig:nonaffine:res2:c}
    \end{subfigure}
\caption{Trajectories for the formula $G_{[0,10]}\left( F_{[0,4]} \left( \mu_1 U_{[1,2]} \left( F_{[1,2]} \left( \mu_2 \bigvee G_{[0,1]}\mu_3\right)  \right) \right) \right)$. The zero super-level sets of the predicate functions are depicted through the blue, red and yellow-shaded regions. The parameters $\hat{\tau}_1, \hat{\tau}_3$ correspond to the satisfaction instances of the eventually operators $F_{[0,4]}$ and $F_{[1,2]}$ respectively, while $\hat{\tau}_2$ corresponds to the until operator.}
\label{fig:nonaffine:res2}
\end{figure}

\subsection{Input-Affine Dynamics}\label{sec:sims:subsec:aff}
Consider the input-affine system:
\(
    \dot{x} = f(x,u) = -0.1\tanh{(x)} + a(0.5x + 1.0),
\)
with $a\in\mathbb{R}$ denoting a system parameter, state $x\in\mathbb{R}$, input $\ u\in \mathcal{U} = \left[ -0.5,0.5\right]\subset\mathbb{R}$. The value function for all predicate functions is computed in the same way as for the non-affine example. 

\subsubsection{Case 1 - Simply Repeated Formula with Conjunction}\label{sec:sims:subsec:aff:case1}
Let $\mu_1,\mu_2,\mu_3$ denote simple predicates with the predicate functions $h_1,h_2,h_3:\fundef{\mathcal{X}}{\mathbb{R}}$ given by $h_1(x) = 10\left( 0.25^2 - (x - 1.0)^2 \right)$, $h_2(x) = 10\left( 0.25^2 - x^2 \right)$ and $h_3(x) = 10\left( 0.2^2 - (x +0.75)^2 \right)$. Consider the STL fragments:
\(
        \psi_1 = 
         G_{[0,15]} F_{[0,5]} \left( \mu_1 U_{[1,2]} \left( F_{[1,2]}\mu_2  \right) \right),
         \psi_2 = 
         F_{[10,30]}G_{[0,1]}\mu_3,
\)
and the formula:
\(
    \phi = \psi_1 \bigwedge \psi_2,
\)
which boils down to a repeated execution of the fragment $\mu_1 U_{[1,2]} \left( F_{[1,2]}\mu_2  \right)$ while $G_{[0,1]}\mu_3$ should be satisfied eventually within the time window $[10,30]$. The trajectories of \eqref{eq:dynamics:state_enhanced} are depicted in Fig. \ref{fig:affine:res1}. Inspecting the bottom subfigure, the parameter corresponding to the fragment $\psi_2$ increases until its maximum value, which means that satisfaction of $G_{[0,1]}\mu_3$ is postponed in favor of satisfying $\psi_1$. This is not a unique solution, as it depends on the specific choice for the parameters' virtual control law as well as their initial condition. 
\subsubsection{Case 2 - Multiply Repeated Formula}\label{sec:sims:subsec:aff:case2}
Consider the simple predicates/predicate function pairs. We treat the following fragment:
\(
    \phi = G_{[0.0,15.0]} \left(  F_{[0.0,15.0]} \left[ \left(  F_{[0.0,8.0]} \left( \mu_1 U_{[0.0,2.0]} \left( F_{[1.0,2.0]} \mu_2 \right)\right) \right)\right.\right.
\)
\(
\left.\left.U_{[15.0,20.0]}\left( F_{[5.0,20.0]}  G_{[2.0,3.0]} \mu_3 \right)  \right]\right)
\).
Compared to the previous cases, this nesting forces multiple repetitions within the time window of the outer always operator $[0,15]$. A trajectory of the state is depicted in Fig. \ref{fig:affine:res}. 
Note that this formula requires several nestings, highlighting the ability of the proposed method to capture complex nested formulae with more than one repetitions.
\begin{figure}
    \centering
    \includegraphics[trim={1.5cm 0.0cm 1.5cm 0.0cm},clip,width= 1\linewidth]{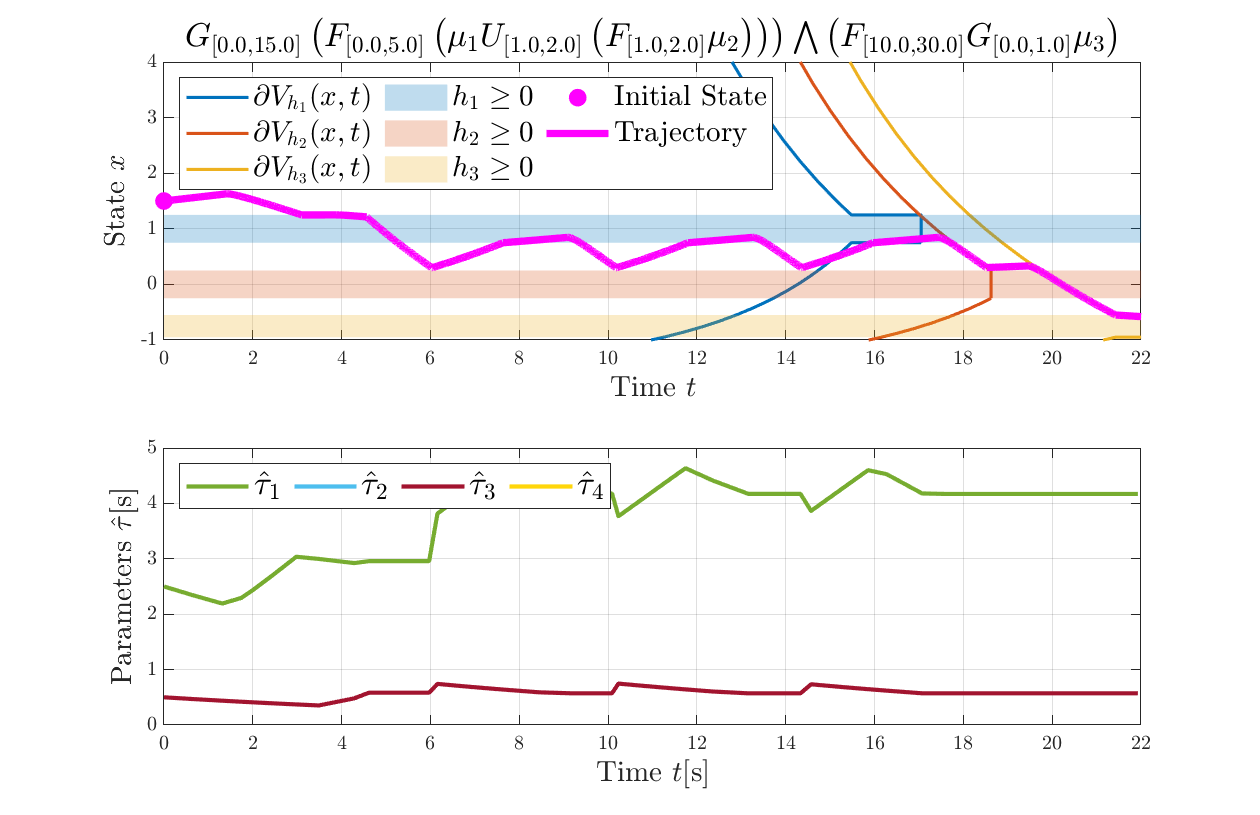}
    \caption{Evolution of the state (top figure) and the parameter values (bottom figure) for the formula $\psi_1 \bigwedge \psi_2$. The parameter $\hat{\tau}_1$ corresponds to the satisfaction instances of  $F_{[10,30]}$.}
    \label{fig:affine:res1}
\end{figure}
\begin{figure*}
\centering
    \includegraphics[trim={2.cm 0.2cm 2cm 0.0cm},clip,width= 1\linewidth]{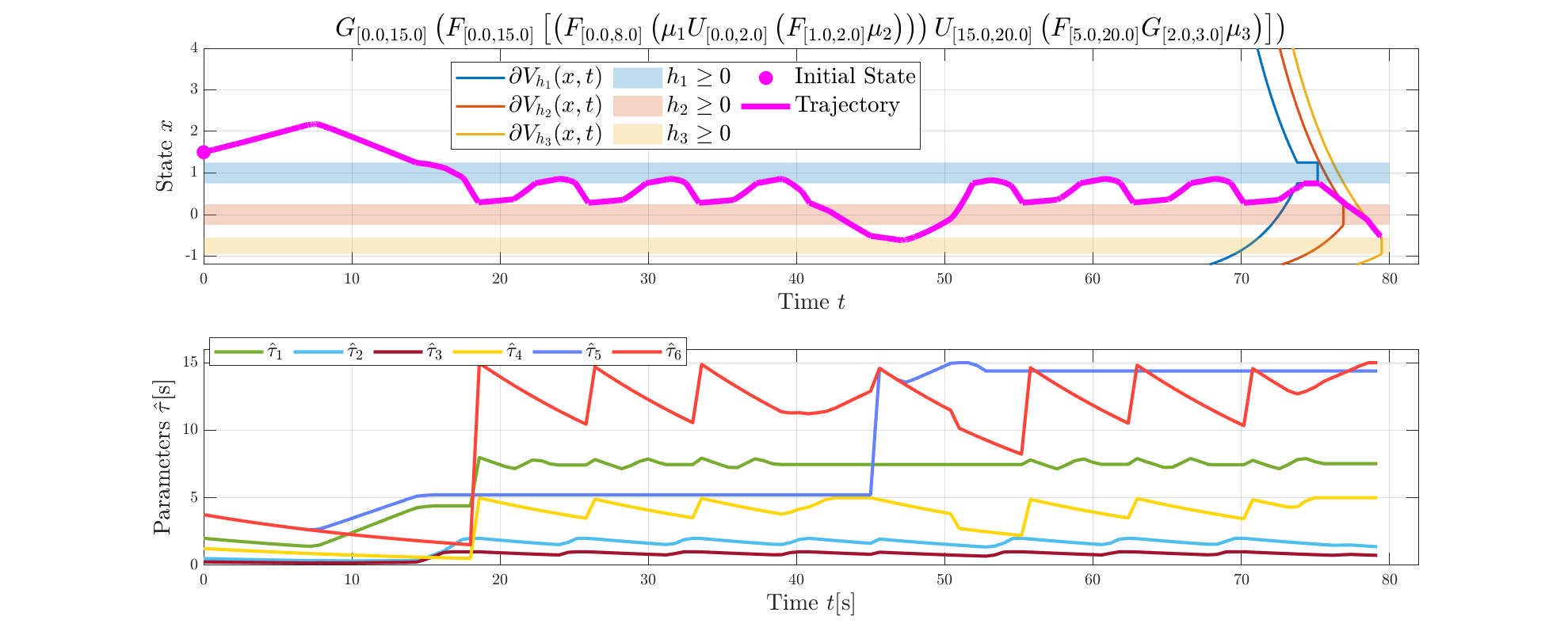}
    \label{fig:affine:res2a}
\caption{State evolution for $G_{[0.0,15.0]} \left(  F_{[0.0,15.0]} \left[ \left(  F_{[0.0,8.0]} \left( \mu_1 U_{[0.0,2.0]} \left( F_{[1.0,2.0]} \mu_2 \right)\right) \right)U_{[15.0,20.0]}\left( F_{[5.0,20.0]}  G_{[2.0,3.0]} \mu_3 \right)  \right]\right)
$ (top), and parameter evolution (bottom). The zero super-level sets of the predicate functions are depicted through the blue, red and yellow-shaded regions.}
\label{fig:affine:res}
\end{figure*}

\subsection{Linear System}
Finally, consider the linear system:
\(
    \dot{x} = 0.1x + u,
\)
with state $x\in\mathbb{R}$, input $\ u\in \mathcal{U} = \left[ -0.5,0.5\right]\subset\mathbb{R}$. The reachability value function is computed similarly to the previous subsections. The same simple predicates/predicate function pairs, with the STL formula:
\(
    \phi =
    \left( 
        G_{[0,15]}F_{[0,5]}
        \left( 
            \mu_1 U_{[1,2]}\left( F_{[1,2]}\mu_2 \right)
        \right)
    \right)
    \bigwedge
    \left( 
        F_{[0,30]}
        G_{[0,1]}\mu_3
    \right).
\)
We compare our method, depicted in Fig. \ref{fig:linear}, against a custom optimization framework based on genetic algorithms with gradient-based local refinement. The objective for the genetic algorithm is to maximize the robustness of the STL formula \citebr{robsem_2} over a discretized trajectory of the system. 
Among many trials and after extensive parameter tuning, the best solution's robustness is $\rho_{\max} = -1.89$ (i.e., the formula is not satisfied). Initially, the same problem was attempted through an MILP approach \citebr{bombara_decision_2016}, however even very rough discretizations of the trajectory yielded problems too large to handle. The same was true for a custom MILP formulation of the problem. This is to be expected, as the employed formula evolves over a large time horizon and thus requires a large number of MILP variables.  
\begin{figure}
    \centering
    \includegraphics[trim={1.25cm 0cm 1.75cm 0.0cm},clip,width=1.0\linewidth]{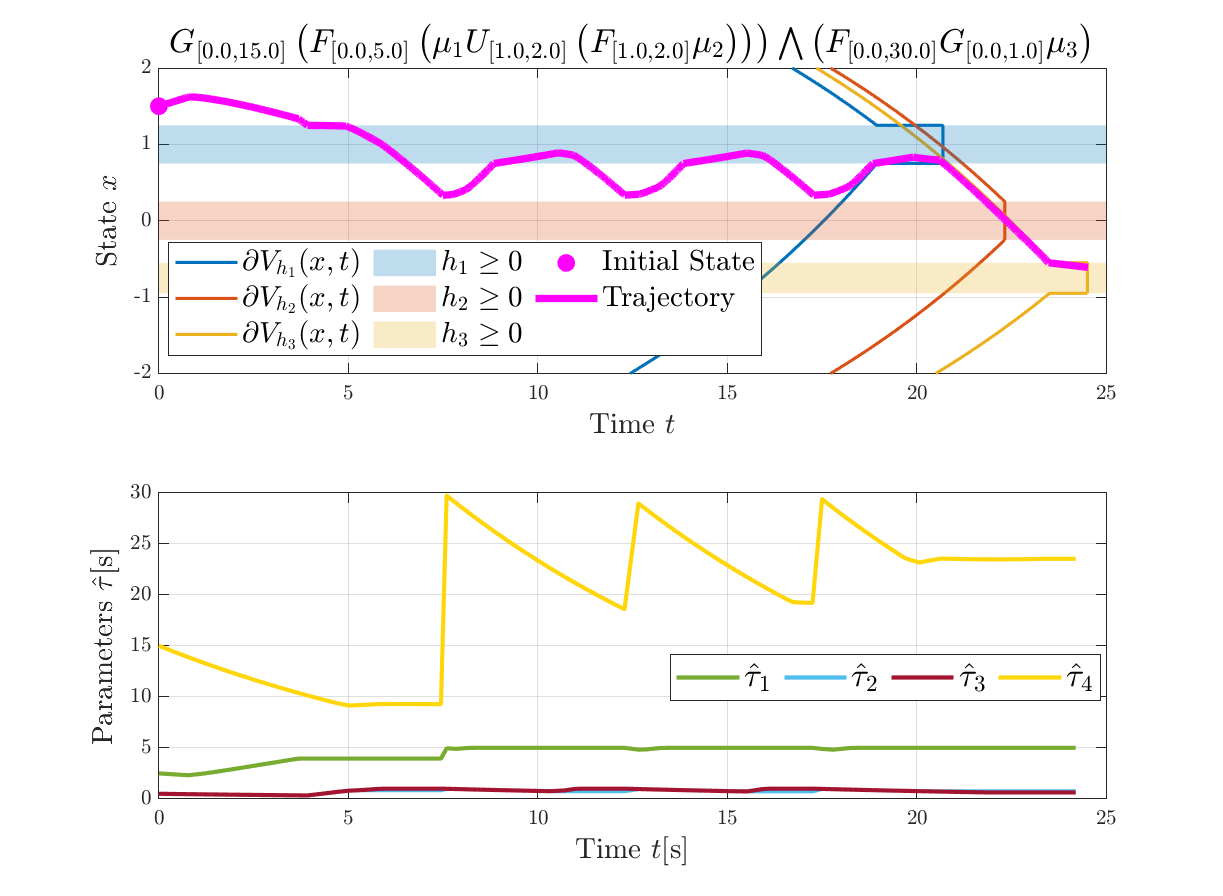}
    \caption{Evolution of the state for our method in the comparative simulation.}
    \label{fig:linear}
\end{figure}

\section{Conclusions \& Future Works}
In this work, we have proposed a novel approach to STL for verification and control synthesis. The method is based on the action of an novel operator, acting on reachability value functions, which effectively provides necessary and sufficient conditions for STL formulae satisfaction. The expressive power of the operator is demonstrated both in theory and simulations, able to capture complex repeated formulae. In future works, we will focus on feasibility analysis on the set of constraints \eqref{eq:constraints:final} associated with the proposed CBF-STL operator as well as applications of the operator in absence of HJB solutions; note how the operator acts on state and time dependent functions, hence can also be applied on CBFs instead of value functions. We also aim at exploring properties of the enhanced state space model \eqref{eq:dynamics:state_enhanced}, as well as employing more sophisticated control methods, e.g., optimal control, for simultaneous planning and control. More specifically, the satisfaction of simple predicates and the resetting of repeated (recursive) fragments are \emph{discrete} events. Therefore, we aim at formalizing and optimizing the CBF-STL operator framework using hybrid systems theory \cite{ROUSSEAS2026101731}.  

% FIX 2: Restored to SageV to match the template's included .bst file
\bibliographystyle{SageH.bst}
\bibliography{bib.bib}

\section*{Appendix}
\subsubsection{Proof of Prop. \ref{prop:operator:sequence}}
\begin{proof}
    First, note that $\forall \tau_1\in\Theta_1,\tau_2\in\Theta_2:\left[ 0, \alpha_2(\tau_2)\right] = \left[ 0, \alpha_1(\tau_1)\right]\cup \left[ \alpha_1(\tau_1), \min \left\{ \alpha_2(\tau_2) , \beta_1(\tau_1) \right\}\right]$.
    For any $t\in [0,\alpha_1]$ {(the arguments of the functions are omitted for clarity)}, select $t' = t + \alpha_{2} - \alpha_{1}$ which implies that $t' \geq t \geq 0$. We get $\forall t\in [0,\alpha_1]: \mathbf{T}_{\alpha_1,\beta_1}^{\Theta_1}V_h(x,t) = V_h(x,t - \alpha_1)$
    % \begin{equation}
    %     \begin{split}
    %         \forall t\in [0,\alpha_1]: \mathbf{T}_{\alpha_1,\beta_1}^{\Theta_1}V_h(x,t) = V_h(x,t - \alpha_1), 
    %     \end{split}
    % \end{equation}
    and% $\forall t\in [0,\alpha_1] \Rightarrow t' \in [\alpha_{2} - \alpha_{1},\alpha_{2}] $: $ \mathbf{T}_{\alpha_2,\beta_2}^{\Theta_2}V_h(x,t') = V_h(x,t'- \alpha_2) = V_h(x,t - \alpha_1)= \mathbf{T}_{\alpha_1,\beta_1}^{\Theta_1}V_h(x,t), \forall x \in \mathcal{X}$.
    \begin{equation}\notag 
        \begin{split}
             \mathbf{T}_{\alpha_2,\beta_2}^{\Theta_2}V_h(x,t') &= V_h(x,t'- \alpha_2) = V_h(x,t - \alpha_1)\\
               &= \mathbf{T}_{\alpha_1,\beta_1}^{\Theta_1}V_h(x,t), \forall x \in \mathcal{X}.
        \end{split} 
    \end{equation}
    Therefore, through Prop. \ref{prop:valu_fun:monotone} and \eqref{eq:operator:action}: 
    \begin{equation}\notag 
        \begin{split}
            \alpha_2 \geq t' \geq t\Leftrightarrow 
            0\geq  t' - \alpha_2 \geq t - \alpha_2  \Leftrightarrow
            \\
            V_h(x,t' - \alpha_2) \leq V_h(x,t - \alpha_2) \Leftrightarrow 
            \\
            \mathbf{T}_{\alpha_1,\beta_1}^{\Theta_1}V_h(x,t) \leq \mathbf{T}_{\alpha_2,\beta_2}^{\Theta_2}V_h(x,t).
        \end{split}
    \end{equation}
    Also, $\forall t \in \left[ \alpha_1(\tau_1), \min \left\{ \alpha_2(\tau_2) , \beta_1(\tau_1) \right\}\right]$:
    $\alpha_1(\tau_1)\leq t \leq  \beta_1(\tau_1) \Rightarrow \mathbf{T}_{\alpha_1,\beta_1}^{\Theta_1}V_h(x,t) = h(x)$,
    and 
    $
    \mathbf{T}_{\alpha_2,\beta_2}^{\Theta_2}V_h(x,t) = V_h(x,t - \alpha_2) \underset{\textrm{Prop.} \ref{prop:valu_fun:monotone}}{ \overset{t - \alpha_2 \leq 0}{\Longleftrightarrow}}
    \mathbf{T}_{\alpha_2,\beta_2}^{\Theta_2}V_h(x,t) \geq V_h(x,0) = h(x) = \mathbf{T}_{\alpha_1,\beta_1}^{\Theta_1}V_h(x,t)
    $.
    Therefore, $\forall t\in\left[ 0, \min \left\{ \alpha_2(\tau_2) , \beta_1(\tau_1) \right\}\right],\forall \tau_1 \in \Theta_1,\tau_2\in\Theta_2$: 
    $\mathbf{T}_{\alpha_2,\beta_2}^{\Theta_2}V_h(x,t) \geq \mathbf{T}_{\alpha_1,\beta_1}^{\Theta_1}V_h(x,t)$, which directly implies: $\mathbf{T}_{\alpha_1 ,\beta_1}^{\Theta_1}V_h(x,t) \geq 0 \Rightarrow \mathbf{T}_{\alpha_{2} ,\beta_{2}}^{\Theta_{2}}V_h(x,t)\geq 0$. 
\end{proof}
\vspace{0.2cm}
\subsubsection{Proof of Prop \ref{prop:always}}
\begin{proof}
    \emph{(Suff.     $\Rightarrow$)} From Def. \ref{def:stl}:
    \begin{equation}\notag 
        (x,t_0) \models G_{\left[\underline{t},\overline{t}\right]}\mu  \Leftrightarrow \forall t' \in \left[t_0+\underline{t}, t_0+\overline{t}\right]: (x,t')\models\mu,
    \end{equation}
    where $(x,t')\models\mu \Leftrightarrow h(x(t')) \geq 0$. W.l.o.g., we choose the initial time $t_0 = 0$. We have to show that:
    \begin{equation}\notag 
        \begin{gathered}
            \forall t \in \left[ 0,\overline{t}\right]:\
            \mathbf{T}_{\underline{t},\overline{t}}^{\emptyset} V_h\left(x\left(t\right),t\right) \geq 0
            \Rightarrow
            (x,0) \models G_{\left[\underline{t},\overline{t}\right]}\mu \Leftrightarrow
            \\
            \forall t \in \left[\underline{t}, \overline{t}\right] :\ h(x(t)) \geq 0.
        \end{gathered}
    \end{equation}
    We have: 
    \(
    \forall t \in \left[ 0,\overline{t}\right]:
    \mathbf{T}_{\underline{t},\overline{t}}^{\emptyset} V_h\left(x\left(t\right),t\right) \geq 0
            \Rightarrow
            \forall t \in \left[ \underline{t},\overline{t}\right]:
            \mathbf{T}_{\underline{t},\overline{t}}^{\emptyset} V_h\left(x\left(t\right),t\right) \geq 0 
            \overset{\eqref{eq:operator:action}}{\Longleftrightarrow}  h(x(t)) \geq 0
    \).
    \newline
    \emph{(Necessity $\Leftarrow$)}
    % From Eq. \eqref{eq:prop:eventualy:toprove},
    For the trajectory $x\in\fundef{\mathbb{R}}{\mathcal{X}}$, which corresponds to a solution to \eqref{eq:dynamics}, consider the input signal $u\in\fundef{\mathbb{R}}{\mathcal{U}}$ that yields the value function $V_h \in \fundef{\mathcal{X}\times\mathbb{R}}{\mathbb{R}}$ \eqref{eq:vfun:def}. Then $\forall t \in \left[0,\underline{t}\right]$, from Prop. \ref{prop:valu_fun:monotone}: 
    \[
    \begin{gathered}
        \mathbf{T}_{\underline{t},\overline{t}}^{\emptyset} V_h(x(t),t)
        \overset{\eqref{eq:operator:action}}{=}
            V_h(x(t),t - \underline{t}) 
            \\
            \overset{\textrm{Prop.} \ref{prop:valu_fun:monotone}}{\geq}  
            V_h(x(\underline{t}),0)
            \overset{\eqref{eq:hjb}}{=}
            h(x(\underline{t}))
            \geq 0.
    \end{gathered}
    \]
    For $t\in \left[\underline{t},\overline{t}\right]$, from \eqref{eq:operator:action}, $\mathbf{T}_{\underline{t},\overline{t}}^{\emptyset} V_h(x(t),t) = h(x(t)) \geq 0$, concluding the proof.
\end{proof}
\subsubsection{Proof of Prop. \ref{prop:eventually}}
\begin{proof}
    \emph{(Sufficiency $\Rightarrow$)} By definition for the eventually operator:
    \(
        (x,t_0) \models F_{\left[\underline{t},\overline{t}\right]}\mu  \Leftrightarrow \exists t' \in \left[t_0+\underline{t}, t_0+\overline{t}\right], (x,t')\models\mu,
    \)
    where $(x,t')\models\mu \Leftrightarrow h(x(t')) \geq 0$. W.l.o.g., we choose the initial time $t_0 = 0$. We have to show that:
    \begin{equation}\notag 
        \begin{gathered}
            \forall t \in \left[ 0,\overline{t}\right]:\
            \mathbf{T}_{\underline{t}+\tau ,\underline{t}+\tau }^{\left[ 0, \overline{t} - \underline{t}\right]} V_h\left(x\left(t\right),t\right) \geq 0 
            \Rightarrow
            (x,0) \models F_{\left[\underline{t},\overline{t}\right]}\mu
            \\
            \Leftrightarrow
            \exists t \in \left[\underline{t} , \overline{t} \right] :\ h(x(t)) \geq 0.
        \end{gathered}
    \end{equation}
     Choosing $t' = \underline{t} + \tau \Rightarrow t' \in \left[ \underline{t},\overline{t}\right],\forall \tau \in \left[ 0,\overline{t} - \underline{t}\right]$, and therefore:
    \begin{equation}\notag 
        \begin{split}
            \forall t \in \left[ 0,\overline{t}\right]:\
            &\mathbf{T}_{\underline{t}+\tau ,\underline{t}+\tau }^{\left[ 0, \overline{t} - \underline{t}\right]} V_h\left(x\left(t\right),t\right) \geq 0
            \overset{\textrm{Prop.} \ref{prop:valu_fun:monotone}}{\Longleftrightarrow}
            \\
            &\mathbf{T}_{\underline{t}+\tau ,\underline{t}+\tau }^{\left[ 0, \overline{t} - \underline{t}\right]} V_h\left(x\left(t'\right),t'\right) \geq 0 
            \overset{\eqref{eq:operator:action}}{\Longleftrightarrow}  h(x(t')) \geq 0.
        \end{split}
    \end{equation}
    This means that $\forall \tau \in \left[ 0,\overline{t} - \underline{t}\right]$:
    \begin{equation}\label{eq:prop:eventualy:toprove}
        \begin{split}
             \exists t'\in \left[ \underline{t},\overline{t}\right]: 
            h(x(t')) \geq 0 \Leftrightarrow (x,0) \models F_{\left[\underline{t},\overline{t}\right]}\mu.
        \end{split}
    \end{equation}
    \emph{(Nec. $\Leftarrow$)}
    From Eq. \eqref{eq:prop:eventualy:toprove}, for the trajectory $x\in\fundef{\mathbb{R}}{\mathcal{X}}$, which corresponds to a solution to \eqref{eq:dynamics}, consider the input signal $u\in\fundef{\mathbb{R}}{\mathcal{U}}$ that yields the value function $V_h \in \fundef{\mathcal{X}\times\mathbb{R}}{\mathbb{R}}$ \eqref{eq:vfun:def}. Then $\forall \tau \in \left[ 0,\overline{t} - \underline{t}\right], \forall t \in \left[0,\underline{t} + \tau\right]$, from Prop. \ref{prop:valu_fun:monotone}: 
    \(
         \mathbf{T}_{\underline{t}+\tau,\underline{t}+\tau}^{\left[0,\overline{t} - \underline{t}\right]} V_h(x(t),t)
         \overset{\eqref{eq:operator:action}}{=}
            V_h(x(t),t - (\underline{t}+\tau)) 
            \geq  V_h(x(\underline{t}+\tau),0) 
            \overset{\eqref{eq:hjb}}{=}
            h(x(t'))
            \geq 0
    \).
\end{proof}
\subsubsection{Proof of Prop. \ref{prop:nesting:event:alw}}
\begin{proof}
    We have
    \begin{equation}\label{thm:nesting:event:equiv:2:star}
       \begin{gathered}
            (x,t_0)\models F_{[ \underline{t}', \overline{t}']}G_{\left[\underline{t},\overline{t}\right]}\mu 
            \Leftrightarrow
            \\
            \forall t_0 \in\mathbb{R}:
            \exists t \in \left[ t_0 + \underline{t}',t_0 + \overline{t}'\right]:
            (x,t)\models G_{\left[\underline{t},\overline{t}\right]}\mu 
            {\Leftrightarrow}
            \\
            \forall t_0 \in\mathbb{R}:
            \exists t' \in \left[ t_0 + \underline{t}',t_0 + \overline{t}'\right]:
            \forall t \in \left[ t' + \underline{t},t' + \overline{t}\right]:
            \\
            h(x(t)) \geq 0 
            \Leftrightarrow
            \\
            \forall t_0 \in\mathbb{R}:
            \exists \tau \in \left[ 0, \overline{t}' - \underline{t}'\right] :  
            \\
            t' = t_0 + \underline{t}'+\tau \in
            \left[ t_0 + \underline{t}',t_0 + \overline{t}'\right]:
            \\
            \forall t \in \left[t' + \underline{t},t' + \overline{t}\right]:
            h(x(t)) \geq 0 
            \Leftrightarrow
            \\
            \forall t_0 \in\mathbb{R}:
            \exists \tau \in \Theta : 
            \forall t \in \left[ t_0+\underline{t}'+\tau  + \underline{t},t_0+\underline{t}'+\tau  + \overline{t}\right]:
            \\
            h(x(t)) \geq 0
            \Leftrightarrow
            \\
            \forall t_0 \in\mathbb{R}:
            \exists \tau \in \Theta : 
            \forall t \in \left[ t_0+\alpha(\tau),t_0+\beta(\tau)\right]:
            \\
            h(x(t)) \geq 0.
       \end{gathered}
       \tag{\ref{thm:nesting:event:equiv:2}*}
    \end{equation}
    Through the action of the operator \eqref{eq:operator:action} and Prop. \ref{prop:operator:reachable}:
    \begin{equation}\label{thm:nesting:event:equiv:1:star}
        \begin{gathered}
            \exists \tau(t_0) \in \fundef{\mathbb{R}}{\Theta}:
            \\
            \forall t \in\left[ t_0,t_0 + \beta(\tau(t_0))\right]:
            \mathbf{T}_{\alpha,\beta}^{\Theta}V_h(x,t;t_0) \geq 0
            \Leftrightarrow
            \\
            \forall t \in 
            \left[ t_0+\alpha(\tau(t_0)),t_0 + \beta(\tau(t_0))\right]:
            h(x(t)) \geq 0.
        \end{gathered}
        \tag{\ref{thm:nesting:event:equiv:1}*}
    \end{equation}
    Note now that, comparing \eqref{thm:nesting:event:equiv:1:star} with \eqref{thm:nesting:event:equiv:2:star}:
    \begin{equation}
        \forall t_0 \in\mathbb{R}:
        \exists \tau \in \Theta 
        \Leftrightarrow
        \exists \tau(t_0) \in\fundef{\mathbb{R}}{\Theta},
    \end{equation}
    which renders \eqref{thm:nesting:event:equiv:1:star}, \eqref{thm:nesting:event:equiv:2:star} identical, showing directly that \eqref{thm:nesting:event:equiv:1:star} $\Leftrightarrow$ \eqref{thm:nesting:event:equiv:2:star}, hence \eqref{thm:nesting:event:equiv:1} $\Leftrightarrow$ \eqref{thm:nesting:event:equiv:2}, concluding the proof. 
\end{proof}
\subsubsection{Proof of Thm. \ref{thm:nesting:alw:event:alw}}
\begin{proof}
   We have:
   \begin{equation}
       \begin{gathered}
           (x,t_0) \models \phi = 
            G_{[\alpha''(\tau''),\beta''(\tau'')]} F_{[\underline{t}',\overline{t}']}\psi
            \Leftrightarrow
            \\
            \forall t'' \in 
            \left[
                t_0 + \alpha''(\tau''),
                t_0 + \beta''(\tau'')
            \right]:
            \\
            \exists t' \in 
            \left[
                t'' + \underline{t}',
                t'' + \overline{t}'
            \right]:
            (x,t') \models \psi
            \overset{\eqref{thm:nesting:alw:event:alw:eq:psi}}{\Leftrightarrow}
            \\
            \forall t'' \in 
            \left[
                t_0 + \alpha''(\tau''),
                t_0 + \beta''(\tau'')
            \right]:
            \\
            \exists t' \in 
            \left[
                t'' + \underline{t}',
                t'' + \overline{t}'
            \right]:
            \\
            \forall t\in
              \left[ 
                t', t' + \beta(\tau)
            \right]:
            \mathbf{T}_{\alpha,\beta}^{\Theta}V_h(x,t;t')\geq 0.
       \end{gathered}
   \end{equation}
   Define the interval for $J\in\{1,\cdots,\bar{J}-1\}$:
   \begin{equation}
       \mathcal{I}_J = 
       \left[
            t_0 + \alpha''(\tau'') + s_{J-1}',
            t_0 + \alpha''(\tau'') + s_{J}'
       \right],
   \end{equation}
   for which:
   \begin{equation}
       \begin{gathered}
           \left[
            t_0 + \alpha''(\tau''),
            t_0 + \beta''(\tau'')
            \right]
            =
            \\
            \left(
                \bigcup_{J=1}^{\bar{J}-1}\mathcal{I}_{J}
            \right)
            \cup
            \left[
                t_0 + \alpha''(\tau'') + s_{\bar{J}-1}',
                t_0 + \beta''(\tau'')
            \right].
       \end{gathered}
   \end{equation}
   We first consider the interval $\bigcup_{J=1}^{\bar{J}-1}\mathcal{I}_{J}$:
   \begin{equation}\label{eq:thm:nesting:alw:event:alw:ind1}
       \begin{gathered}
            (x,t_0) \models \phi \Leftrightarrow
            \forall J \in 
            \{1,\cdots,\bar{J}\}:
            \\
            \forall t_J'' \in 
            \left[
                t_0 + \alpha''(\tau'') + s_{J-1}',
                t_0 + \alpha''(\tau'') + s_{J}'
            \right]:
            \\
            \exists t' \in 
            \left[
                t_J'' + \underline{t}',
                t_J'' + \overline{t}'
            \right]:
            \\
            \forall t\in
              \left[ 
                t', t' + \beta(\tau)
            \right]:
            \mathbf{T}_{\alpha,\beta}^{\Theta}V_h(x,t;t')\geq 0
            \Leftrightarrow
            \\
             \forall t_J'' \in 
             \mathcal{I}_J:
            \\
            \exists \tau' \in \left[ 0 , \overline{t}' - \underline{t}' \right]:
            t' =  t_J'' + \underline{t}' + \tau' \in
            \left[
                t_J'' + \underline{t}',
                t_J'' + \overline{t}'
            \right]:
            \\
            \forall t\in
              \left[ 
                t', t' + \beta(\tau)
            \right]:
            \mathbf{T}_{\alpha,\beta}^{\Theta}V_h(x,t;t')\geq 0.
       \end{gathered}
   \end{equation}
   However, $\forall t_J'' \in \mathcal{I}_J :\exists \tau' \in \left[ 0 , \overline{t}' - \underline{t}' \right]  \triangleq \Theta' \Leftrightarrow \exists \tau_J'\in\fundef{\mathcal{I}_J}{\Theta'}$. Choosing $\tau_J'(t_J'') \triangleq t_0 + \alpha''(\tau'') + s_J' - t_J''$ and replacing $t'$ with its experession in \eqref{eq:thm:nesting:alw:event:alw:ind1}, the above yields $ \forall t_J'' \in \mathcal{I}_J$:
   \begin{equation}\label{eq:thm:nesting:alw:event:alw:ind2}
       \begin{gathered}
            \forall t\in
              \left[ 
                t', t' + \beta(\tau)
            \right]:
            \mathbf{T}_{\alpha,\beta}^{\Theta}V_h(x,t;t')\geq 0
            \Leftrightarrow
            \\
            \forall t\in
              \left[ 
                t_J'' + \underline{t}' + \tau', t_J'' + \underline{t}' + \tau' + \beta(\tau)
            \right]:
            \\
            \mathbf{T}_{\alpha,\beta}^{\Theta}V_h(x,t;t_J'' + \underline{t}' + \tau')\geq 0
            \overset{\tau' = \tau_J'(t_J'')}\Leftrightarrow
            \\
             \forall t\in
              \left[ 
                t_0 + \alpha''(\tau'') + s_J' + \underline{t}' ,
                t_0 + \alpha''(\tau'') + s_J' + \underline{t}' + \beta(\tau)
            \right]:
            \\
            \mathbf{T}_{\alpha,\beta}^{\Theta}V_h(x,t;t_0 + \alpha''(\tau'') + s_J' + \underline{t}')\geq 0
            \overset{\mathrm{Prop. } \ref{prop:operator:reachable}}{\Leftrightarrow}
            \\
             \forall t\in
              \left[ 
                t_0 + \alpha'' + s_J' + \underline{t}' + \alpha(\tau),
                t_0 + \alpha'' + s_J' + \underline{t}' + \beta(\tau)
            \right]:
            \\
            \mathbf{T}_{\alpha,\beta}^{\Theta}V_h(x,t;t_0 + \alpha''(\tau'') + s_J' + \underline{t}')\geq 0.
       \end{gathered}
   \end{equation}
   Furthermore:
   \begin{equation}
       \begin{gathered}
           \mathbf{T}_{\alpha,\beta}^{\Theta}V_h(x,t;t_0 + \alpha'' + s_J' + \underline{t}') = 
           \\
           \begin{cases}
               V_h
               \left( 
                    x,t - \alpha(\tau) - t_0 -\alpha'' - s_J' - \underline{t}'
               \right),
               \\
               t \in 
               \left[ 
                    t_0 + \alpha'' + s_J' + \underline{t}', 
                    t_0 + \alpha'' + s_J' + \underline{t}' + \alpha(\tau)
                \right],
               \\
               \\
               h(x(t)),
               \\
               t\in
               \left[ 
                    t_0 + \alpha'' + s_J' + \underline{t}' + \alpha(\tau),
                    t_0 + \alpha'' + s_J' + \underline{t}' + \beta(\tau)
               \right]
           \end{cases}  
           \\
           \overset{\eqref{eq:thm:nesting:alw:event:alw:params}}{=}
           \begin{cases}
               V_h
               \left( 
                    x,t - t_0  - \alpha_J
               \right), 
               t\in
               \left[ 
                    t_0, t_0 + \alpha_J
               \right]
               \\
                h(x(t)), 
                t\in
               \left[ 
                    t_0 + \alpha_J, t_0 + \beta_J
               \right]
           \end{cases}
           \\
           =
           \mathbf{T}_{\alpha_J,\beta_J}^{\Theta_J}V_h(x,t;t_0).
       \end{gathered}
   \end{equation}
   Therefore, Eq. \eqref{eq:thm:nesting:alw:event:alw:ind2} yields:
   \begin{equation}
       \begin{gathered}
           (x,t_0) \models \phi \Leftrightarrow
            \forall J \in 
            \{1,\cdots,\bar{J}\}:
            \forall t_J'' \in \mathcal{I}_J
            \\
            \forall t\in
              \left[ 
                 t_0 + \alpha_J ,
                 t_0 + \beta_J
            \right]:
            \mathbf{T}_{\alpha_J,\beta_J}^{\Theta_J}V_h(x,t;t_0) \geq 0
            \\
            \overset{\mathrm{Prop. } \ref{prop:operator:reachable}}{\Leftrightarrow}
            \forall t\in
              \left[ 
                 t_0  ,
                 t_0 + \beta_J
            \right]:
            \mathbf{T}_{\alpha_J,\beta_J}^{\Theta_J}V_h(x,t;t_0) \geq 0.
       \end{gathered}
   \end{equation}
   Furthermore, note that $\alpha_1,\alpha_2,\cdots,\alpha_{\bar{J}}$ is an increasing sequence. There are two cases: 1) $t_0 + \beta_{J-1} < t_0 + \alpha_{J}$, 2) $t_0 + \beta_{J-1} \geq t_0 + \alpha_{J}$.
   \par 
   Case 1): Let $t \geq  t_0 + \alpha_{J}>t_0 + \beta_{J-1}$, implying that instances of the operator for $j<J$ have elapsed. From Prop. \ref{prop:operator:sequence}, for any $J<\bar{J}$:
   \begin{equation}
       \begin{gathered}
           \mathbf{T}_{\alpha_J,\beta_J}^{\Theta_J}V_h \leq 
           \mathbf{T}_{\alpha_{J+1},\beta_{J+1}}^{\Theta_{J+1}}V_h
           \leq \cdots \leq 
            \mathbf{T}_{\alpha_{\bar{J}},\beta_{\bar{J}}}^{\Theta_{\bar{J}}}V_h,
       \end{gathered}
   \end{equation}
   hence
   \begin{equation}
           \mathbf{T}_{\alpha_J,\beta_J}^{\Theta_J}V_h 
        \geq 0 
        \Leftrightarrow
        \underset{j\in\{J,\cdots,\bar{J}\}}{\min}
       \left\{
           \mathbf{T}_{\alpha_j,\beta_j}^{\Theta_j}V_h 
        \right\}\geq 0,
   \end{equation}
   concluding the proof for the interval $\bigcup_{J=1}^{\bar{J}-1}\mathcal{I}_{J}$. 
   \par
   Case 2)
   With $t_0 + \beta_{J-1} \geq t_0 + \alpha_{J}$, for $t\in\left[  t_0 + \alpha_{J}, t_0 + \beta_{J-1}\right]$:
   \begin{equation}
       \begin{gathered}
           \mathbf{T}_{\alpha_J,\beta_J}^{\Theta_J}V_h(x,t;t_0)  \geq 0 \Leftrightarrow  h(x(t)) \geq 0,
       \end{gathered}
   \end{equation} 
   and 
   \begin{equation}
       \begin{gathered}
           \forall t \in \left[  t_0 + \alpha_{J-1}, t_0 + \beta_{J-1}\right]:
           \\
           \mathbf{T}_{\alpha_{J-1},\beta_{J-1}}^{\Theta_{J-1}}V_h(x,t;t_0) = h(x(t))
           \overset{\alpha_J \geq \alpha_{J-1}}{\Longleftrightarrow}\\
           \forall t \in \left[  t_0 + \alpha_{J}, t_0 + \beta_{J-1}\right]:
           \\
           \mathbf{T}_{\alpha_{J-1},\beta_{J-1}}^{\Theta_{J-1}}V_h(x,t;t_0) = h(x(t)).
       \end{gathered}
   \end{equation}
   Therefore 
   \begin{equation}
       \begin{gathered}
           \mathbf{T}_{\alpha_J,\beta_J}^{\Theta_J}V_h(x,t;t_0)  \geq 0 \Leftrightarrow  h(x(t)) \geq 0
           \\
           \Leftrightarrow
           \mathbf{T}_{\alpha_{J-1},\beta_{J-1}}^{\Theta_{J-1}}V_h(x,t;t_0)\geq 0.
       \end{gathered}
   \end{equation} 
   The above two cases can be inductively shown to hold till $J =1$, concluding case 2 for the interval $\bigcup_{J=1}^{\bar{J}-1}\mathcal{I}_{J}$. 
   \par
   We finally consider the remaining interval: 
   $ \left[
        t_0 + \alpha''(\tau'') + s_{\bar{J}-1}',
        t_0 + \beta''(\tau'')
    \right]$. By virtue of \eqref{eq:thm:nesting:alw:event:alw:terminal}:
        \begin{equation}
            \begin{gathered}
                \nexists \tau_{\bar{J}}' \in \Theta:
                \alpha'' + \sum_{j=1}^{\bar{J}}\tau_j \geq \beta''
                \Leftrightarrow \\
                \forall  \tau_{\bar{J}} \in\Theta:
                \alpha'' + \sum_{j=1}^{\bar{J}}\tau_j > \beta''
                \Leftrightarrow \\
                t_0 + \alpha'' + s_{\bar{J}} >  t_0 +\beta'',
            \end{gathered}
        \end{equation}
        and
        \begin{equation}
            \begin{gathered}
                \alpha'' + \sum_{j=1}^{\bar{J}-1}\tau_j \leq \beta''
                \Leftrightarrow \\
                \alpha'' + \sum_{j=1}^{\bar{J}-1}\tau_j \leq \beta''
                \Leftrightarrow \\
                t_0 + \alpha'' + s_{\bar{J}-1} \leq  t_0 + \beta'',
            \end{gathered}
        \end{equation}
        Since the function $\tau_J'(t_J'') = t_0 + \alpha''(\tau'') + s_J' - t_J''$  is strictly decreasing, $\forall t \in \left[  t_0 + \alpha'' + s_{\bar{J}-1}, t_0 +\beta''\right]$
        \begin{equation}
            \begin{gathered}
                t_0 + \alpha'' + s_{\bar{J}-1} \leq  t <t_0 + \alpha'' + s_{\bar{J}}   \Leftrightarrow 
                \\
                \hat{\tau}\left( t_0 + \alpha'' + s_{\bar{J}-1} \right)
                \geq  
                \hat{\tau}\left(t\right)
                >
                \hat{\tau}\left(t_0 + \alpha'' + s_{\bar{J}}\right)
                 \\
                 \Leftrightarrow
                s_{\bar{J}} - s_{\bar{J}-1} = \tau_{\bar{J}}
                \geq  
                \hat{\tau}\left(t\right)
                >
                0 \Leftrightarrow
                \hat{\tau}\left(t\right)\in\Theta,
            \end{gathered}
        \end{equation}
        which shows that $\mathbf{T}_{\alpha_{\bar{J}},\beta_{\bar{J}}}^{\Theta_{\bar{J}}}V_h\geq 0$ for the time $t_0 + \alpha'' + s_{\bar{J}-1}$ implies satisfaction for the entire remaining interval, showing equivalence for the interval $I_{\bar{J}}$. The necessity part of the equivalence stems from \eqref{eq:thm:nesting:alw:event:alw:params} for $\beta_J$, where it is restricted to the value $\beta'' + \underline{t}' + \beta$. 
        Therefore, \eqref{eq:thm:nesting:alw:event:alw:1} $\Leftrightarrow$ \eqref{eq:thm:nesting:alw:event:alw:2}, concluding the proof. 
\end{proof}
\subsubsection{Proof of Thm. \ref{thm:nested:nested}}
\begin{proof}
    We have:
    \begin{equation}
        \begin{gathered}
            (x,t_0) \models \psi_1  
            =
            G_{[\alpha_1''(\tau_1''),\beta_1''(\tau_1'')]} F_{[\underline{t}_1',\overline{t}_1']}\psi_0
            \Leftrightarrow
            \\
            \forall t'' \in 
            \left[
                t_0 + \alpha_1''(\tau_1''),
                t_0 + \beta_1''(\tau_1'')
            \right]:
            \\
            \exists t' \in 
            \left[
                t'' + \underline{t}_1',
                t'' + \overline{t}_1'
            \right]:
            (x,t') \models \psi_0
            \overset{\eqref{eq:thm:nested:nested:psi0}}{\Leftrightarrow}
            \\
            \forall t'' \in 
            \left[
                t_0 + \alpha_1''(\tau_1''),
                t_0 + \beta_1''(\tau_1'')
            \right]:
            \\
            \exists t' \in 
            \left[
                t'' + \underline{t}_1',
                t'' + \overline{t}_1'
            \right]:
           \forall J_0\in\{1,\cdots,\bar{J}_0\}:
                \\
                \exists \tau_{J_0}^0 \in \fundef{\mathbb{R}}{\Theta_0}:
                \forall t \in\left[ t',t' + \beta_{J_0}\right]:
                \\
                \underset{j_0 \in \{1,\cdots,{J}_0\}}{\min} 
                \left\{
                    \mathbf{T}_{\alpha_{j_0},\beta_{J_0}}^{\Theta_0}V_h(x,t;t') 
                \right\}    
                    \geq 0.
        \end{gathered}
    \end{equation}
    Define the interval $\forall J_1 \in \{1,\cdots,\bar{J}_1\}$:
    \begin{equation}
        \begin{gathered}
            \mathcal{I}^1_{J_1}
            =
            \left[
                t_0 + \alpha_1''(\tau_1'') + s_{J_1-1}^1,
                t_0 + \beta_1''(\tau_1'') + s_{J_1}^1
            \right],
        \end{gathered}
    \end{equation}
    with 
    \begin{equation}
        s_{J_1}^1 = \sum_{j_1 = 1}^{J_1}\tau^1_{j_1}, 
        \tau_{j_1}^1 \in \left[0, \overline{t}_1' - \underline{t}_1'\right].
    \end{equation}
    for which:
   \begin{equation}
       \begin{gathered}
           \left[
            t_0 + \alpha_1''(\tau_1''),
            t_0 + \beta_1''(\tau_1'')
            \right]
            =
            \\
            \left(
                \bigcup_{J_1=1}^{\bar{J}_1-1}\mathcal{I}^1_{J_1}
            \right)
            \cup
            \left[
                t_0 + \alpha_1''(\tau_1'') + s_{\bar{J}_1-1}',
                t_0 + \beta_1''(\tau_1'')
            \right].
       \end{gathered}
   \end{equation}
   We first consider the interval $\bigcup_{J_1=1}^{\bar{J}_1-1}\mathcal{I}_{J_1}^1$:
   \begin{equation}\label{eq:thm:nesting:nesting:ind1}
       \begin{gathered}
            (x,t_0) \models \phi \Leftrightarrow
            \forall J_1 \in 
            \{1,\cdots,\bar{J}_1\}:
            \\
            \forall t_{J_1}'' \in 
            \left[
                t_0 + \alpha_1''(\tau_1'') + s^1_{J_1-1},
                t_0 + \alpha_1''(\tau_1'') + s^1_{J_1}
            \right]:
            \\
            \exists t' \in 
            \left[
                t_{J_1}'' + \underline{t}_1',
                t_{J_1}'' + \overline{t}_1'
            \right]:
            \\
            \forall J_0\in\{1,\cdots,\bar{J}_0\}:
                \\
                \exists \tau_{J_0}^0 \in \fundef{\mathbb{R}}{\Theta_0}:
                \forall t \in\left[ t',t' + \beta_{J_0}\right]:
                \\
                \underset{j_0 \in \{1,\cdots,{J}_0\}}{\min} 
                \left\{
                    \mathbf{T}_{\alpha_{j_0},\beta_{J_0}}^{\Theta_0}V_h(x,t;t') 
                \right\}    
                    \geq 0
            \Leftrightarrow \\
            \forall J_1 \in 
            \{1,\cdots,\bar{J}_1\}:
            \\
            \forall t_{J_1}'' \in 
            \left[
                t_0 + \alpha''(\tau'') + s^1_{J_1-1},
                t_0 + \alpha''(\tau'') + s^1_{J_1}
            \right]:
            \\
            \exists
            \tau_{J_1}^1\in
            \left[
                0,
                \overline{t}' - \underline{t}'
            \right]
            ,
            t'  = t_{J_1}'' + \underline{t}' +\tau_{J_1}^1
            \in 
            \left[
                t_{J_1}'' + \underline{t}_1',
                t_{J_1}'' + \overline{t}_1'
            \right]
            :
            \\
            \forall J_0\in\{1,\cdots,\bar{J}_0\}:
                \\
                \exists \tau_{J_0}^0 \in \fundef{\mathbb{R}}{\Theta_0}:
                \forall t \in\left[ t',t' + \beta_{J_0}\right]:
                \\
                \underset{j_0 \in \{1,\cdots,{J}_0\}}{\min} 
                \left\{
                    \mathbf{T}_{\alpha_{j_0},\beta_{J_0}}^{\Theta_0}V_h(x,t;t') 
                \right\}    
                    \geq 0
            \Leftrightarrow \\
            \forall J_1 \in 
            \{1,\cdots,\bar{J}_1\}:
            \\
            \forall t_{J_1}'' \in 
            \left[
                t_0 + \alpha''(\tau'') + s^1_{J_1-1},
                t_0 + \alpha''(\tau'') + s^1_{J_1}
            \right]:
            \\
            \exists
            \tau_{J_1}^1\in
            \left[
                0,
                \overline{t}' - \underline{t}'
            \right]
            \triangleq
            \Theta_1:
            \\
            \forall J_0\in\{1,\cdots,\bar{J}_0\}:
                \exists \tau_{J_0}^0 \in \fundef{\mathbb{R}}{\Theta_0}:
                \\
                \forall t \in
                \left[ 
                    t_{J_1}'' + \underline{t}' +\tau_{J_1}^1,
                    t_{J_1}'' + \underline{t}' +\tau_{J_1}^1 + \beta_{J_0}
                \right]:
                \\
                \underset{j_0 \in \{1,\cdots,{J}_0\}}{\min} 
                \left\{
                    \mathbf{T}_{\alpha_{j_0},\beta_{J_0}}^{\Theta_0}V_h(x,t;t_{J_1}'' + \underline{t}' +\tau_{J_1}^1) 
                \right\}    
                    \geq 0
            \overset{\mathrm{Prop. } \ref{prop:operator:reachable}}{\Leftrightarrow}
            \\
            \forall J_1 \in 
            \{1,\cdots,\bar{J}_1\}:
            \\
            \forall t_{J_1}'' \in 
            \left[
                t_0 + \alpha''(\tau'') + s^1_{J_1-1},
                t_0 + \alpha''(\tau'') + s^1_{J_1}
            \right]:
            \\
            \exists
            \tau_{J_1}^1\in
            \left[
                0,
                \overline{t}' - \underline{t}'
            \right]
            \triangleq
            \Theta_1:
            \\
            \forall J_0\in\{1,\cdots,\bar{J}_0\}:
                \exists \tau_{J_0}^0 \in \fundef{\mathbb{R}}{\Theta_0}:
                \\
                \forall t \in
                \left[ 
                    t_{J_1}'' + \underline{t}' +\tau_{J_1}^1+ \alpha_{J_0},
                    t_{J_1}'' + \underline{t}' +\tau_{J_1}^1 + \beta_{J_0}
                \right]:
                \\
                \underset{j_0 \in \{1,\cdots,{J}_0\}}{\min} 
                \left\{
                    \mathbf{T}_{\alpha_{j_0},\beta_{J_0}}^{\Theta_0}V_h(x,t;t_{J_1}'' + \underline{t}' +\tau_{J_1}^1) 
                \right\}    
                    \geq 0
       \end{gathered}
   \end{equation}
   However, $\forall t_{J_1}'' \in \mathcal{I}^1_{J_1} :\exists \tau_{J_1}^1 \in  \Theta_1 \Leftrightarrow \exists \tau_{J_1}^1\in\fundef{\mathcal{I}^1_{J_1}}{\Theta_1}$. Choosing $\tau_{J_1}^1(t_{J_1}'') \triangleq t_0 + \alpha''(\tau'') + s^1_{J_1} - t_{J_1}''$ the above yields:
   \begin{equation}\label{eq:thm:nesting:nesting:ind2}
       \begin{gathered}
            \forall J_1 \in 
            \{1,\cdots,\bar{J}_1\}:
            \\
            \forall t_{J_1}'' \in 
            \left[
                t_0 + \alpha''(\tau'') + s^1_{J_1-1},
                t_0 + \alpha''(\tau'') + s^1_{J_1}
            \right]:
            \\
            \forall J_0\in\{1,\cdots,\bar{J}_0\}:
                \exists \tau_{J_0}^0 \in \fundef{\mathbb{R}}{\Theta_0}:
                \\
                \forall t \in
                \left[ 
                     t_0 + \alpha'' + s^1_{J_1} + \underline{t}' + \alpha_{J_0},
                     t_0 + \alpha'' + s^1_{J_1}  + \underline{t}' + \beta_{J_0}
                \right]:
                \\
                \underset{j_0 \in \{1,\cdots,{J}_0\}}{\min} 
                \left\{
                    \mathbf{T}_{\alpha_{j_0},\beta_{J_0}}^{\Theta_0}V_h
                    (x,t;  t_0 + \alpha''(\tau'') + s^1_{J_1}  + \underline{t}' )
                \right\}    
                    \geq 0.
       \end{gathered}
   \end{equation}
   Furthermore $\forall j_0 \in \{1,\cdots,{J}_0\}$:
   \begin{equation}
       \begin{gathered}
            \mathbf{T}_{\alpha_{j_0},\beta_{j_0}}^{\Theta_0}V_h
            (x,t;  t_0 + \alpha_1''(\tau_1'') + s^1_{J_1}  + \underline{t}' )
           = 
           \\
           \begin{cases}
               V_h
               \left( 
                    x,t -t_0  - \alpha_1'' - s^1_{J_1}  - \underline{t}'
               \right),
               \\
               t \in 
               \left[ 
                    t_0 - \alpha_1'' - s^1_{J_1}  - \underline{t}',
                    t_0  - \alpha_1'' - s^1_{J_1}  - \underline{t}' + \alpha_{j_0}
                \right],
               \\
               \\
               h(x(t)),
               \\
               t\in
               \left[ 
                    t_0  - \alpha_1'' - s^1_{J_1}  - \underline{t}' + \alpha_{j_0},
                    t_0  - \alpha_1'' - s^1_{J_1}  - \underline{t}' + \beta_{J_0}
               \right]
           \end{cases}  
           \\
           \overset{\eqref{eq:thm:nesting:nesting:params}}{=}
           \begin{cases}
               V_h
               \left( 
                    x,t - t_0  - \alpha_{j_0,J_1}
               \right), 
               t\in
               \left[ 
                    t_0, t_0 + \alpha_{j_0,J_1}
               \right]
               \\
                h(x(t)), 
                t\in
               \left[ 
                    t_0 + \alpha_{j_0,J_1}, 
                    t_0 + \beta_{j_0,J_1}
               \right]
           \end{cases}
           \\
           =
           \mathbf{T}_{\alpha_{j_0,J_1},\beta_{j_0,J_1}}^{\Theta_{j_0,J_1}}V_h(x,t;t_0).
       \end{gathered}
   \end{equation}
   Therefore, Eqs. \eqref{eq:thm:nesting:alw:event:alw:ind1}, \eqref{eq:thm:nesting:alw:event:alw:ind2} yield:
   \begin{equation}
       \begin{gathered}
           (x,t_0) \models \phi \Leftrightarrow
            \\
            \forall J_1 \in 
            \{1,\cdots,\bar{J}_1\}:
            \\
            \forall t_{J_1}'' \in 
            \left[
                t_0 + \alpha''(\tau'') + s^1_{J_1-1},
                t_0 + \alpha''(\tau'') + s^1_{J_1}
            \right]:
            \\
            \forall J_0\in\{1,\cdots,\bar{J}_0\}:
                \exists \tau_{J_0}^0 \in \fundef{\mathbb{R}}{\Theta_0}:
                \\
                \forall t \in
                \left[ 
                     t_0 + \alpha_{J_0,J_1},
                     t_0 +  \beta_{J_0,J_1}
                \right]:
                \\
                \underset{j_0 \in \{1,\cdots,{J}_0\}}{\min} 
                \left\{
                    \mathbf{T}_{\alpha_{j_0,J_1},\beta_{j_0,J_1}}^{\Theta_{0,1}}V_h
                    (x,t; t_0 )
                \right\}    
                    \geq 0
       \end{gathered}
   \end{equation}
   Furthermore, note that $\forall J_0 \in \{1,\cdots,\bar{J}_0 \},\alpha_{J_0,1},\alpha_{J_0,2},\cdots,\alpha_{J_0,\bar{J}_1}$ is an increasing sequence. 
   There are two cases: 1) $t_0 + \beta_{J_0,J_1-1} < t_0 + \alpha_{J_0,J_1}$, 2) $t_0 + \beta_{J_0,J_1-1} \geq t_0 + \alpha_{J_0,J_1}$.
   \par 
   Case 1): Let $J_1:t \geq t_0 + \alpha_{J_0,J_1} > t_0 + \beta_{J_0,J_1-1}$, implying that instances of the operator for $j_1<J_1$ have elapsed. From Prop. \ref{prop:operator:sequence}, for any $J_1<\bar{J}_1$:
   \begin{equation}
       \begin{gathered}
           \underset{J_0 \in \{1,\cdots,\bar{J}_0\}}{\min}
           \left\{
                \mathbf{T}_{\alpha_{J_0,J_1},\beta_{J_0,J_1}}^{\Theta_{J_0,J_1}}V_h
            \right\}    
            \leq 
            \\
           \underset{J_0 \in \{1,\cdots,\bar{J}_0\}}{\min}
           \left\{
                \mathbf{T}_{\alpha_{J_0,J_1+1},\beta_{J_0,J_1+1}}^{\Theta_{J_0,J_1+1}}V_h
            \right\}    
            \leq 
            \\
            \underset{J_0 \in \{1,\cdots,\bar{J}_0\}}{\min}
           \left\{
                \mathbf{T}_{\alpha_{J_0,\bar{J}_1},\beta_{J_0,\bar{J}_1}}^{\Theta_{J_0,\bar{J}  _1}}V_h
            \right\} 
       \end{gathered}
   \end{equation}
   hence
   \begin{equation}
       \begin{gathered}
           \underset{J_0\in\{1,\cdots,\bar{J}_0\}}{\min}
           \left\{
               \mathbf{T}_{\alpha_{J_0,J_1},\beta_{J_0,J_1}}^{\Theta_{J_0,J_1}}V_h 
            \right\}\geq 0 
            \Leftrightarrow
            \\
            \underset{J_1\in\{j_1,\cdots,\bar{J}_1\}}{\min}
           \left\{
                \underset{J_0\in\{1,\cdots,\bar{J}_0\}}{\min}
               \left\{
                   \mathbf{T}_{\alpha_{J_0,J_1},\beta_{J_0,J_1}}^{\Theta_{J_0,J_1}}V_h 
                \right\}
            \right\}\geq 0,
       \end{gathered}
   \end{equation}
   concluding the proof for the interval $\bigcup_{J_1=1}^{\bar{J}_1-1}\mathcal{I}_{J_1}^1$. 
    Case 2)
   With $t_0 + \beta_{J_0,J_1-1} \geq t_0 + \alpha_{J_0,J_1}$, for $t\in\left[  t_0 + \alpha_{J_0,J_1}, t_0 + \beta_{J_0,J_1-1}\right]$:
   \begin{equation}
       \begin{gathered}
           \underset{J_0 \in \{1,\cdots,\bar{J}_0\}}{\min}
           \left\{
                \mathbf{T}_{\alpha_{J_0,J_1},\beta_{J_0,J_1}}^{\Theta_{J_0,J_1}}V_h
            \right\} 
            =
            h(x)
            \geq 0,
       \end{gathered}
   \end{equation} 
   and due to the overlap of the intervals 
   \begin{equation}
       \begin{gathered}
            h(x) = 
           \underset{J_0 \in \{1,\cdots,\bar{J}_0\}}{\min}
           \left\{
                \mathbf{T}_{\alpha_{J_0,J_1-1},\beta_{J_0,J_1-1}}^{\Theta_{J_0,J_1-1}}V_h
            \right\} 
            \geq 0,
       \end{gathered}
   \end{equation} 
   concluding case 2 for the interval $\bigcup_{J_1=1}^{\bar{J}_1-1}\mathcal{I}_{J_1}^1$.
   \par
   Finally consider the the remaining interval 
   $ \left[
        t_0 + \alpha_1''(\tau_1'') + s^1_{\bar{J}_1-1},
        t_0 + \beta_1''(\tau_1'')
    \right]$. By virtue of \eqref{eq:thm:nesting:nested:terminal}:
        \begin{equation}
            \begin{gathered}
                \nexists \tau_{\bar{J}_1}^1 \in \Theta:
                \alpha'' + \sum_{j_1=1}^{\bar{J}_1}\tau_{j_1}^1 \geq \beta_1''
                \Leftrightarrow \\
                \forall  \tau_{\bar{J}_1} \in\Theta:
                \alpha_1'' + \sum_{j=1}^{\bar{J}}\tau_{j_1}^1 > \beta_1''
                \Leftrightarrow \\
                t_0 + \alpha_1'' + s_{\bar{J_1}}^1 >  t_0 +\beta_1'',
            \end{gathered}
        \end{equation}
        and
        \begin{equation}
            \begin{gathered}
                \alpha_1'' + \sum_{j_1=1}^{\bar{J}_1-1}\tau_{j_1}^1 \leq \beta_1''
                \Leftrightarrow \\
                \alpha_1'' + \sum_{j_1=1}^{\bar{J}-1}\tau_{j_1}^1 \leq \beta_1''
                \Leftrightarrow \\
                t_0 + \alpha_1'' + s_{\bar{J}_1-1}^1 \leq  t_0 + \beta_1'',
            \end{gathered}
        \end{equation}
        Since the function $\tau_{J_1}^1(t_{J_1}'') = t_0 + \alpha_1''(\tau_1'') + s_{J_1}^1 - t_{J_1}''$  is strictly decreasing, $\forall t \in \left[  t_0 + \alpha_1'' + s_{\bar{J}_1-1}^1, t_0 +\beta_1''\right]$
        \begin{equation}
            \begin{gathered}
                t_0 + \alpha_1'' + s_{\bar{J}_1-1}^1 \leq  t <t_0 + \alpha'' + s_{\bar{J}_1}   \Leftrightarrow 
                \\
                \hat{\tau}\left( t_0 + \alpha_1'' + s_{\bar{J}_1-1}^1 \right)
                \geq  
                \hat{\tau}\left(t\right)
                >
                \hat{\tau}\left(t_0 + \alpha_1'' + s_{\bar{J}_1}^1\right)
                 \\
                 \Leftrightarrow
                s_{\bar{J}_1} - s_{\bar{J}-1}^1 = \tau_{\bar{J}_1}
                \geq  
                \hat{\tau}\left(t\right)
                >
                0 \Leftrightarrow
                \hat{\tau}\left(t\right)\in\Theta_1,
            \end{gathered}
        \end{equation}
        which shows that $\underset{J_0\in\{1,\cdots,\bar{J}_0\}}{\min}\{\mathbf{T}_{\alpha_{J_0,\bar{J}_1},\beta_{J_0,\bar{J}_1}}^{\Theta_{J_0,\bar{J}_1}}V_h\}\geq 0$ from the time $t_0 + \alpha_1'' + s_{\bar{J}_1-1}^1$ implies satisfaction for the entire remaining interval, where necessity stems from restricting the final time instance in \eqref{eq:thm:nesting:nesting:params}.
        Therefore, \eqref{eq:thm:nesting:nested:1} $\Leftrightarrow$ \eqref{eq:thm:nesting:nested:2}, concluding the proof. 
\end{proof}

\subsubsection{Proof of Thm. \ref{thm:nested:nested:eventually}}
\begin{proof}
    We have:
    \begin{equation}
        \begin{gathered}
            (x,t_0) \models \phi  = F_{[\underline{t},\overline{t}]}\psi\Leftrightarrow
            \\
            \exists t \in
            \left[ 
                t_0 + \underline{t}, t_0 + \overline{t}
            \right]: (x,t) \models \psi 
            \Leftrightarrow
            \\
            \exists \tau \in \left[ 0, \overline{t} - \underline{t} \right]:
            (x,t_0 + \underline{t}+\tau) \models \psi \overset{\eqref{eq:thm:event:nesting:nested:1}}{\Leftrightarrow}
            \\
            \exists \tau \in \left[ 0, \overline{t} - \underline{t} \right]:
            \\
            \forall J_N \in \{1,\cdots,\bar{J}_N\},
            \exists \tau_{J_N}^N \in\fundef{\mathbb{R}}{\Theta_N}:
            \\
            \forall J_{N-1} \in \{1,\cdots,\bar{J}_{N-1}\},
            \exists \tau_{J_{N-1}}^{N-1} \in\fundef{\mathbb{R}}{\Theta_{N-1}}:
             \\
            \vdots 
            \\
            \forall J_0 \in \{1,\cdots,\bar{J}_0\},
            \exists \tau_{J_0}^0\in\fundef{\mathbb{R}}{\Theta_0}
            :
            \\
            \underset{j_{N}\in\{1,\cdots \bar{J}_{N}\}}{\min}
            \left\{
                \cdots
                \underset{j_{0}\in\{1,\cdots \bar{J}_{0}\}}{\min}
                \left\{
                    \mathbf{T}_{\alpha,\beta}^{\Theta}V_h(x,t;t_0 + \underline{t}+\tau)
                \right\}
            \right\}
            \geq 0.
        \end{gathered}
    \end{equation}
    Furthermore $\forall n \in \{1,\cdots,N\}, \forall  J_n \in \{1,\cdots,\bar{J}_n\}$ through \eqref{eq:operator:action}:
    \begin{equation}
        \begin{gathered}
            \mathbf{T}_{\alpha,\beta}^{\Theta}V_h(x,t;t_0 + \underline{t}+\tau) = 
            \\
            \begin{cases}
                V_h(x,t - \alpha-t_0 - \underline{t}-\tau), \\ 
                \qquad t  \in \left[t_0,\alpha+t_0 + \underline{t}+\tau\right]
                \\
                h(x) , \\ 
                \quad t\in \left(\alpha+t_0+\underline{t}+\tau,\beta+t_0+\underline{t}+\tau\right]
            \end{cases}
            \\
            \overset{\eqref{eq:thm:event:nesting:nested:params}}{=}
            \begin{cases}
                V_h(x,t - \alpha'-t_0), & t  \in \left[t_0,\alpha'+t_0 \right]
                \\
                h(x) , & t\in \left(\alpha'+t_0,\beta'+t_0\right]
            \end{cases}
            \\
            \overset{\eqref{eq:operator:action}}{=}
            \mathbf{T}_{\alpha',\beta'}^{\Theta'}V_h(x,t;t_0),
        \end{gathered}
    \end{equation}
    which concludes the proof. 
\end{proof}

\subsubsection{Proof of Prop. \ref{prop:prelim:thm:formula:satisfaction}}
\begin{proof}
 Consider a trajectory of system \eqref{eq:dynamics} $x\in\fundef{\mathbb{R}}{\mathcal{X}}$ under an admissible input, as well as the constituents of the operators in Eq. \eqref{eq:neted:operator:ex}. Let any two complete paths of $\mathcal{T}^T$ be denoted as $\mathbf{p}_{k_1}, \mathbf{p}_{k_2},k_1,k_2\in\{1,\cdots,K\}$, given by \eqref{eq:complete:path}. Owing to the tree structure, they share a common subset of operator nodes on $\mathcal{T}^T$, i.e.:
 \begin{equation}\notag 
     \begin{split}
         \mathbf{p}_{k_1} &= \mathbb{X}_0 \widetilde{\Theta}_1  \cdots  \mathbb{X}_{n-1}  \widetilde{\Theta}_{n} \mathbb{X}_{n}' \cdots \widetilde{\Theta}_{N_f^{k_1}-1}' \mathbb{X}_{N_f^{k_1}-1}' \widetilde{\Theta}_{N_f^{k_1}}' \mathbb{X}_{N_f^{k_1}}'
         \\
                          &= \mathbf{p}_{0:n} \mathbf{p}_{n:N_f^{k_1}}',
         \\
         \mathbf{p}_{k_2} &= \mathbb{X}_0 \widetilde{\Theta}_1  \cdots  \mathbb{X}_{n-1}  \widetilde{\Theta}_{n} \mathbb{X}_{n} \cdots \widetilde{\Theta}_{N_f^{k_2}-1} \mathbb{X}_{N_f^{k_1}-1} \widetilde{\Theta}_{N_f^{k_2}} \mathbb{X}_{N_f^{k_2}}
         \\
                          &= \mathbf{p}_{0:n} \mathbf{p}_{n:N_f^{k_2}},
     \end{split}
 \end{equation}
 where $\mathbf{p}_{i:j}$ denotes the path segment from $i$ to $j$, $\widetilde{\Theta}_{n} \in\{\wedge,\vee\}$, $n \in \left\{1,\cdots,\min\{ N_f^{k_1}, N_f^{k_2} \}\right\}$ and $\widetilde{\Theta}_{i} \neq \widetilde{\Theta}_{i}',   \mathbb{X}_{i}\neq \mathbb{X}_{i}', \forall i \in \left\{n,\cdots,\min\{ N_f^{k_1}, N_f^{k_2} \}\right\}$. $\widetilde{\Theta}_{n}$ corresponds exactly to the LCA of $l_{k_1},l_{k_2}$, denoted by $v$. We remind the reader that, following the arguments of Prop. \ref{prop:complete_paths}, the two differing sub-paths of $\mathbf{p}_{k_1}, \mathbf{p}_{k_2}$ correspond to the operator nestings for $i\in\{1,2\}$:
 \begin{equation}\notag 
     \begin{split}
         \mathbf{T}_{\alpha^{n+1},\beta^{n+1}}^{\Theta^{n+1}}V_{h_{k_i}} \triangleq \mathbf{T}_{\alpha_{k_i}^{n+1},\beta_{k_i}^{n+1}}^{\Theta_{k_i}^{n+1}}  \cdots \mathbf{T}_{\alpha_{k_i}^{N_f^{k_i}},\beta_{k_i}^{N_f^{k_i}}}^{\Theta_{k_i}^{N_f^{k_i}}}V_{h_{k_i}},
         % \\
         % \mathbf{T}_{\hat{\alpha}^{n+1},\hat{\beta}^{n+1}}^{\hat{\Theta}^{n+1}}V_{h_{k_2}} \triangleq \mathbf{T}_{\alpha_{k_2}^{n+1},\beta_{k_2}^{n+1}}^{\Theta_{k_2}^{n+1}}  \cdots \mathbf{T}_{\alpha_{k_2}^{N_f^{k_2}},\beta_{k_2}^{N_f^{k_2}}}^{\Theta_{k_1}^{N_f^{k_1}}}V_{h_{k_2}},
     \end{split}
 \end{equation}
 while the common segment of the paths corresponds to:
 \begin{equation}\notag 
     \mathbf{T}_{\alpha',\beta'}^{\Theta'} = 
     \mathbf{T}_{\alpha_1,\beta_1}^{\Theta_1}  \cdots \mathbf{T}_{\alpha_{n-1},\beta_{n-1}}^{\Theta_{n-1}},
 \end{equation}
and as in Prop. \ref{prop:complete_paths}, $\mathbf{T}_{\alpha^{n},\beta^{n}}^{\Theta^{n}}$ is the identity operator corresponding to $\widetilde{\Theta}_{n} \in\{\wedge,\vee\}$ (or $v$ in the notation of $\mathcal{G}_{\textrm{LOG}}$). 
For the two sub-paths $ i \in \{1,2\}:(x,0) \cong \mathbf{p}_{k_i}$ is equivalent ($\Leftrightarrow$) to:
\begin{equation}\notag 
    \begin{split}
        \left( 
            (x,0) \cong \mathbf{p}_{0:n}
        \right) 
        \wedge
        \left(
            \forall t \in \left[ \underline{t}_{n},  \overline{t}_{n}\right]:
            (x,t) \cong \mathbf{p}_{n:N_f^{k_i}}'
        \right),
    \end{split}
\end{equation}
where $\left[ \underline{t}_{n},  \overline{t}_{n}\right]$ is the time interval coding for the $n$-th point on the segments (see Def. 13 in \citebr{10388467}). The above yields:
\begin{equation}
    \begin{split}
        &\left( 
                (x,0) \cong \mathbf{p}_{k_1}
            \right) 
            \wedge
            \left( 
                (x,0) \cong \mathbf{p}_{k_2}
            \right)
            \Leftrightarrow
            \left( 
                (x,0) \cong \mathbf{p}_{0:n}
            \right)
            \\
            &\wedge
            \left[
                \left(
                    \forall t \in \left[ \underline{t}_{n},  \overline{t}_{n}\right]:
                    (x,t) \cong \mathbf{p}_{n:N_f^{k_1}}'
                \right)
                \right.
                \\
            &\qquad\qquad\qquad\left.
                \wedge
                \left(
                    \forall t \in \left[ \underline{t}_{n},  \overline{t}_{n}\right]:
                    (x,t) \cong \mathbf{p}_{n:N_f^{k_2}}'
                \right)            
            \right],
    \end{split}
    \notag 
\end{equation}
which following the arguments of Prop. \ref{prop:complete_paths} is equivalent to:
\begin{equation}\notag 
    \begin{split}
        \mathbf{T}_{\alpha',\beta'}^{\Theta'} \left(  \mathbf{T}_{\alpha^{n+1},\beta^{n+1}}^{\Theta^{n+1}}V_{h_{k_1}}(x,t) \wedge \mathbf{T}_{\hat{\alpha}^{n+1},\hat{\beta}^{n+1}}^{\hat{\Theta}^{n+1}}V_{h_{k_2}}(x,t) \right) \geq 0.
    \end{split}
\end{equation}
However, according to  Props. \ref{thm:nesting:always}, \ref{thm:nesting:eventually}, this is equal to:
\begin{equation}\notag 
    \min
    \left\{ 
        \mathbf{T}_{\alpha_{k_1},\beta_{k_1}}^{\Theta_{k_1}}V_{h_{k_1}}(x,t),
        \mathbf{T}_{\alpha_{k_2},\beta_{k_2}}^{\Theta_{k_2}}V_{h_{k_2}}(x,t)
    \right\}
    \geq 0,
\end{equation}
after implementing the nesting rules of Thms. \ref{thm:nesting:alw:event:alw}, \ref{thm:nested:nested}, \ref{thm:nesting:always}, \ref{thm:nesting:eventually}. Necessity and sufficiency comes from Prop. \ref{prop:complete_paths}. The same holds for $\vee,\max$.
\end{proof}
\end{document}